\documentclass{article} 
\usepackage{salesforce,times}

\usepackage{amsmath,amsfonts,bm}

\def\eqref#1{equation~\ref{#1}}

\def\1{\bm{1}}

\DeclareMathAlphabet{\mathsfit}{\encodingdefault}{\sfdefault}{m}{sl}
\SetMathAlphabet{\mathsfit}{bold}{\encodingdefault}{\sfdefault}{bx}{n}

\newcommand{\KL}{D_{\mathrm{KL}}}

\usepackage{hyperref}
\usepackage{url}
\usepackage{bbm}
\usepackage{enumitem}
\usepackage{amsthm}
\usepackage{amssymb}
\usepackage{graphicx}
\usepackage{subcaption}
\usepackage{wrapfig}
\usepackage{algorithm}
\usepackage{algorithmic}
\usepackage{multirow}
\usepackage{booktabs}
\usepackage[table]{xcolor}

\newtheorem{theorem}{Theorem}
\newtheorem{proposition}[theorem]{Proposition}
\newtheorem{corollary}[theorem]{Corollary}
\newtheorem{lemma}[theorem]{Lemma}
\newtheorem{remark}{Remark}

\newcommand{\dt}[1]{{\scriptsize\color{gray}#1}}

\title{RISE: Recursive Improvement via Self-Extrapolating Policy Distillation}

\author{Yang Li, Semih Yavuz, Shafiq Joty}
\sfsetaffiliation{Salesforce AI Research}
\sfsetauthornote{\texttt{\{yli2,syavuz,sjoty\}@salesforce.com}}

\begin{document}

\maketitle

\begin{abstract}
On-policy distillation (OPD) provides dense, per-token supervision for language model post-training, but its effectiveness is bottlenecked by teacher quality: external teachers suffer from distribution mismatch, while self-distillation with privileged conditioning is limited by in-context learning capacity. We propose \textbf{RISE} (\textbf{R}ecursive \textbf{I}mprovement via \textbf{S}elf-\textbf{E}xtrapolating Policy Distillation), which constructs a synthetic teacher directly from the model's own RLVR training trajectory. By extrapolating the displacement between the current checkpoint and a trailing anchor---in parameter space or output logit space---RISE converts a sparse outcome-induced parameter update into a dense token-level target, without any external model or privileged conditioning. RISE combines RLVR and OPD in a complementary loop: outcome rewards ground the extrapolation toward correct reasoning, while the extrapolated teacher refines token-level decisions. Moreover, since the teacher is refreshed every iteration as the student improves, distillation becomes a recursive improvement mechanism rather than a one-shot compression step. Experiments spanning mathematical reasoning, multi-domain STEM, code generation, and multi-turn agentic tasks show that RISE outperforms RLVR-only training and on-policy self-distillation across all settings.
\end{abstract}

\section{Introduction}
\label{sec:intro}

A central aspiration of artificial intelligence is to build systems that can \emph{recursively improve themselves}---becoming increasingly capable from their own experience with minimal human supervision~\citep{yang2026selfimprovement,tao2024survey}. For large language models (LLMs), this means updating parameters from self-generated data via reinforcement learning from human feedback or verifiable rewards (RLHF/RLVR)~\citep{ouyang2022training,shao2024deepseekmath,guo2025deepseek,yu2026dapo}, self-play~\citep{chen2024self,liu2025spiral,huang2025r,fang2026serl}, or rejection sampling fine-tuning~\citep{zelikman2022star,gulcehre2023reinforced,xiong2025minimalist}. All these approaches share a common limitation: the learning signal is \emph{sequence-level}---an outcome reward, a binary accept/reject label, or a scalar preference---providing no guidance on \emph{which tokens} were responsible for success or failure.

On-policy distillation (OPD) offers a richer signal for weight-based self-improvement. Rather than a single scalar per response, OPD provides dense per-token supervision: a teacher specifies a full next-token distribution at every position, telling the model not just \emph{whether} its response was correct but \emph{how} to improve each token-level decision~\citep{song2026survey}. However, OPD's promise hinges on a critical question: \emph{where does a reliable teacher come from?} External teachers suffer from distribution mismatch~\citep{li2026rethinking,zhu2026many}; on-policy self-distillation (OPSD) with privileged conditioning often fails to reflect token-level correctness, due to limited in-context learning (ICL) capability and uninformative privileged information~\citep{zhao2026self,hubotter2026reinforcement,li2026learning}; and subsequent heuristic remedies---token-level gating~\citep{xu2026tip,lu2026self}, divergence mixing~\citep{jung2025todi,jin2026entropy}, DAgger-style sampling~\citep{li2026revisiting,zhao2026decoupling}, trajectory refinement~\citep{yang2026reasoning}---treat symptoms rather than the root cause. All these approaches accept a flawed teacher as given; none addresses the fundamental question: \emph{what should the teacher be?}

We propose a shift in perspective. A natural candidate teacher for a policy $\pi_\theta$ is its own \emph{converged future self}---the policy $\pi^*$ that training is converging toward. While $\pi^*$ is unknown, the training trajectory reveals the direction of convergence. Let $\varphi$ map a policy to a vector space where linear operations are meaningful (e.g., logits or parameters).
The displacement $\varphi(\pi_{\theta'}) - \varphi(\pi_\theta)$ between the current checkpoint $\theta'$ and an earlier anchor $\theta$ captures the direction of recent improvement. We can \textbf{extrapolate} this displacement to synthesize a future teacher:
\begin{equation}
    \varphi(\pi_{\text{future}}) = \varphi(\pi_{\theta}) + \beta \cdot \big(\varphi(\pi_{\theta'}) - \varphi(\pi_{\theta})\big), \quad \beta > 1.
    \label{eq:rise}
\end{equation}
When $\beta = 1$ the teacher equals the current checkpoint (no distillation signal); when $\beta > 1$ it amplifies the model's most recent update, projecting beyond its current state along the training trajectory. Recent analyses reveal that post-training updates are dominated by a low-rank subspace and evolve near-linearly~\citep{cai2025predictability,wang2026not}, making extrapolation along the training direction a principled approximation (\S\ref{sec:background}). This construction eliminates the core pathologies of prior OPD: reduced distribution mismatch (the teacher extrapolates the student's own optimization path, so it remains nearby in weight and distribution space), no privileged conditioning, and no external model---only previous checkpoints, naturally available during training.

The construction above is iterative by design: each distillation step refines the current policy, which in turn provides a fresh displacement for the next extrapolation. For this recursion to converge, the extrapolation direction must point toward genuine improvement---yet extrapolation itself is direction-agnostic, amplifying whatever update the model made regardless of quality. We resolve this by using RLVR to produce the update $\theta \to \theta'$: outcome rewards ground the displacement in verified improvement. On-policy distillation then applies the extrapolated teacher to refine $\theta'$, providing the per-token supervision that outcome rewards alone cannot. The two signals are complementary: RLVR ensures the direction is meaningful; OPD ensures the refinement is fine-grained (Fig.~\ref{fig:rise-overview}).

\begin{figure}[t]
\centering
\begin{subfigure}[b]{0.44\textwidth}
\centering
\includegraphics[width=\textwidth]{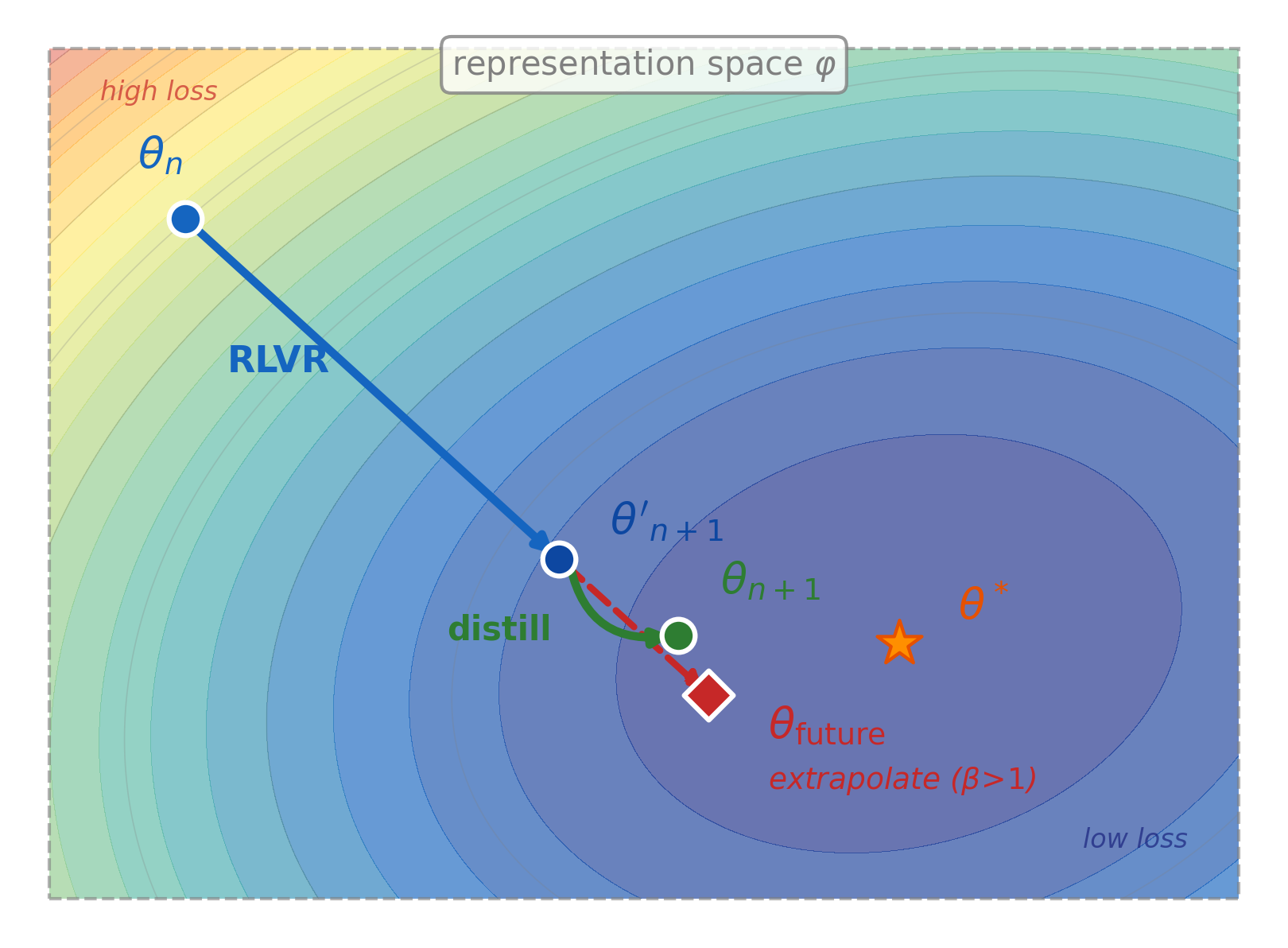}
\caption{Extrapolation geometry.}
\label{fig:extrapolation}
\end{subfigure}
\hfill
\begin{subfigure}[b]{0.54\textwidth}
\centering
\includegraphics[width=\textwidth]{figures/figure1b.png}
\caption{The RISE training loop.}
\label{fig:pipeline}
\end{subfigure}

\vspace{-2pt}
\caption{\textbf{RISE overview.} (a)~RLVR updates $\theta_n \to \theta_{n\text{+}1}'$ (blue); extrapolation amplifies this displacement to construct $\theta_{\text{future}}$ (red); OPD distills $\theta_{\text{future}}$ into the student, yielding $\theta_{n\text{+}1}$ (green). (b)~The training loop: RLVR grounds the direction, while OPD refines token-level decisions.}
\label{fig:rise-overview}
\vspace{-10pt}
\end{figure}

Our contributions are threefold. \textbf{(1)}~We identify teacher quality as the central bottleneck of OPD and propose \textbf{RISE}, which constructs the teacher by extrapolating the model's own RLVR trajectory, requiring no external model and no privileged context. Because the teacher is refreshed every iteration from the student's latest update, distillation becomes a recursive improvement loop rather than a one-shot compression step (\S\ref{sec:method}). \textbf{(2)}~We unify this construction under a representation map $\varphi$ and instantiate it in two spaces---\emph{logit-space} (geometric mixture of output distributions) and \emph{weight-space} (task arithmetic)---characterizing their computational and statistical trade-offs and ablating both (\S\ref{sec:method}, \S\ref{sec:experiments}). \textbf{(3)}~We show that RISE outperforms RLVR and OPSD baselines across mathematical reasoning, multi-domain STEM, code generation, and multi-turn agentic tasks, with improved sample efficiency, at no additional sampling cost and a modest 1.3--1.6$\times$ wall-time overhead (\S\ref{sec:experiments}).

\section{Background and Related Work}
\label{sec:background}

\subsection{Reinforcement Learning with Verifiable Rewards}
\label{sec:bg-rlvr}

RLVR optimizes a language model policy $\pi_\theta$ to maximize expected reward on tasks with objectively checkable outputs. Given a prompt $x$, the model generates a response $y = (y_1, \ldots, y_T) \sim \pi_\theta(\cdot \mid x)$ and receives an outcome reward $R(x, y) \in [0, 1]$. Policy gradient algorithms such as GRPO~\citep{guo2025deepseek} and DAPO~\citep{yu2026dapo} estimate per-token advantages from group-normalized outcome rewards and update the policy via clipped surrogate objectives:
\begin{equation}\textstyle
    \mathcal{L}_{\text{PG}} = -\mathbb{E}\left[\sum_{t=1}^{T} \hat{A}_t \cdot \min\!\Big(\rho_t, \operatorname{clip}(\rho_t, 1\!-\!\epsilon, 1\!+\!\epsilon)\Big)\right], \quad \rho_t = \frac{\pi_\theta(y_t \mid x, y_{<t})}{\pi_{\text{old}}(y_t \mid x, y_{<t})},
    \label{eq:pg}
\end{equation}
where $\hat{A}_t$ is the advantage estimate (constant across tokens within a response in GRPO/DAPO). While effective, the outcome-level reward assigns identical advantage to every token in a response, creating a credit assignment bottleneck: the gradient signal cannot distinguish helpful reasoning steps from irrelevant or harmful ones.

\subsection{On-policy Distillation}
\label{sec:bg-opd}

On-policy distillation (OPD) augments policy gradient training with per-token distributional supervision from a teacher model $\pi_T$. At each position $t$, the student minimizes a divergence between its distribution and the teacher's:
\begin{equation}\textstyle
    \mathcal{L}_{\text{OPD}} = \mathbb{E}_{y \sim \pi_\theta}\left[\sum_{t=1}^{T} \KL\!\left(\pi_\theta(\cdot \mid x, y_{<t}) \;\|\;\pi_T(\cdot \mid x, y_{<t})\right)\right].
    \label{eq:opd}
\end{equation}
Unlike the outcome reward, this loss provides a rich gradient at every token: the teacher's full distribution over the vocabulary specifies not only which token is best, but the relative desirability of all alternatives. In practice, computing the divergence over the full vocabulary is expensive, so most implementations resort to either a top-$K$ truncation or a REINFORCE-style sample-based estimator~\citep{lu2025onpolicydistillation,agarwal2024policy}. The critical challenge is the choice of teacher $\pi_T$.
Existing approaches fall into two categories, each with fundamental limitations.

\textbf{External teacher OPD.~}
The most direct approach uses a separate, stronger model as $\pi_T$---for example, distilling from a larger model in the same family~\citep{mimo2025flash,zeng2026glm}. The teacher evaluates $\pi_T(\cdot \mid x, y_{<t})$ conditioned on the student's generated prefix $y_{<t}$. As training progresses, the student explores increasingly diverse reasoning paths, many unseen by the teacher. The teacher's distribution conditioned on such out-of-distribution prefixes becomes unreliable: it reflects how the teacher would continue from an alien context, not whether the student's reasoning is sound. This \emph{prefix distribution mismatch} degrades teacher quality precisely when the student needs guidance most---at novel or challenging reasoning steps~\citep{li2026rethinking,zhu2026many}.

\textbf{On-policy self-distillation (OPSD).~}
An alternative formulation, on-policy self-distillation (OPSD), uses the \emph{same} base model as both teacher and student, conditioning the teacher on privileged information such as correct solutions or environmental feedback~\citep{zhao2026self,hubotter2026reinforcement}. The teacher distribution becomes $\pi_T(\cdot \mid x, C, y_{<t})$, where $C$ is the privileged context. The hope is that ICL enables the model to produce a more informed distribution when conditioned on $C$. However, this approach has its own limitations. First, the model's ICL capability may be insufficient to meaningfully incorporate the privileged context~\citep{li2026learning}. Second, the teacher and student see different inputs, creating a different form of distribution mismatch: the teacher's distribution may reflect the privileged context in ways that are not transferable or even harmful to the student's unconditioned setting~\citep{kim2026does,harne2026privileged}.

\textbf{Heuristic Remedies.~}
Recognizing these limitations, recent work has proposed several heuristic strategies. \textbf{Token-level gating} selectively weights the per-token loss using the teacher--student probability gap~\citep{xu2026tip,lu2026self,xing2026trust}, token entropy~\citep{xu2026tip,ko2026scaling}, or teacher confidence~\citep{liu2026prefix}. \textbf{Divergence mixing} interpolates between forward and reverse KL or selects adaptively per token, trading mode coverage against mode seeking~\citep{hubotter2026reinforcement,jung2025todi,jin2026entropy,jia2026asymmetric,jang2026stable}. \textbf{DAgger-style rollout mixing}~\citep{ross2011reduction} interleaves teacher and student tokens during generation to keep prefixes on-distribution~\citep{li2026revisiting,agrawal2026reinforcement,zhao2026decoupling}. \textbf{Trajectory refinement} refines rollouts before distillation so the teacher conditions on higher-quality prefixes~\citep{yang2026reasoning,jiang2026trajectory}. All these remedies accept a flawed teacher and engineer ways to tolerate its noise, rather than questioning whether a better teacher exists. In \S\ref{sec:method}, we show that the model's own training trajectory can be used to construct such a teacher.

\subsection{Task Arithmetic and Linear Trajectories}
\label{sec:bg-task-arithmetic}

RISE builds on convergent empirical evidence that post-training trajectories are low-dimensional and approximately linear. In the model merging literature, \textbf{task vectors} $\boldsymbol{\tau} = \theta_{\text{ft}} - \theta_{\text{pre}}$ encode task-specific knowledge in a composable, approximately linear fashion~\citep{ilharco2022editing}, and fine-tuning trajectories exhibit approximate \textbf{linear mode connectivity}---checkpoints along the training path can be linearly interpolated without loss barriers~\citep{frankle2020linear}. The related \textbf{model soups} and \textbf{weight averaging} literature~\citep{wortsman2022model,izmailov2018averaging} demonstrates that uniform or stochastic averaging of checkpoints along (or near) the training trajectory consistently improves generalization, reinforcing that these trajectories lie in well-behaved, low-curvature regions of parameter space amenable to linear operations.

More recently, analyses of RLVR training reveal that parameter updates are dominated by a \textbf{rank-1 subspace} whose projection coefficient evolves near-linearly throughout training~\citep{cai2025predictability,wei2026you}, and that these low-rank weight dynamics propagate to linear evolution in output log-probabilities~\citep{wang2026not}. A parallel finding holds for OPD, whose cumulative updates rapidly lock into a narrow low-dimensional channel~\citep{shen2026geometry}. Together, these properties justify extrapolation along the training direction: if $\theta_0 \to \theta_n$ is approximately linear and confined to a low-dimensional subspace, then $\theta_0 + \beta \cdot (\theta_n - \theta_0)$ with $\beta > 1$ is a principled estimate of a more capable future checkpoint. Whereas model soups exploit linearity through interpolation ($\beta \in [0,1]$) for robustness, RISE exploits the same structure through extrapolation ($\beta > 1$) for capability improvement. RISE does not, however, require strict linearity: confinement to a low-dimensional subspace limits how far extrapolation can deviate from the true trajectory, and our own runs confirm this---three directions capture ${\sim}87\%$ of the variance (Appendix~\ref{app:linearity}).

\subsection{Joint RLVR and OPD Training}
\label{sec:bg-combine}

RLVR and OPD provide complementary signals, and recent work seeks to combine them. The OPD gradient is dense but bounded by teacher quality, while the RLVR gradient is sparse but can drive the student beyond the teacher's capability ceiling~\citep{song2026survey}. The most direct approach augments the policy gradient objective with an auxiliary KL divergence term toward a teacher distribution~\citep{xu2025kdrl,ramos2026combining,zhang2026reinforcement}. ExOPD~\citep{yang2026learning} and others~\citep{wang2026openclaw,oh2026kl} formalize an equivalence between standard OPD and dense KL-constrained RL, where the log-probability ratio between teacher and student serves as a per-token reward. RLSD~\citep{yang2026self} uses OPD to reweight the advantage from GRPO for finer-grained credit assignment. An alternative sequential strategy allocates RL to the teacher for capability discovery and OPD to the student for compression~\citep{xu2026beyond}. Whether combining with an external teacher or a privileged self-teacher, all these approaches take the teacher as given and focus on how to best integrate its signal with RL. RISE departs from this paradigm by eliminating the need for an external or privileged-conditioned teacher altogether: the teacher is constructed by extrapolating the model's own training trajectory.

Among these, ExOPD's reward extrapolation resembles RISE structurally: its optimal policy is the same geometric mixture as RISE's logit-space formulation (Eq.~\ref{eq:logit-space}). The key distinction is not merely teacher source but teacher dynamics. ExOPD fixes a static external policy as teacher before distillation begins; no matter how well the student trains, the teacher gap $J(\pi^*) - J(\pi_T)$ in Theorem~\ref{thm:suboptimality} remains constant---a hard ceiling. RISE's teacher is non-stationary by construction: as the student improves via RLVR, the displacement vector updates, and the extrapolated teacher advances in lockstep. This eliminates the static teacher ceiling and converts OPD from a one-shot compression step into a recursive improvement loop. A direct empirical comparison is inapplicable because ExOPD requires an external stronger model---precisely what RISE eliminates.

\section{Method}
\label{sec:method}

Theorem~\ref{thm:suboptimality} (Appendix~\ref{app:theory}) formalizes \S\ref{sec:intro}'s intuition: the optimal teacher for $\pi_{\theta_n}$ is $\pi^*$, but since $\pi^*$ is unknown, RISE approximates it by extrapolating the model's own RLVR-grounded training trajectory, then distills that teacher's per-token distribution into the student.

\subsection{Self-Extrapolated Policy Distillation}
\label{sec:method-rise}

Consider a policy $\pi_{\theta_n}$ updated by RLVR to $\pi_{\theta_{n+1}'}$. Let $\varphi$ be a representation map into a vector space where linear operations are meaningful. The self-extrapolated teacher is defined as:
\begin{equation}
    \varphi(\pi_{\text{future}}) = \varphi(\pi_{\theta_n}) + \beta \cdot \big(\varphi(\pi_{\theta_{n+1}'}) - \varphi(\pi_{\theta_n})\big), \quad \beta > 1.
    \label{eq:unified-extrapolation}
\end{equation}
The extrapolation scale $\beta$ controls how far beyond the current policy we project: $\beta = 1$ recovers $\pi_{\theta_{n+1}'}$, and $\beta > 1$ continues along the improvement direction. The choice of $\varphi$ yields a family of instantiations with different computational and statistical properties.

\textbf{Weight-space extrapolation ($\varphi = \theta$).~}
Setting $\varphi$ to the identity on parameters gives
\begin{equation}\textstyle
    \theta_{\text{future}} = \theta_n + \beta \cdot (\theta_{n+1}' - \theta_n),
    \label{eq:weight-space}
\end{equation}%
which is precisely task arithmetic~\citep{ilharco2022editing} with an extrapolation coefficient. The teacher is then the model defined by $\theta_{\text{future}}$.

\textbf{Logit-space extrapolation ($\varphi = \log \pi$).~}
Let $s_t \triangleq (x, y_{<t})$ denote the context at position $t$. Setting $\varphi$ to map each policy to its per-token output log-probabilities gives:
\begin{equation}
    \log \pi_{\text{future}}(\cdot \mid s_t) = \log \pi_{\theta_n}(\cdot \mid s_t) + \beta \cdot \big(\log \pi_{\theta_{n+1}'}(\cdot \mid s_t) - \log \pi_{\theta_n}(\cdot \mid s_t)\big) + \text{const},
    \label{eq:logit-space}
\end{equation}
where the constant ensures normalization. Equivalently, $\pi_{\text{future}} \propto \pi_{\theta_n}^{1-\beta} \cdot \pi_{\theta_{n+1}'}^{\beta}$---a \emph{geometric mixture} that amplifies the probability ratio $\pi_{\theta_{n+1}'} / \pi_{\theta_n}$.

\begin{remark}[Connection to KL regularization]
\label{rem:kl}
Given the geometric mixture, the distillation loss decomposes as
\begin{equation}
    \KL(\pi_\theta \| \pi_{\text{future}}) = -(\beta-1)\,\KL(\pi_\theta \| \pi_{\theta_n}) + \beta\,\KL(\pi_\theta \| \pi_{\theta_{n+1}'}) + \log Z,
    \label{eq:kl-decomp}
\end{equation}
where $Z$ is the normalization constant of $\pi_{\text{future}}$, and both $\pi_{\theta_n}$ and $\pi_{\theta_{n+1}'}$ are treated as fixed (stop-gradient) when constructing $\pi_{\text{future}}$. For $\beta > 1$, the first term has a negative coefficient and is repulsive---it pushes the current policy away from the anchor, continuing the RLVR improvement direction; the second term regularizes it toward the post-RLVR checkpoint $\pi_{\theta_{n+1}'}$, preventing overshoot during OPD. Both terms operate at the \emph{token level}, providing fine-grained credit assignment that a scalar outcome reward cannot.
\end{remark}

\textbf{Relationships between instantiations.~}
Let $f(\theta)$ denote the mapping from parameters to output logits.
Weight-space extrapolation computes $f(\theta_n + \beta \cdot \Delta\theta)$, while logit-space computes $f(\theta_n) + \beta \cdot \Delta f$---the \emph{first-order Taylor approximation} of weight-space extrapolation around $\theta_n$.
The two coincide exactly when $f$ is linear; for neural networks they diverge. Weight-space produces a coherent model whose cross-position predictions are jointly consistent, at the cost of materializing parameters and one teacher forward pass per OPD step; logit-space needs no parameter manipulation and fixes the teacher before OPD begins.

\subsection{Distillation Loss Design}
\label{sec:method-design}

Computing a divergence over the full vocabulary is prohibitively expensive. Following prior works~\citep{hubotter2026reinforcement,zhao2026self}, we use a top-$K$ approximation. For logit-space extrapolation, the unbiased procedure would extrapolate over the full vocabulary and then select the top-$K$ of $\pi_{\text{future}}$; however, $\pi_{\text{future}}$ does not exist as a model---it is defined only through the logit-space formula---so obtaining its full-vocabulary logits would require keeping two extra model copies in memory or caching $O(V)$ logits per position. Instead, letting $S = \text{Top}_K(\pi_{\theta_{n+1}'})$, we project both $\pi_{\theta_{n+1}'}$ and $\pi_{\theta_n}$ onto the same $(K\!+\!1)$-simplex (retaining log-probabilities at $S$ and appending a tail bucket for the remaining mass, each renormalized), then apply Eq.~\ref{eq:logit-space} to all $K\!+\!1$ log-probabilities:
\begin{equation}
    \log \pi_{\text{future}}(v \mid \cdot) = \log \pi_{\theta_n}(v \mid \cdot) + \beta \cdot \big(\log \pi_{\theta_{n+1}'}(v \mid \cdot) - \log \pi_{\theta_n}(v \mid \cdot)\big), \quad v \in S \cup \{\text{tail}\}.
    \label{eq:topk-extrapolation}
\end{equation}
With $K\!=\!100$, the top-$K$ tokens account for essentially all probability mass under typical LLM distributions, so any bias from using $S$ instead of $\text{Top}_K(\pi_{\text{future}})$ is negligible (Appendix~\ref{app:topk-bias}). For weight-space extrapolation, a forward pass through $\theta_{\text{future}}$ (Eq.~\ref{eq:weight-space}) yields $\pi_{\text{future}}$ exactly, and we evaluate it at the student's top-$K$ indices with a tail bucket appended and renormalized.

In both cases, $\pi_{\text{future}}$ is held fixed (stop-gradient) throughout all OPD gradient steps. The RISE distillation loss is:
\begin{equation}\textstyle
    \mathcal{L}_{\text{RISE}} = \mathbb{E}_{y \sim \pi_{\theta}}\!\left[\sum_{t=1}^{T} \KL\!\left(\pi_{\theta}(\cdot \mid s_t) \;\|\; \mathrm{sg}\!\left[\pi_{\text{future}}(\cdot \mid s_t)\right]\right)\right],
    \label{eq:rise-loss}
\end{equation}
where the KL divergence is computed over the $(K\!+\!1)$-dimensional simplex.

\textbf{Choice of divergence.~} Our analysis (Remark~\ref{rem:kl}, Theorem~\ref{thm:suboptimality}) uses reverse KL for an exact trust-region decomposition. In practice we replace KL with the Jensen--Shannon divergence (JSD) in Eq.~\ref{eq:rise-loss}, which is bounded by $\log 2$ thus avoiding the numerical instability of reverse KL. Since $\mathrm{JSD}(P \| Q) \leq \tfrac{1}{2}\KL(P \| Q)$, the trust-region guarantees from the KL analysis carry over as conservative bounds. The two losses also share the same unique minimizer ($\pi_\theta = \pi_{\text{future}}$), so the teacher targeted is identical.

\subsection{Training Procedure}
\label{sec:method-procedure}

Each RISE iteration alternates between two phases:
\begin{enumerate}[leftmargin=*, itemsep=1pt, topsep=0pt]
    \item \textbf{RLVR phase.} Sample rollouts from the current policy $\pi_{\theta_n}$, compute outcome rewards, and update the policy via a policy gradient method (e.g., GRPO) to obtain $\theta_{n+1}'$.
    \item \textbf{OPD phase.} Construct $\pi_{\text{future}}$ from $\pi_{\theta_{n+1}'}$ and anchor $\pi_{\theta_n}$ via logit-space (Eq.~\ref{eq:logit-space}) or weight-space (Eq.~\ref{eq:weight-space}) extrapolation. Distill $\pi_{\text{future}}$ into $\pi_{\theta_{n+1}'}$ by minimizing $\mathcal{L}_{\text{RISE}}$ (Eq.~\ref{eq:rise-loss}), yielding $\theta_{n+1}$.
\end{enumerate}
The two phases are complementary: RLVR discovers capability improvements via sparse outcome rewards, while OPD compresses the extrapolated teacher's token-level distribution into the current policy (Algorithms~\ref{alg:rise} and~\ref{alg:rise-weight}). Both phases operate on the \emph{same} set of rollouts---the responses $y \sim \pi_{\theta_n}$ sampled for RLVR are reused for OPD, which therefore adds no additional sampling cost. Reusing pre-RLVR rollouts introduces mild off-policy-ness (contexts come from pre-RLVR), but the distributional loss is well-defined for any input sequence without importance sampling correction. We confirm empirically that this shift has negligible impact (\S\ref
{sec:exp-ablations}).

Algorithm~\ref{alg:rise} gives the full pseudocode for logit-space extrapolation, where the teacher distribution is built by extrapolating cached top-$K$ logits from the post-RLVR checkpoint and the anchor without materializing a separate model. Algorithm~\ref{alg:rise-weight} gives the weight-space variant, where the extrapolated parameters $\theta_{\text{future}}$ are materialized explicitly and a forward pass through them produces the teacher distribution. Both variants share the same two-phase structure and differ only in how $\pi_{\text{future}}$ is computed.

\begin{algorithm}[t]
\caption{RISE (logit-space). The teacher is built \emph{once} per iteration, before the OPD loop.}
\label{alg:rise}
\begin{algorithmic}[1]
\REQUIRE Initial policy $\pi_{\theta_0}$, extrapolation scale $\beta_0$, top-$K$, anchor rate $\eta$, iterations $N$
\STATE Initialize anchor $\theta_{\text{anchor}} \leftarrow \theta_0$
\FOR{$n = 0, 1, \ldots, N-1$}
    \STATE $\beta_n \leftarrow 1 + (\beta_0 - 1) \cdot (1 - n/N)$ \hfill $\triangleright$ Extrapolation decay
    \STATE \textbf{// RLVR phase}
    \STATE Sample rollouts $\mathcal{D} = \{(x_i, y_i)\} \sim \pi_{\theta_n}$, compute rewards $R(x_i, y_i)$
    \STATE $\theta_{n+1}' \leftarrow \text{PolicyGradient}\big(\theta_n, \{(x_i, y_i, R_i)\}\big)$
    \STATE \textbf{// Teacher construction} \hfill $\triangleright$ two forward passes over $\mathcal{D}$
    \STATE At every position of $\mathcal{D}$: set $S \leftarrow \text{Top}_K(\pi_{\theta_{n+1}'})$ and project $\pi_{\theta_{n+1}'}, \pi_{\theta_{\text{anchor}}}$ onto $S \cup \{\text{tail}\}$
    \STATE Cache $\log \pi_{\text{future}} \leftarrow \log \pi_{\theta_{\text{anchor}}} + \beta_n \cdot (\log \pi_{\theta_{n+1}'} - \log \pi_{\theta_{\text{anchor}}})$ \hfill $\triangleright$ Eq.~\ref{eq:topk-extrapolation}
    \STATE \textbf{// OPD phase}
    \STATE $\theta \leftarrow \theta_{n+1}'$
    \FOR{each minibatch $\mathcal{B} \subset \mathcal{D}$}
        \STATE Retrieve cached $\pi_{\text{future}}$ at the positions of $\mathcal{B}$ \hfill $\triangleright$ no teacher forward pass
        \STATE $\theta \leftarrow \mathrm{Optimizer}\!\left(\theta,\; \nabla_{\theta}\,\mathcal{L}_{\text{RISE}}(\mathcal{B})\right)$ \hfill $\triangleright$ Eq.~\ref{eq:rise-loss}
    \ENDFOR
    \STATE $\theta_{n+1} \leftarrow \theta$
    \STATE Update anchor: $\theta_{\text{anchor}} \leftarrow (1\!-\!\eta)\,\theta_{\text{anchor}} + \eta\,\theta_{n+1}$ \hfill $\triangleright$ $\eta\!=\!1$: previous ckpt; $\eta\!<\!1$: EMA
\ENDFOR
\RETURN $\pi_{\theta_N}$
\end{algorithmic}
\end{algorithm}

\begin{algorithm}[t]
\caption{RISE (weight-space). The teacher is evaluated \emph{inside} the OPD loop.}
\label{alg:rise-weight}
\begin{algorithmic}[1]
\REQUIRE Initial policy $\pi_{\theta_0}$, extrapolation scale $\beta_0$, top-$K$, anchor rate $\eta$, iterations $N$
\STATE Initialize anchor $\theta_{\text{anchor}} \leftarrow \theta_0$
\FOR{$n = 0, 1, \ldots, N-1$}
    \STATE $\beta_n \leftarrow 1 + (\beta_0 - 1) \cdot (1 - n/N)$ \hfill $\triangleright$ Extrapolation decay
    \STATE \textbf{// RLVR phase}
    \STATE Sample rollouts $\mathcal{D} = \{(x_i, y_i)\} \sim \pi_{\theta_n}$, compute rewards $R(x_i, y_i)$
    \STATE $\theta_{n+1}' \leftarrow \text{PolicyGradient}\big(\theta_n, \{(x_i, y_i, R_i)\}\big)$
    \STATE \textbf{// Teacher construction}
    \STATE Materialize $\theta_{\text{future}} \leftarrow \theta_{\text{anchor}} + \beta_n \cdot (\theta_{n+1}' - \theta_{\text{anchor}})$ \hfill $\triangleright$ Eq.~\ref{eq:weight-space}
    \STATE \textbf{// OPD phase}
    \STATE $\theta \leftarrow \theta_{n+1}'$
    \FOR{each minibatch $\mathcal{B} \subset \mathcal{D}$}
        \STATE At every position of $\mathcal{B}$: set $S \leftarrow \text{Top}_K(\pi_{\theta})$ from the student forward pass
        \STATE Forward pass through $\theta_{\text{future}}$ to get $\pi_{\text{future}}$ on $S \cup \{\text{tail}\}$ \hfill $\triangleright$ one teacher forward per step
        \STATE $\theta \leftarrow \mathrm{Optimizer}\!\left(\theta,\; \nabla_{\theta}\,\mathcal{L}_{\text{RISE}}(\mathcal{B})\right)$ \hfill $\triangleright$ Eq.~\ref{eq:rise-loss}
    \ENDFOR
    \STATE $\theta_{n+1} \leftarrow \theta$
    \STATE Update anchor: $\theta_{\text{anchor}} \leftarrow (1\!-\!\eta)\,\theta_{\text{anchor}} + \eta\,\theta_{n+1}$
\ENDFOR
\RETURN $\pi_{\theta_N}$
\end{algorithmic}
\end{algorithm}

\textbf{Extrapolation Decay.~}
A fixed $\beta$ is unsuitable across the full training trajectory. Under the linear-trajectory model (Proposition~\ref{prop:extrapolation-gap}, Appendix~\ref{app:theory}), the safe $\beta$ range narrows as the policy approaches the optimum: beyond a $\beta$-threshold the extrapolated teacher overshoots $\pi^*$ and the suboptimality bound worsens. We therefore apply a monotonically decreasing schedule $\beta_n \to 1$ (e.g., linear: $\beta_n = 1 + (\beta_0 - 1)(1 - n/N)$). Large $\beta_n$ early in training extrapolates aggressively when the policy is far from optimal; as training progresses, decreasing $\beta_n$ keeps the teacher within the valid range.

\textbf{Anchor Dynamics.~}
The anchor $\theta_{\text{anchor}}$---the reference checkpoint used to compute the displacement vector---is by default the previous checkpoint $\theta_n$. Alternatively, an EMA anchor $\theta_{\text{anchor}} \leftarrow (1-\eta)\theta_{\text{anchor}} + \eta\,\theta_{n+1}$ both smooths and enlarges the displacement: the anchor lags behind the current policy, averaging the direction over multiple iterations while increasing $\|\theta_{n+1}' - \theta_{\text{anchor}}\|$. The previous-checkpoint anchor is the special case $\eta=1$. We ablate this choice in \S\ref{sec:exp-ablations}.

\textbf{Role of OPD.~}
Why distill from $\pi_{\text{future}}$ rather than adopt it directly as the next policy? The extrapolated point lies beyond the policy's trust region: adopting it wholesale risks degenerate behavior at large $\beta$, and amplifies any noise in the RLVR step. OPD instead acts as a \emph{trust-region projection}---it moves the policy toward $\pi_{\text{future}}$'s token-level distribution while the divergence term in $\mathcal{L}_{\text{RISE}}$ keeps the update anchored near $\pi_{\theta_{n+1}'}$; we validate this choice and the need for RLVR grounding in \S\ref{sec:exp-analysis}. Importantly, RISE does not introduce new task information beyond what RLVR provides---no external data or model enters the system. Rather, it exploits the model's own inductive structure to convert a sparse outcome-induced parameter update into a dense token-level target, redistributing the same training signal into a form that enables finer-grained credit assignment.

\section{Experiments}
\label{sec:experiments}

We evaluate RISE across multiple model scales, architectures, and domains to answer five questions: (1)~Does RISE improve over RLVR-only training and existing OPSD methods? (2)~Does the improvement hold across model scales and families? (3)~Does RISE preserve performance on out-of-distribution tasks while improving in-domain accuracy? (4)~What mechanisms drive the improvement? (5)~Which design choices matter most?

\subsection{Setup}
\label{sec:exp-setup}

\textbf{Models and datasets.~}
We evaluate RISE on four task families. \emph{Mathematical reasoning}: Qwen3-8B, Qwen3-1.7B, and Qwen3-1.7B-Base~\citep{qwen3technicalreport} on DAPOMath~\citep{yu2026dapo}, and OLMo3-7B-Instruct-SFT~\citep{olmo2025olmo3} on OpenR1-Math-46K~\citep{yan2026learning}, with GPQA-Diamond, IFEval, and MMLU-Pro as out-of-distribution checks. \emph{Multi-domain (math + STEM)}: Qwen3-4B-Base on a mixed corpus following Guru~\citep{cheng2026revisiting}, using only their STEM data and replacing their math split with DAPOMath (the original problems are low complexity). \emph{Code generation}: Qwen3-8B-Base on Skywork-OR1-Code~\citep{he2025skywork}. \emph{Agentic tasks}: Qwen2.5-3B-Instruct~\citep{qwen2.5} on ALFWorld~\citep{shridhar2020alfworld} and WebShop~\citep{yao2022webshop} following GIGPO~\citep{feng2025group}. Evaluation benchmarks, sample counts, and full training details are in Appendix~\ref{sec:app-exp-setup}.

\textbf{Baselines.~}
We report the base/SFT model to establish the starting point and GRPO~\citep{guo2025deepseek} as the RLVR-only baseline. To isolate the contribution of RISE's teacher construction, we compare against three methods that also pair GRPO's sequence-level signal with per-token self-distillation from the \emph{same} privileged teacher---the student conditioned on a sibling correct solution---but differ in how they integrate it: GRPO+SDPO~\citep{hubotter2026reinforcement} adds an auxiliary KL loss, SDAR~\citep{lu2026self} gates the GRPO advantage by the teacher--student probability gap, and RLSD~\citep{yang2026self} reweights advantage estimates for finer-grained credit assignment. We exclude external-teacher OPD, which requires a separate stronger model, whereas RISE and all OPSD baselines use only the model itself (Appendix~\ref{sec:app-baseline}).

\textbf{RISE configurations.~}
We report both logit-space and weight-space RISE. Default hyperparameters: $\beta_0 = 1.2$, linear decay to $\beta_N = 1$, $K = 100$ ($K = 20$ for code, where the output distribution is more peaked); for the anchor we use $\eta = 0.1$ (EMA) for Qwen models and $\eta = 1$ (previous checkpoint) for OLMo (ablated in \S\ref{sec:exp-ablations}); the distillation loss uses Jensen--Shannon divergence (cf.\ \S\ref{sec:method-design}). All other training hyperparameters (learning rate, batch size, rollout length) match the GRPO baseline for fair comparison. All methods are trained for one epoch on the respective training set.

\subsection{Main Results}
\label{sec:exp-results}

\begin{table}[t]
\centering
\caption{\textbf{Mathematical reasoning results.} Accuracy (\%) across model configurations. We evaluate on in-domain math benchmarks and out-of-distribution (OOD) tasks to assess preservation. Best in \textbf{bold}, second best \underline{underlined}.}
\label{tab:main-math}
\small
\setlength{\tabcolsep}{2.5pt}
\begin{tabular}{l|cccccc>{\columncolor{yellow!15}}c|ccc>{\columncolor{yellow!15}}c}
\toprule
\multirow{2}{*}{\textbf{Method}} & \multicolumn{7}{c|}{\textbf{In-Domain Math}} & \multicolumn{4}{c}{\textbf{OOD}} \\
& {\scriptsize MATH500} & {\scriptsize AIME24} & {\scriptsize AIME25} & {\scriptsize AMC23} & {\scriptsize Minerva} & {\scriptsize OlyBench} & {\scriptsize Avg.} & {\scriptsize GPQA} & {\scriptsize IFEval} & {\scriptsize MMLU} & {\scriptsize Avg.} \\
\midrule
\rowcolor{gray!15} \multicolumn{12}{l}{\textbf{Qwen3-8B (DAPOMath)}} \\
Base & 73.2 & 27.1 & 23.1 & 66.1 & 21.1 & 48.6 & 43.2 & 50.8 & 82.4 & 68.4 & 67.2 \\
GRPO & 83.8 & 54.4 & 42.9 & 89.4 & 31.5 & 57.9 & 60.0 & 57.0 & 81.5 & 73.3 & 70.6 \\
GRPO+SDPO & 83.2 & 42.3 & 32.7 & 84.4 & 30.6 & 62.5 & 55.9 & 55.7 & \underline{82.8} & 69.4 & 69.3 \\
SDAR & \underline{84.5} & 50.2 & 38.3 & 89.4 & 31.2 & \textbf{64.7} & 59.7 & 55.5 & 82.4 & 70.1 & 69.3 \\
RLSD & 83.0 & 53.1 & 40.0 & 90.6 & 31.3 & 60.5 & 59.8 & 56.2 & 82.6 & 72.6 & 70.5 \\
\textbf{RISE (logit)} & \textbf{84.8} & \underline{56.9} & \textbf{46.7} & \underline{91.6} & \underline{32.5} & 62.4 & \underline{62.5} & \underline{58.2} & 81.3 & \textbf{74.3} & \underline{71.3} \\
\textbf{RISE (weight)} & 84.4 & \textbf{58.1} & \underline{45.8} & \textbf{91.9} & \textbf{32.6} & \underline{63.3} & \textbf{62.7} & \textbf{59.0} & \textbf{83.4} & \underline{73.5} & \textbf{72.0} \\
\midrule
\rowcolor{gray!15} \multicolumn{12}{l}{\textbf{Qwen3-1.7B (DAPOMath)}} \\
Base & 64.6 & 11.5 & 12.1 & 41.9 & 17.6 & 40.3 & 31.3 & 34.9 & 67.7 & 48.3 & 50.3 \\
GRPO & 75.6 & 30.0 & 26.3 & 65.6 & 23.9 & 51.1 & 45.4 & 32.3 & 68.8 & 53.4 & 51.5 \\
GRPO+SDPO & \underline{76.0} & 21.5 & 23.1 & 65.0 & 23.1 & 50.3 & 43.2 & \underline{34.9} & 69.0 & 52.0 & 52.0 \\
SDAR & 75.0 & 26.0 & 26.0 & 71.3 & 23.6 & 50.8 & 45.5 & \textbf{35.3} & \underline{69.1} & 53.2 & \underline{52.5} \\
RLSD & 75.3 & 26.7 & 25.2 & 59.4 & 23.0 & 49.8 & 43.2 & 33.4 & 67.5 & \underline{54.1} & 51.7 \\
\textbf{RISE (logit)} & 75.2 & \textbf{36.3} & \textbf{33.3} & \textbf{78.8} & \underline{24.9} & \underline{52.9} & \textbf{50.2} & 34.8 & 68.2 & 53.7 & 52.2 \\
\textbf{RISE (weight)} & \textbf{77.3} & \underline{32.9} & \underline{31.3} & \underline{75.0} & \textbf{25.7} & \textbf{53.3} & \underline{49.2} & 34.8 & \textbf{70.2} & \textbf{56.3} & \textbf{53.8} \\
\midrule
\rowcolor{gray!15} \multicolumn{12}{l}{\textbf{OLMo3-7B-Instruct-SFT (OpenR1)}} \\
Base & 57.0 & 6.0 & 8.1 & 43.8 & 19.6 & 31.5 & 27.7 & 33.0 & 77.8 & 35.4 & 48.7 \\
GRPO & 77.5 & 30.2 & 28.3 & 70.6 & 25.2 & 53.5 & 47.6 & 31.0 & 76.0 & \underline{46.6} & 51.2 \\
GRPO+SDPO & 77.7 & 33.8 & 28.5 & 70.6 & 25.3 & 53.4 & 48.2 & 37.4 & 77.1 & 46.1 & 53.5 \\
SDAR & 79.8 & 35.2 & 27.1 & 73.8 & 25.3 & 56.3 & 49.6 & \underline{39.0} & 77.8 & 44.6 & \underline{53.8} \\
RLSD & 77.2 & 28.3 & 24.0 & 68.1 & 24.6 & 52.9 & 45.9 & 37.6 & \underline{78.4} & 44.1 & 53.4 \\
\textbf{RISE (logit)} & \textbf{82.9} & \textbf{46.9} & \textbf{36.0} & \textbf{83.1} & \underline{27.6} & \textbf{61.9} & \textbf{56.4} & \textbf{40.7} & \textbf{79.3} & 46.4 & \textbf{55.5} \\
\textbf{RISE (weight)} & \underline{82.8} & \underline{42.5} & \underline{32.5} & \underline{76.9} & \textbf{27.8} & \underline{59.6} & \underline{53.7} & 37.1 & 77.1 & \textbf{46.9} & 53.7 \\
\bottomrule
\end{tabular}
\end{table}

\textbf{Mathematical Reasoning.~}
Table~\ref{tab:main-math} presents results across three model configurations spanning different scales (1.7B--8B), architectures (Qwen, OLMo), and training corpora (DAPOMath, OpenR1).

\emph{RISE consistently outperforms all baselines.} Both RISE variants beat every baseline on in-domain Math Avg in all three settings, and at least one ranks first or second on every individual in-domain benchmark. The gains are most pronounced on competition benchmarks: on OLMo3-7B, RISE (logit) improves AIME'24 from 30.2 to 46.9 (+16.7) and Math Avg from 47.6 to 56.4 (+8.8); on Qwen3-1.7B, RISE (logit) lifts Math Avg from 45.4 to 50.2 (+4.8); and on Qwen3-8B, RISE (weight) lifts it from 60.0 to 62.7 (+2.7). Neither extrapolation space consistently dominates, suggesting the gains stem from the extrapolation principle itself. Multi-seed experiments confirm reproducibility: RISE (weight) Math Avg = $62.4 \pm 0.2$ vs.\ GRPO $60.1 \pm 0.2$ on Qwen3-8B, and RISE (logit) $49.6 \pm 0.4$ vs.\ GRPO $45.1 \pm 0.4$ on Qwen3-1.7B, across three seeds (Appendix~\ref{sec:app-seed-variance}).

\emph{Privileged-conditioning baselines provide limited or negative gains.} GRPO+SDPO, which jointly optimizes the GRPO loss and a KL divergence toward the privileged teacher, underperforms GRPO on both Qwen3-8B (55.9 vs.\ 60.0) and Qwen3-1.7B (43.2 vs.\ 45.4). SDAR and RLSD fare better but remain within about two points of GRPO in most settings. This pattern is consistent with our analysis in \S\ref{sec:background}: privileged-conditioning teachers are limited by ICL capability, capping the benefit of per-token supervision regardless of integration strategy.

\emph{Out-of-distribution preservation.} RISE maintains or slightly improves OOD performance across all models---Qwen3-8B OOD Avg rises from 70.6 (GRPO) to 72.0 (RISE weight), OLMo3-7B from 51.2 to 55.5 (RISE logit)---so token-level refinement does not degrade general capabilities.

\begin{figure}[h]
\centering
\captionsetup[subfigure]{skip=2pt}
\begin{subfigure}[t]{0.325\linewidth}
    \centering
    \includegraphics[width=\linewidth]{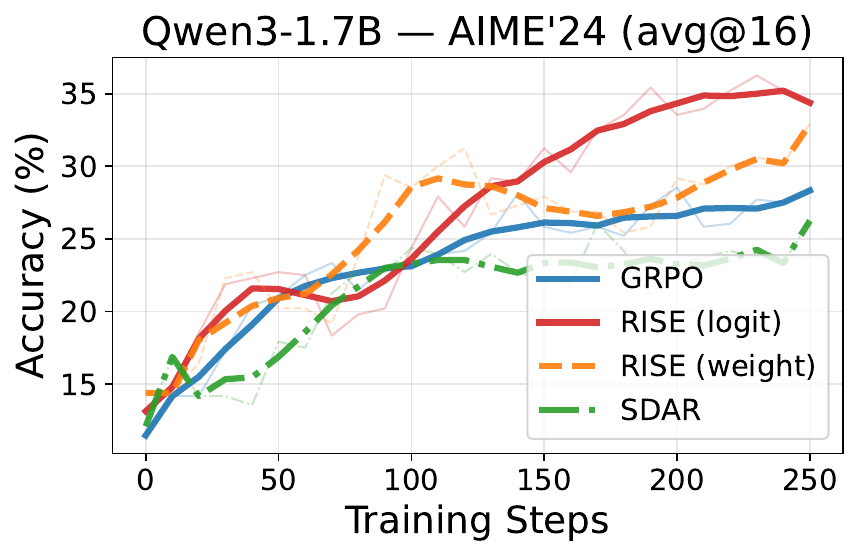}
    \caption{1.7B AIME'24}
    \label{fig:conv-1.7b-aime24}
\end{subfigure}
\hfill
\begin{subfigure}[t]{0.325\linewidth}
    \centering
    \includegraphics[width=\linewidth]{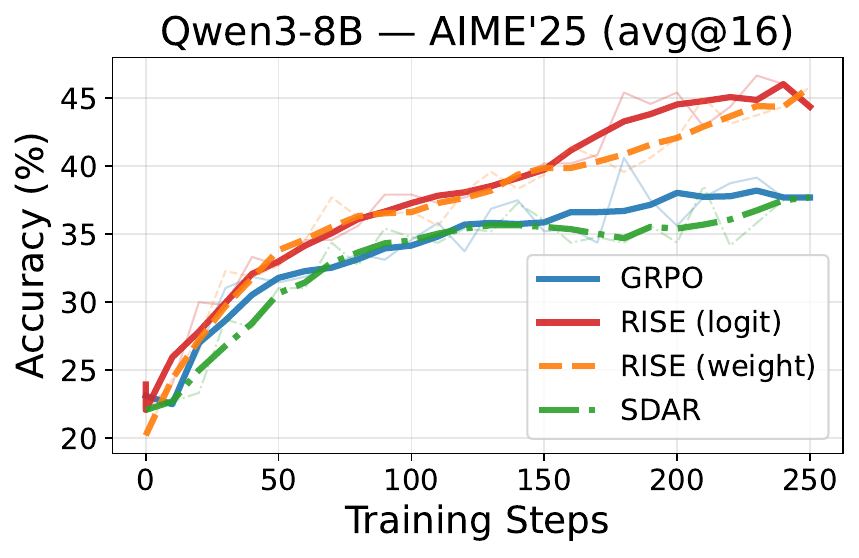}
    \caption{8B AIME'25}
    \label{fig:conv-8b-aime25}
\end{subfigure}
\hfill
\begin{subfigure}[t]{0.325\linewidth}
    \centering
    \includegraphics[width=\linewidth]{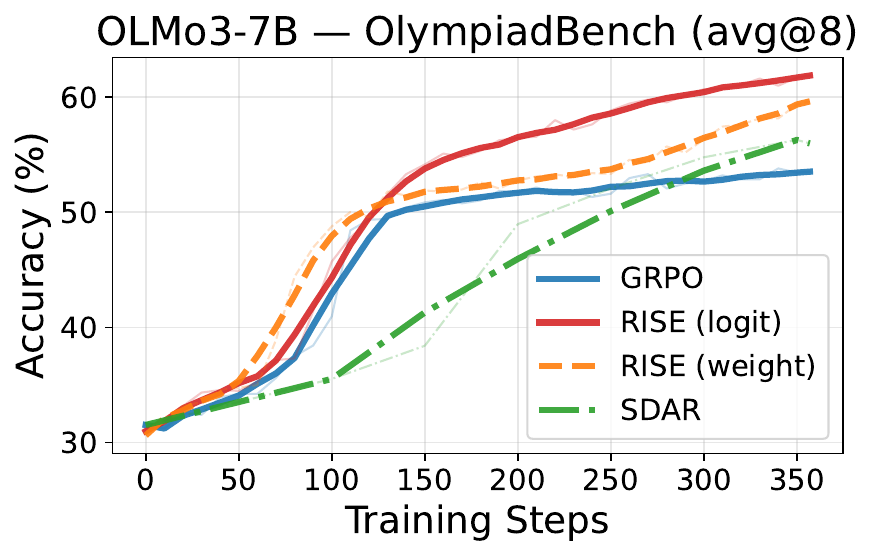}
    \caption{OLMo OlyBench}
    \label{fig:conv-olmo-olympiad}
\end{subfigure}
\vspace{-6pt}
\caption{RISE reaches higher accuracy in fewer steps across three model scales, demonstrating improved sample efficiency.}
\label{fig:convergence-math}
\vspace{-8pt}
\end{figure}

\emph{Sample efficiency.} RISE also learns faster: Figure~\ref{fig:convergence-math} plots evaluation accuracy at successive checkpoints across three configurations, and RISE reaches higher accuracy earlier throughout, with the gap widest in early training where the extrapolated teacher supplies dense signal before RL advantage estimates stabilize (further curves in Appendix~\ref{sec:app-convergence-curves}).

\begin{table}[t]
\centering
\caption{\textbf{Multi-domain results (Qwen3-4B-Base, mixed math + STEM training).} Accuracy (\%) on math reasoning and STEM benchmarks. Best in \textbf{bold}, second best \underline{underlined}.}
\label{tab:multidomain}
\small
\setlength{\tabcolsep}{2.5pt}
\begin{tabular}{l|cccccc>{\columncolor{yellow!15}}c|cccc>{\columncolor{yellow!15}}c}
\toprule
\multirow{2}{*}{\textbf{Method}} & \multicolumn{7}{c|}{\textbf{Math}} & \multicolumn{5}{c}{\textbf{STEM}} \\
& {\tiny MATH500} & {\tiny AIME24} & {\tiny AIME25} & {\tiny AMC23} & {\tiny Minerva} & {\tiny OlyBench} & {\tiny Avg.} & {\tiny GPQA} & {\tiny SuperGPQA} & {\tiny MMLU} & {\tiny ThrmQA} & {\tiny Avg.} \\
\midrule
Base & 48.2 & 9.8 & 10.4 & 35.0 & 12.1 & 27.5 & 23.8 & 14.1 & 8.0 & 6.6 & 27.0 & 13.9 \\
GRPO & 69.1 & 23.8 & 20.6 & 59.4 & \underline{26.6} & 42.0 & 40.2 & 42.4 & \underline{32.5} & 59.5 & 47.8 & 45.5 \\
GRPO+SDPO & 70.5 & 24.8 & 20.4 & 68.1 & 25.9 & 43.1 & 42.1 & \underline{44.7} & 31.0 & \underline{61.7} & 50.5 & 47.0 \\
SDAR & 69.1 & 22.9 & 20.0 & 65.6 & 23.7 & 41.2 & 40.4 & 40.6 & 30.3 & 58.7 & 49.1 & 44.7 \\
RLSD & 68.9 & 23.1 & 20.0 & 60.0 & 24.6 & 39.9 & 39.4 & 33.7 & 28.8 & 55.1 & 46.5 & 41.0 \\
\textbf{RISE (logit)} & \textbf{71.9} & \underline{25.6} & \underline{22.1} & \underline{70.6} & 25.9 & \underline{43.8} & \underline{43.3} & 44.1 & 32.1 & \textbf{61.9} & \underline{51.3} & \underline{47.4} \\
\textbf{RISE (weight)} & \underline{71.3} & \textbf{28.8} & \textbf{25.4} & \textbf{71.3} & \textbf{26.9} & \textbf{44.9} & \textbf{44.8} & \textbf{44.7} & \textbf{32.9} & 60.9 & \textbf{51.5} & \textbf{47.5} \\
\bottomrule
\end{tabular}
\end{table}

\textbf{Multi-Domain Results.~}
Table~\ref{tab:multidomain} tests whether RISE holds when the displacement aggregates improvements across heterogeneous domains. On Qwen3-4B-Base trained with mixed math and STEM data, RISE (weight) achieves the highest Math Avg (44.8 vs.\ GRPO's 40.2) and STEM Avg (47.5 vs.\ 45.5), improving competition math (AIME'24: 28.8, +5.0) while simultaneously lifting STEM (TheoremQA: 51.5, +3.7)---extrapolation along a multi-domain trajectory does not dilute gains in any single domain. GRPO+SDPO is more competitive here (STEM Avg 47.0) than in the math-only setting, consistent with STEM providing more informative privileged context~\citep{hubotter2026reinforcement}, yet RISE still outperforms it without privileged information.

\begin{figure}[h]
\centering
\captionsetup[subfigure]{skip=2pt}
\begin{minipage}[t]{0.49\textwidth}
    \centering
    \begin{subfigure}[t]{0.48\linewidth}
        \centering
        \includegraphics[width=\linewidth]{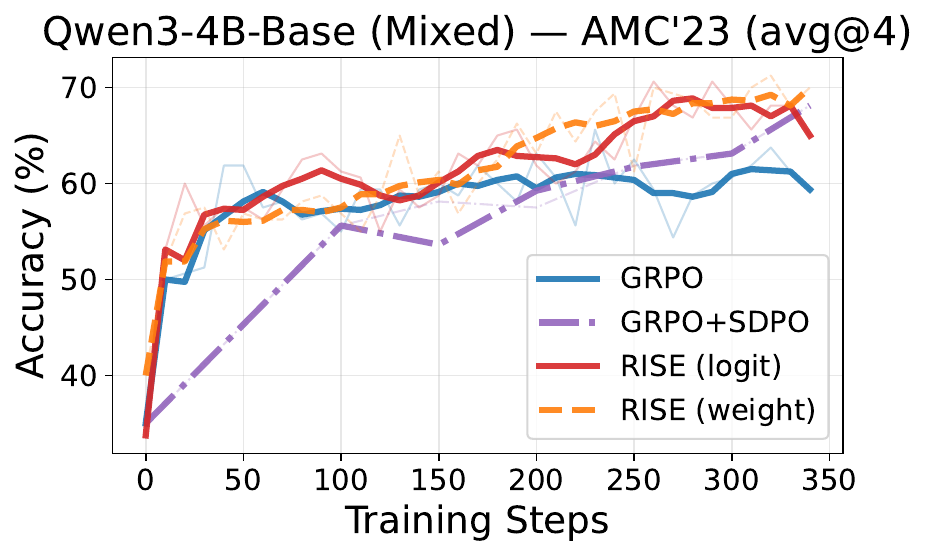}
        \caption{AMC'23}
        \label{fig:conv-4b-amc23}
    \end{subfigure}
    \hfill
    \begin{subfigure}[t]{0.5\linewidth}
        \centering
        \includegraphics[width=\linewidth]{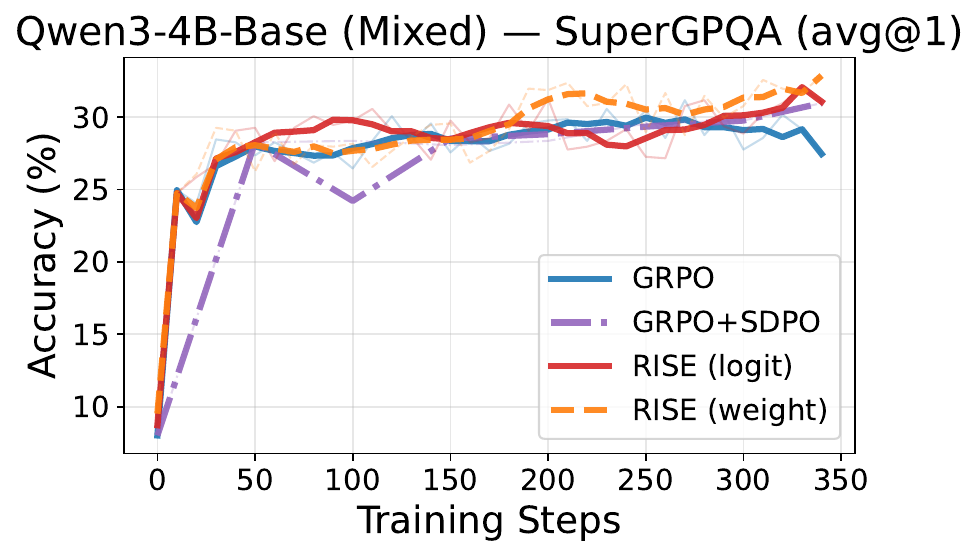}
        \caption{SuperGPQA}
        \label{fig:conv-4b-supergpqa}
    \end{subfigure}
    \vspace{-4pt}
    \caption{RISE reaches higher accuracy on both math and STEM benchmarks.}
    \label{fig:convergence-mixed}
\end{minipage}
\hfill
\begin{minipage}[t]{0.49\textwidth}
    \centering
    \begin{subfigure}[t]{0.5\linewidth}
        \centering
        \includegraphics[width=\linewidth]{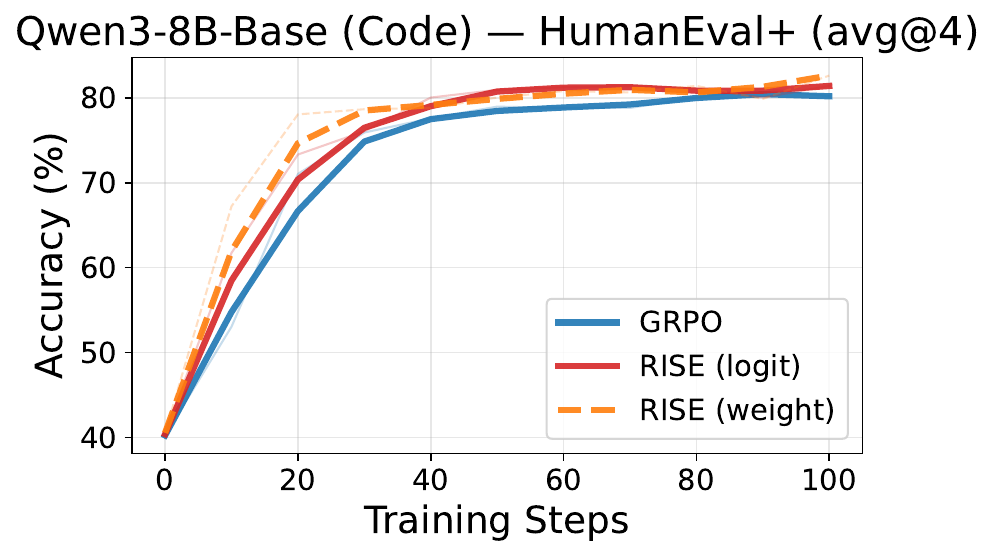}
        \caption{HumanEval+}
        \label{fig:code-humaneval}
    \end{subfigure}
    \hfill
    \begin{subfigure}[t]{0.48\linewidth}
        \centering
        \includegraphics[width=\linewidth]{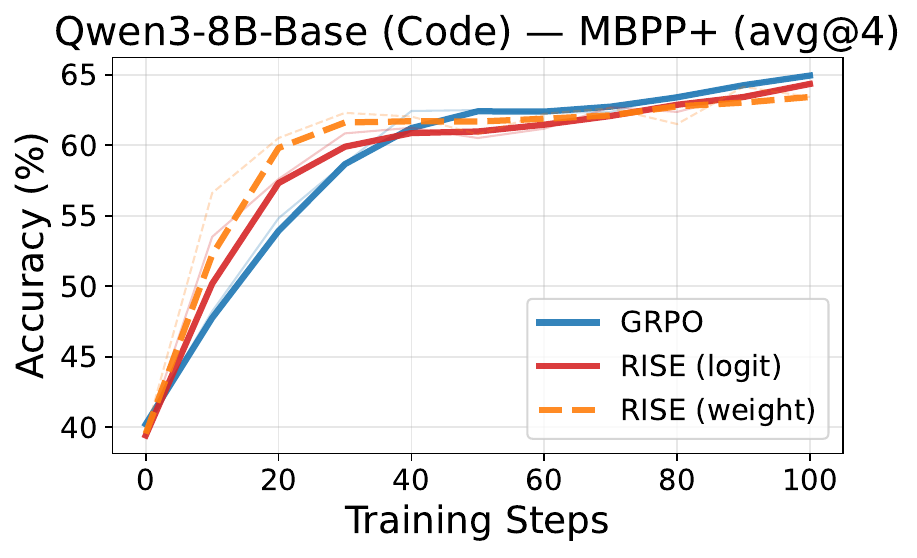}
        \caption{MBPP+}
        \label{fig:code-mbpp}
    \end{subfigure}
    \vspace{-4pt}
    \caption{RISE converges faster than GRPO on code generation.}
    \label{fig:code-convergence}
\end{minipage}
\vspace{-8pt}
\end{figure}

The sample-efficiency advantage extends to this setting. On AMC'23 (Fig.~\ref{fig:conv-4b-amc23}) both RISE variants stay above GRPO throughout training, with RISE (weight) reaching 70.0\% while GRPO rises mid-training but declines to 59.4\%; on SuperGPQA (Fig.~\ref{fig:conv-4b-supergpqa}) RISE (weight) holds a consistent edge, so the signal benefits both math and STEM.

\textbf{Code Generation.~}
We evaluate Qwen3-8B-Base on code generation (Fig.~\ref{fig:code-convergence}). Both RISE variants converge faster than GRPO early and reach comparable final accuracy---RISE (logit) matches GRPO's final HumanEval+ accuracy at step~50 versus step~90---so the extrapolation principle transfers beyond mathematical reasoning. Gains are more modest than in math, likely because all methods saturate quickly on these benchmarks. Full curves (avg@4 and pass@4) are in Appendix~\ref{sec:app-code}.

\begin{wraptable}{r}{0.33\textwidth}
\vspace{-14pt}
\centering
\caption{Success rate on ALFWorld, Score/Acc on WebShop.}
\label{tab:agentic}
\vspace{-4pt}
\scriptsize
\setlength{\tabcolsep}{3pt}
\begin{tabular}{lccc}
\toprule
\textbf{Method} & \textbf{ALF} & \textbf{WS-S} & \textbf{WS-A} \\
\midrule
Base & 21.9 & 6.7 & 0.8 \\
GRPO & 75.0 & \underline{79.8} & 63.3 \\
\textbf{RISE (logit)} & \underline{78.1} & 78.8 & \underline{68.0} \\
\textbf{RISE (weight)} & \textbf{84.4} & \textbf{86.3} & \textbf{74.2} \\
\bottomrule
\end{tabular}
\vspace{-14pt}
\end{wraptable}

\textbf{Agentic Tasks.~}
Table~\ref{tab:agentic} reports results on two multi-turn agentic benchmarks using Qwen2.5-3B-Instruct as the base model, following the GIGPO training setup~\citep{feng2025group}. RISE (weight) substantially outperforms GRPO on both ALFWorld (+9.4) and WebShop (+10.9 Acc), demonstrating that the extrapolation principle extends to sequential decision-making tasks where the reward signal is sparse and delayed.

\subsection{Analysis: How Does Extrapolation Help?}
\label{sec:exp-analysis}

Beyond aggregate accuracy, we ask \emph{why} extrapolation helps. RISE couples two ingredients---an RLVR phase that supplies the displacement direction, and an OPD phase that projects the extrapolated teacher back onto the policy. We first isolate each phase by removing it (Fig.~\ref{fig:analysis}), then test the extrapolation premise directly on saved checkpoints (Fig.~\ref{fig:extrap-sweep}).

\begin{figure}[t]
\centering
\begin{subfigure}[b]{0.32\textwidth}
    \includegraphics[width=\linewidth]{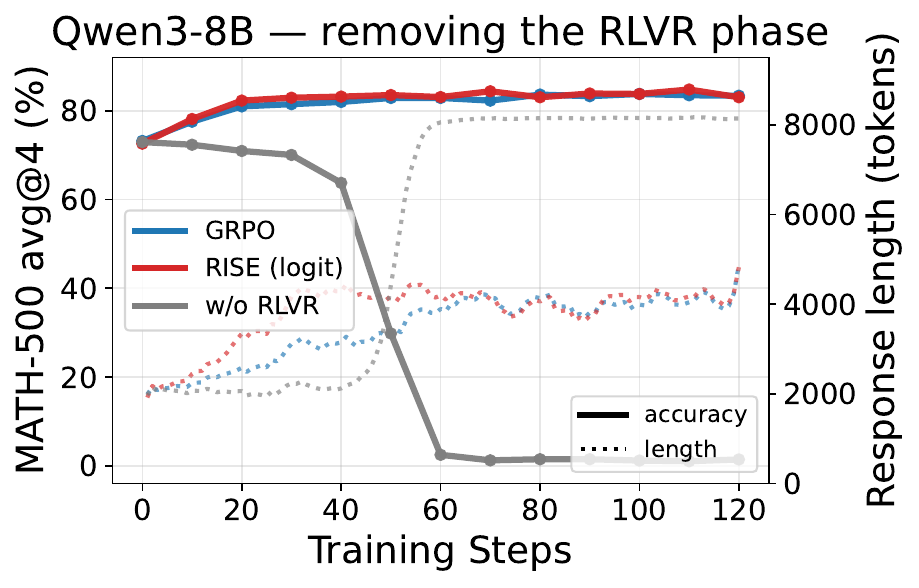}
    \caption{Removing RLVR (8B)}
    \label{fig:analysis-grounding}
\end{subfigure}
\hfill
\begin{subfigure}[b]{0.32\textwidth}
    \includegraphics[width=\linewidth]{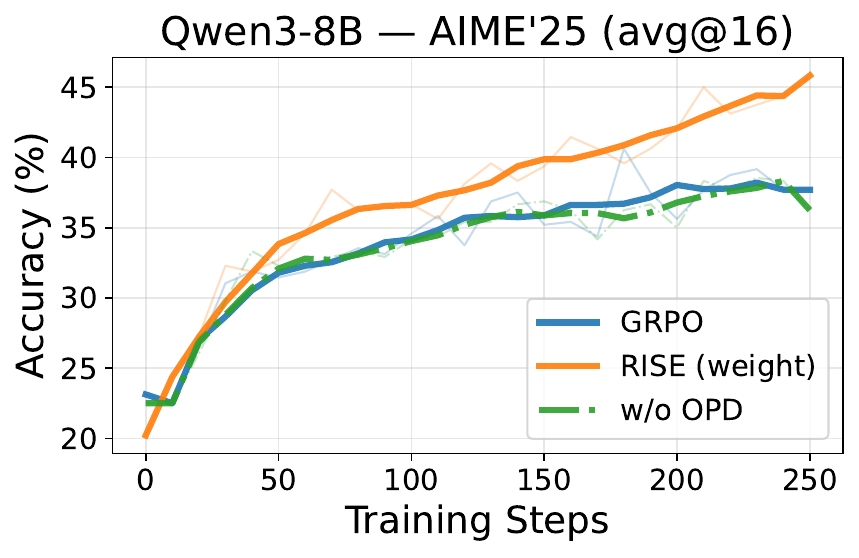}
    \caption{Removing OPD (8B)}
    \label{fig:analysis-trust-8b}
\end{subfigure}
\hfill
\begin{subfigure}[b]{0.32\textwidth}
    \includegraphics[width=\linewidth]{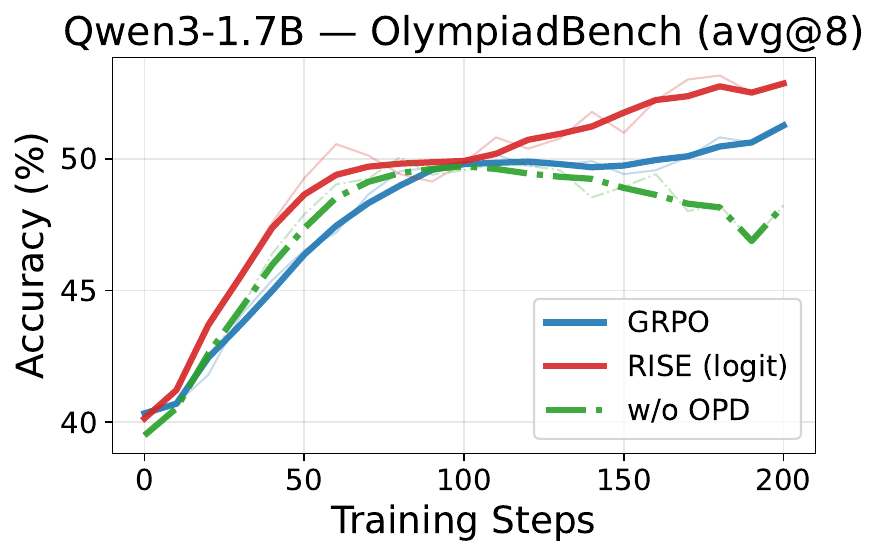}
    \caption{Removing OPD (1.7B)}
    \label{fig:analysis-trust-17b}
\end{subfigure}
\vspace{-5pt}
\caption{\textbf{Removing either phase hurts.} (a)~Without RLVR, accuracy collapses (solid) as generation length explodes (dotted). (b,~c)~Without OPD, training is stable but gains vanish.}
\label{fig:analysis}
\vspace{-10pt}
\end{figure}

\textbf{Extrapolation is only safe when the direction is grounded in verifiable reward.~}
We remove the RLVR phase entirely and extrapolate along the displacement induced by self-distillation alone. Training collapses within 60 steps (Fig.~\ref{fig:analysis-grounding}): MATH-500 accuracy drops from the base level to 2.4\%, response length explodes from 2K to the 8K context cap, and training reward drops to zero (Fig.~\ref{fig:app-grounding-reward-entropy}). GRPO and RISE show no such behaviour. Without an outcome reward to anchor it, the displacement is purely self-referential---each iteration extrapolates along the model's prior move---and repeated amplification drives the policy into degenerate non-terminating generation.

\textbf{The gain comes from distillation, not from taking a longer step.~}
We next remove the OPD phase, adopting the extrapolated $\theta_{\text{future}}$ directly as the next policy. This variant is stable but does not improve the \emph{average}: on Qwen3-8B Math Avg moves from 60.0 (GRPO) to 60.3 (w/o OPD), and on Qwen3-1.7B from 45.4 to 45.6---negligible gains compared to RISE's $+2.7$ and $+4.8$ (Table~\ref{tab:app-no-opd}). Individual benchmarks shift in opposite directions at different scales (Figs.~\ref{fig:analysis-trust-8b},~\ref{fig:analysis-trust-17b}), but these redistributions do not translate into consistent improvement. Taking a longer step in weight space \emph{shifts} the policy without reliably improving it; it is the OPD phase---projecting the extrapolated teacher's per-token distribution back onto the policy---that converts the extrapolation direction into consistent gains across benchmarks and scales.

\begin{wrapfigure}{r}{0.36\textwidth}
\vspace{-16pt}
\centering
\includegraphics[width=\linewidth]{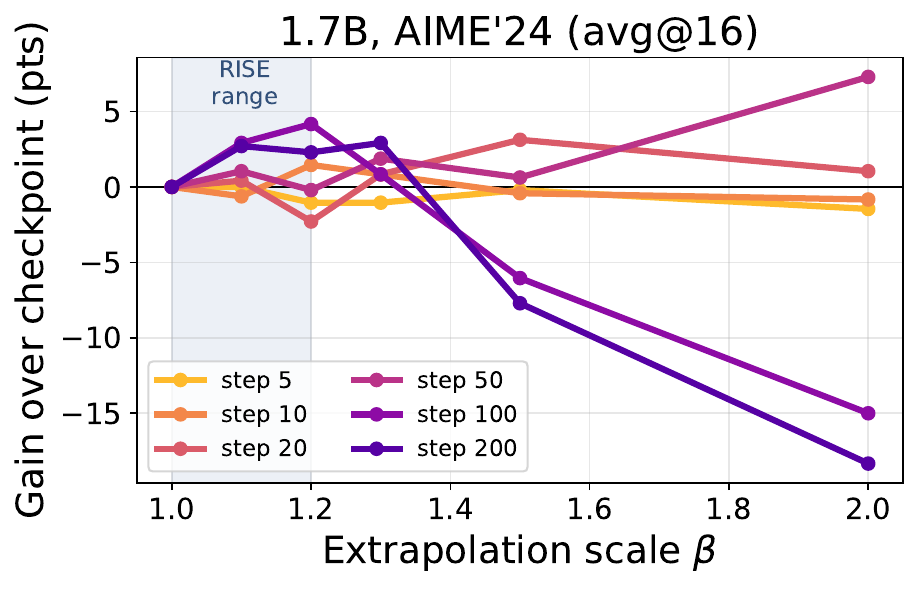}
\caption{Gain on AIME'24 (avg@16) from extrapolating Qwen3-1.7B GRPO checkpoints at varying $\beta$, relative to $\beta\!=\!1$. Shaded: RISE's operating range ($\beta_0\!=\!1.2$, decaying to~1).}
\label{fig:extrap-sweep}
\vspace{-13pt}
\end{wrapfigure}

\textbf{The safe extrapolation range narrows over training.~}
We test the extrapolation premise directly by materializing $\theta_{\text{base}} + \beta(\theta_n - \theta_{\text{base}})$ from GRPO checkpoints of Qwen3-1.7B and evaluating on AIME'24 (Fig.~\ref{fig:extrap-sweep}). Two patterns emerge. First, within RISE's operating range ($\beta\!\leq\!1.2$, shaded), extrapolation is safe throughout: gains are near-neutral early and reach $+2$--$4$ points at steps~100--200, while never degrading accuracy. Second, the \emph{tolerable range} of $\beta$ contracts sharply as training proceeds: at step~50, even $\beta\!=\!2.0$ adds $+7.3$ points, but by step~100 the same $\beta$ is catastrophic ($-15$ points), and by step~200 $\beta\!=\!1.5$ already costs $7.7$ points. This is the empirical counterpart of Proposition~\ref{prop:extrapolation-gap} and the direct justification for the decaying schedule: a fixed aggressive $\beta$ would eventually destroy the teacher, whereas RISE's conservative, decaying $\beta$ stays inside the safe region throughout.

\textbf{RISE broadens solution coverage, not just average quality.~}
A natural concern with distillation is that it sharpens the policy around solutions it already finds, improving mean accuracy at the expense of coverage. We test this by comparing avg@16 with pass@16 (best-of-16) on the AIME benchmarks (Appendix~\ref{sec:app-avg-vs-pass}). RISE improves pass@16 over GRPO in nearly every setting, with the effect strongest at 1.7B: RISE (logit) lifts AIME'24 pass@16 by $+9.6$ points versus $+6.3$ on avg@16, indicating that the extrapolated teacher expands the set of solvable problems rather than merely sharpening existing solutions. At 8B the coverage gains are more modest, consistent with less headroom when the base policy is already stronger.

\subsection{Ablation Studies}
\label{sec:exp-ablations}

We conduct ablations primarily on Qwen3-8B and Qwen3-1.7B(-Base) to isolate the contribution of each design choice.

\begin{wraptable}{r}{0.42\textwidth}
\vspace{-14pt}
\centering
\caption{\textbf{$\beta_0$ and schedule ablation} (Qwen3-8B, Math Avg \%).}
\label{tab:ablation-beta-schedule}
\vspace{-4pt}
\scriptsize
\setlength{\tabcolsep}{2pt}
\begin{tabular}{l|ccc|cccc}
\toprule
& \multicolumn{3}{c|}{$\beta_0$ (linear decay)} & \multicolumn{4}{c}{Schedule ($\beta_0\!=\!1.2$)} \\
& 1.2 & 1.3 & 1.5 & cosine & linear & exp & fixed \\
\midrule
Math Avg & 62.5 & 62.3 & 61.9 & \textbf{63.0} & 62.5 & 62.0 & 61.8 \\
\bottomrule
\end{tabular}
\vspace{-8pt}
\end{wraptable}

\textbf{Extrapolation scale $\beta$ and decay schedule.~}
Table~\ref{tab:ablation-beta-schedule} reports sensitivity to $\beta_0$ and the decay schedule on Qwen3-8B (logit-space extrapolation). Performance is stable across $\beta_0 \in [1.2, 1.5]$, with $\beta_0\!=\!1.2$ slightly ahead; $\beta_0\!=\!2.0$ diverged in preliminary experiments, consistent with the narrowing safe range shown in Fig.~\ref{fig:extrap-sweep}: although $\beta\!=\!2.0$ is benign early in training, it becomes catastrophic once the policy nears optimality, and linear decay from $\beta_0\!=\!2.0$ does not reduce $\beta$ fast enough to avoid this regime. All decay schedules perform comparably, though any form of decay edges out a fixed $\beta$---consistent with the prediction that $\beta$ should shrink as the policy approaches optimality. On Qwen3-1.7B-Base the pattern is similar: $\beta_0\!=\!1.3$ (linear) gives the best average (29.4), while $\beta_0\!=\!1.5$ degrades to 28.5 (per-benchmark breakdowns in Appendix~\ref{sec:app-beta-schedule}).

\begin{wrapfigure}{r}{0.28\textwidth}
    \vspace{-13pt}
    \centering
    \includegraphics[width=\linewidth]{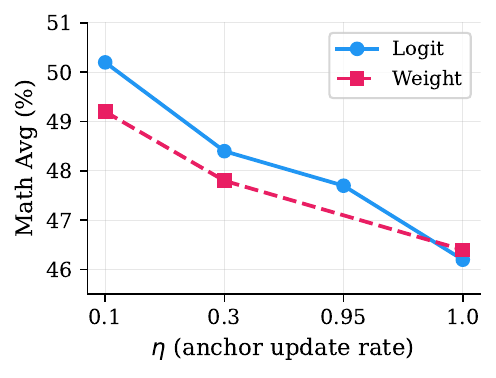}
    \vspace{-18pt}
    \caption{Anchor ablation on Qwen3-1.7B. Lower $\eta$ improves Math Avg.}
    \label{fig:anchor-ablation}
    \vspace{-12pt}
\end{wrapfigure}

\textbf{Anchor dynamics ($\eta$).~}
The EMA anchor ($\eta < 1$) smooths the displacement direction over multiple iterations, which helps when per-step RLVR updates are noisy. Fig.~\ref{fig:anchor-ablation} ablates $\eta$ on Qwen3-1.7B. In logit space, reducing $\eta$ monotonically improves Math Avg: from 46.2 ($\eta\!=\!1.0$) to 48.4 ($\eta\!=\!0.3$) to 50.2 ($\eta\!=\!0.1$); in weight space the trend is consistent ($+2.8$ from $\eta\!=\!1.0$ to $\eta\!=\!0.1$). However, on OLMo3-7B $\eta\!=\!1.0$ outperforms $\eta\!=\!0.1$ (56.4 vs.\ 50.1), likely because GRPO on OLMo already produces large, stable per-step directions, making EMA smoothing unnecessary and its lag counterproductive. We therefore default to $\eta\!=\!0.1$ for Qwen and $\eta\!=\!1$ for OLMo (Appendix~\ref{sec:app-anchor}).

\textbf{Resample after RLVR vs.\ rollout reuse.~}
Generating fresh rollouts from $\pi_{\theta_{n+1}'}$ for the OPD phase eliminates the mild off-policy-ness discussed in \S\ref{sec:method-procedure}, but does not improve results: on Qwen3-8B the resample variant matches rollout reuse exactly (62.5 Math Avg), and on Qwen3-1.7B-Base it slightly underperforms (28.4 vs.\ 29.2). Rollout reuse is therefore the practical default, halving the per-iteration sampling cost (per-benchmark details in Appendix~\ref{sec:app-resample}).

\textbf{Compute-matched comparison.~}
Since RISE applies two gradient phases per iteration (RLVR + OPD), we compare against GRPO that performs a second inner-loop gradient pass on the same rollouts within each iteration (GRPO-2$\times$), matching the total gradient budget. The extra pass helps only marginally---$+1.9$ Math Avg on Qwen3-1.7B (47.3 vs.\ 45.4) and $+0.5$ on Qwen3-8B (60.5 vs.\ 60.0)---while RISE gains $+4.8$ and $+2.7$ over GRPO at the same budget, outperforming GRPO-2$\times$ by $+2.9$ and $+2.2$ respectively (per-benchmark details in Appendix~\ref{sec:app-compute-matched}). The gains therefore stem from the quality of extrapolated supervision rather than additional optimization steps.

\textbf{Training cost.~}
RISE reuses the RLVR rollouts for OPD, so the overhead comes from the OPD gradient phase and the anchor forward pass. On Qwen3-8B, RISE runs at ${\sim}$1.6$\times$ GRPO wall time per iteration; on Qwen3-1.7B, ${\sim}$1.3$\times$. The smaller relative overhead at 1.7B reflects that rollout generation---shared between GRPO and RISE---dominates at smaller scale (73\% of GRPO wall time vs.\ 63\% at 8B).

\section{Conclusion}
\label{sec:conclusion}

We introduced RISE, which constructs a per-token teacher by extrapolating the model's own RLVR training trajectory---in logit space or weight space---then distills the resulting future policy back into the current student. The two phases are complementary: outcome rewards ground the extrapolation in verified improvement, while distillation converts the teacher's token-level distribution into durable policy gains. Across four task families (math, STEM, code, and agentic tasks) spanning 1.7B--8B parameters, RISE consistently outperforms RLVR and privileged-conditioning baselines, with the largest margins on challenging competition benchmarks. Analysis confirms that both phases are necessary, that the safe extrapolation range narrows over training, and that RISE broadens solution coverage alongside mean accuracy---all at 1.3--1.6$\times$ GRPO wall time with no additional sampling cost. More broadly, RISE demonstrates that a model's own training trajectory contains sufficient structure to serve as a self-improving teacher---converting sparse outcome signals into dense token-level supervision without any external knowledge. We believe this principle extends naturally to other post-training paradigms such as preference optimization and multi-agent settings.

\textbf{Limitations.~}
RISE relies on the training trajectory being sufficiently low-dimensional for linear extrapolation to remain meaningful at modest $\beta$. The $\beta$-decay schedule and EMA anchor provide empirical robustness, yet principled detection of when extrapolation becomes unreliable remains open. Additionally, RISE inherits any biases in the RLVR reward signal: if the reward is hackable, extrapolation amplifies the spurious direction. Integrating reward model uncertainty or multi-objective rewards to qualify the extrapolation is a promising avenue for future work.

\bibliography{main}

@article{ouyang2022training,
  title={Training language models to follow instructions with human feedback},
  author={Ouyang, Long and Wu, Jeffrey and Jiang, Xu and Almeida, Diogo and Wainwright, Carroll and Mishkin, Pamela and Zhang, Chong and Agarwal, Sandhini and Slama, Katarina and Ray, Alex and others},
  journal={Advances in neural information processing systems},
  volume={35},
  pages={27730--27744},
  year={2022}
}

@article{yang2026selfimprovement,
  title={Self-Improvement of Large Language Models: A Technical Overview and Future Outlook},
  author={Yang, Haoyan and Xerri, Mario and Park, Solha and Zhang, Huajian and Feng, Yiyang and Kogilathota, Sai Akhil and Zhou, Jiawei},
  journal={arXiv preprint arXiv:2603.25681},
  year={2026}
}

@article{tao2024survey,
  title={A survey on self-evolution of large language models},
  author={Tao, Zhengwei and Lin, Ting-En and Chen, Xiancai and Li, Hangyu and Wu, Yuchuan and Li, Yongbin and Jin, Zhi and Huang, Fei and Tao, Dacheng and Zhou, Jingren},
  journal={arXiv preprint arXiv:2404.14387},
  year={2024}
}

@article{shao2024deepseekmath,
  title={Deepseekmath: Pushing the limits of mathematical reasoning in open language models},
  author={Shao, Zhihong and Wang, Peiyi and Zhu, Qihao and Xu, Runxin and Song, Junxiao and Bi, Xiao and Zhang, Haowei and Zhang, Mingchuan and Li, YK and Wu, Yang and others},
  journal={arXiv preprint arXiv:2402.03300},
  year={2024}
}

@article{guo2025deepseek,
  title={Deepseek-r1: Incentivizing reasoning capability in llms via reinforcement learning},
  author={Guo, Daya and Yang, Dejian and Zhang, Haowei and Song, Junxiao and Wang, Peiyi and Zhu, Qihao and Xu, Runxin and Zhang, Ruoyu and Ma, Shirong and Bi, Xiao and others},
  journal={arXiv preprint arXiv:2501.12948},
  year={2025}
}

@article{yu2026dapo,
  title={Dapo: An open-source llm reinforcement learning system at scale},
  author={Yu, Qiying and Zhang, Zheng and Zhu, Ruofei and Yuan, Yufeng and Zuo, Xiaochen and Yue, Yu and Dai, Weinan and Fan, Tiantian and Liu, Gaohong and Liu, Lingjun and others},
  journal={Advances in Neural Information Processing Systems},
  volume={38},
  pages={113222--113244},
  year={2026}
}

@article{fang2026serl,
  title={Serl: Self-play reinforcement learning for large language models with limited data},
  author={Fang, Wenkai and Liu, Shunyu and Zhou, Yang and Zhang, Kongcheng and Zheng, Tongya and Chen, Kaixuan and Song, Mingli and Tao, Dacheng},
  journal={Advances in Neural Information Processing Systems},
  volume={38},
  pages={103706--103738},
  year={2026}
}

@article{chen2024self,
  title={Self-play fine-tuning converts weak language models to strong language models},
  author={Chen, Zixiang and Deng, Yihe and Yuan, Huizhuo and Ji, Kaixuan and Gu, Quanquan},
  journal={arXiv preprint arXiv:2401.01335},
  year={2024}
}

@article{liu2025spiral,
  title={Spiral: Self-play on zero-sum games incentivizes reasoning via multi-agent multi-turn reinforcement learning},
  author={Liu, Bo and Guertler, Leon and Yu, Simon and Liu, Zichen and Qi, Penghui and Balcells, Daniel and Liu, Mickel and Tan, Cheston and Shi, Weiyan and Lin, Min and others},
  journal={arXiv preprint arXiv:2506.24119},
  year={2025}
}

@article{huang2025r,
  title={R-zero: Self-evolving reasoning llm from zero data},
  author={Huang, Chengsong and Yu, Wenhao and Wang, Xiaoyang and Zhang, Hongming and Li, Zongxia and Li, Ruosen and Huang, Jiaxin and Mi, Haitao and Yu, Dong},
  journal={arXiv preprint arXiv:2508.05004},
  year={2025}
}

@article{zelikman2022star,
  title={Star: Bootstrapping reasoning with reasoning},
  author={Zelikman, Eric and Wu, Yuhuai and Mu, Jesse and Goodman, Noah},
  journal={Advances in Neural Information Processing Systems},
  volume={35},
  pages={15476--15488},
  year={2022}
}

@article{xiong2025minimalist,
  title={A minimalist approach to llm reasoning: from rejection sampling to reinforce},
  author={Xiong, Wei and Yao, Jiarui and Xu, Yuhui and Pang, Bo and Wang, Lei and Sahoo, Doyen and Li, Junnan and Jiang, Nan and Zhang, Tong and Xiong, Caiming and others},
  journal={arXiv preprint arXiv:2504.11343},
  year={2025}
}

@article{gulcehre2023reinforced,
  title={Reinforced self-training (rest) for language modeling},
  author={Gulcehre, Caglar and Paine, Tom Le and Srinivasan, Srivatsan and Konyushkova, Ksenia and Weerts, Lotte and Sharma, Abhishek and Siddhant, Aditya and Ahern, Alex and Wang, Miaosen and Gu, Chenjie and others},
  journal={arXiv preprint arXiv:2308.08998},
  year={2023}
}

@article{song2026survey,
  title={A survey of on-policy distillation for large language models},
  author={Song, Mingyang and Zheng, Mao},
  journal={arXiv preprint arXiv:2604.00626},
  year={2026}
}

@article{li2026rethinking,
  title={Rethinking on-policy distillation of large language models: Phenomenology, mechanism, and recipe},
  author={Li, Yaxuan and Zuo, Yuxin and He, Bingxiang and Zhang, Jinqian and Xiao, Chaojun and Qian, Cheng and Yu, Tianyu and Gao, Huan-ang and Yang, Wenkai and Liu, Zhiyuan and others},
  journal={arXiv preprint arXiv:2604.13016},
  year={2026}
}

@article{zhu2026many,
  title={The many faces of on-policy distillation: Pitfalls, mechanisms, and fixes},
  author={Zhu, Siqi and Ye, Xuyan and Lu, Hongyu and Shi, Weiye and Liu, Ge},
  journal={arXiv preprint arXiv:2605.11182},
  year={2026}
}

@article{zhao2026self,
  title={Self-Distilled Reasoner: On-Policy Self-Distillation for Large Language Models},
  author={Zhao, Siyan and Xie, Zhihui and Liu, Mengchen and Huang, Jing and Pang, Guan and Chen, Feiyu and Grover, Aditya},
  journal={arXiv preprint arXiv:2601.18734},
  year={2026}
}

@article{hubotter2026reinforcement,
  title={Reinforcement Learning via Self-Distillation},
  author={H{\"u}botter, Jonas and L{\"u}beck, Frederike and Behric, Lejs and Baumann, Anton and Bagatella, Marco and Marta, Daniel and Hakimi, Ido and Shenfeld, Idan and Buening, Thomas Kleine and Guestrin, Carlos and others},
  journal={arXiv preprint arXiv:2601.20802},
  year={2026}
}

@article{li2026learning,
  title={Learning from Language Feedback via Variational Policy Distillation},
  author={Li, Yang and Nijkamp, Erik and Yavuz, Semih and Joty, Shafiq Rayhan},
  journal={arXiv preprint arXiv:2605.15113},
  year={2026}
}

@article{lu2026self,
  title={Self-distilled agentic reinforcement learning},
  author={Lu, Zhengxi and Yao, Zhiyuan and Han, Zhuowen and Wang, Zi-Han and Wu, Jinyang and Gu, Qi and Cai, Xunliang and Lu, Weiming and Xiao, Jun and Zhuang, Yueting and others},
  journal={arXiv preprint arXiv:2605.15155},
  year={2026}
}

@article{yang2026self,
  title={Self-distilled rlvr},
  author={Yang, Chenxu and Qin, Chuanyu and Si, Qingyi and Chen, Minghui and Gu, Naibin and Yao, Dingyu and Lin, Zheng and Wang, Weiping and Wang, Jiaqi and Duan, Nan},
  journal={arXiv preprint arXiv:2604.03128},
  year={2026}
}

@article{xu2026tip,
  title={Tip: Token importance in on-policy distillation},
  author={Xu, Yuanda and Sang, Hejian and Zhou, Zhengze and He, Ran and Wang, Zhipeng and Geramifard, Alborz},
  journal={arXiv preprint arXiv:2604.14084},
  year={2026}
}

@article{jiang2026trajectory,
  title={Trajectory-Refined Distillation},
  author={Jiang, Li and Xu, Haoran and Ding, Yichuan and Zhang, Amy},
  journal={arXiv preprint arXiv:2606.08432},
  year={2026}
}

@inproceedings{jung2025todi,
  title={Todi: Token-wise distillation via fine-grained divergence control},
  author={Jung, Seongryong and Yoon, Suwan and Kim, DongGeon and Lee, Hwanhee},
  booktitle={Proceedings of the 2025 Conference on Empirical Methods in Natural Language Processing},
  pages={8089--8102},
  year={2025}
}

@article{jin2026entropy,
  title={Entropy-Aware On-Policy Distillation of Language Models},
  author={Jin, Woogyeol and Min, Taywon and Yang, Yongjin and Kadhe, Swanand Ravindra and Zhou, Yi and Wei, Dennis and Baracaldo, Nathalie and Lee, Kimin},
  journal={arXiv preprint arXiv:2603.07079},
  year={2026}
}

@article{jia2026asymmetric,
  title={Asymmetric on-policy distillation: Bridging exploitation and imitation at the token level},
  author={Jia, Nan and Yang, Haojin and Ma, Xing and Lian, Jiesong and Zhang, Shuailiang and Zhang, Weipeng and Zeng, Ke and Cai, Xunliang and Sun, Zequn},
  journal={arXiv preprint arXiv:2605.06387},
  year={2026}
}

@article{xing2026trust,
  title={Trust Region On-Policy Distillation},
  author={Xing, Xingrun and Wang, Haoqing and Gao, Boyan and Li, Ziheng and Tang, Yehui},
  journal={arXiv preprint arXiv:2606.01249},
  year={2026}
}

@article{li2026revisiting,
  title={Revisiting DAgger in the Era of LLM-Agents},
  author={Li, Changhao and Qiang, Rushi and Huang, Jiawei and Gao, Chenxiao and Zhang, Chao and He, Niao and Dai, Bo},
  journal={arXiv preprint arXiv:2605.12913},
  year={2026}
}

@article{yang2026reasoning,
  title={Reasoning Compression with Mixed-Policy Distillation},
  author={Yang, Han and Wu, Mingyan and He, Bailan and Cao, Zeyu and Yan, Sikuan and Lin, Kevin Qinghong and Ding, Zifeng},
  journal={arXiv preprint arXiv:2605.08776},
  year={2026}
}

@article{jang2026stable,
  title={Stable On-Policy Distillation through Adaptive Target Reformulation},
  author={Jang, Ijun and Yeom, Jewon and Yeo, Juan and Lim, Hyunggu and Kim, Taesup},
  journal={arXiv preprint arXiv:2601.07155},
  year={2026}
}

@article{liu2026prefix,
  title={Prefix Teach, Suffix Fade: Local Teachability Collapse in Strong-to-Weak On-Policy Distillation},
  author={Liu, Kaiyuan and Zhuang, Ziyuan and Bai, Yang and Wang, Bing and Weng, Rongxiang and Ye, Jieping},
  journal={arXiv preprint arXiv:2605.13643},
  year={2026}
}

@article{agrawal2026reinforcement,
  title={Reinforcement Learning from Rich Feedback with Distributional DAgger},
  author={Agrawal, Rishabh and Fein-Ashley, Jacob and Rashidinejad, Paria},
  journal={arXiv preprint arXiv:2606.05152},
  year={2026}
}

@article{zhao2026decoupling,
  title={Decoupling kl and trajectories: A unified perspective for sft, dagger, offline rl, and opd in llm distillation},
  author={Zhao, Anhao and Xin, Haoran and Fan, Yingqi and Tong, Junlong and Li, Wenjie and Shen, Xiaoyu},
  journal={arXiv preprint arXiv:2605.16826},
  year={2026}
}

@article{lu2025onpolicydistillation,
  author = {Kevin Lu and Thinking Machines Lab},
  title = {On-Policy Distillation},
  journal = {Thinking Machines Lab: Connectionism},
  year = {2025},
  note = {https://thinkingmachines.ai/blog/on-policy-distillation},
  doi = {10.64434/tml.20251026},
}

@inproceedings{agarwal2024policy,
  title={On-policy distillation of language models: Learning from self-generated mistakes},
  author={Agarwal, Rishabh and Vieillard, Nino and Zhou, Yongchao and Stanczyk, Piotr and Ramos Garea, Sabela and Geist, Matthieu and Bachem, Olivier},
  booktitle={International Conference on Learning Representations},
  volume={2024},
  pages={21246--21263},
  year={2024}
}

@misc{mimo2025flash,
  title={MiMo-V2-Flash Technical Report},
  author={LLM-Core Xiaomi},
  year={2025},
  url={https://github.com/XiaomiMiMo/MiMo-V2-Flash/paper.pdf}
}

@article{zeng2026glm,
  title={Glm-5: from vibe coding to agentic engineering},
  author={Zeng, Aohan and Lv, Xin and Hou, Zhenyu and Du, Zhengxiao and Zheng, Qinkai and Chen, Bin and Yin, Da and Ge, Chendi and Huang, Chenghua and Xie, Chengxing and others},
  journal={arXiv preprint arXiv:2602.15763},
  year={2026}
}

@article{kim2026does,
  title={Why Does Self-Distillation (Sometimes) Degrade the Reasoning Capability of LLMs?},
  author={Kim, Jeonghye and Luo, Xufang and Kim, Minbeom and Lee, Sangmook and Kim, Dohyung and Jeon, Jiwon and Li, Dongsheng and Yang, Yuqing},
  journal={arXiv preprint arXiv:2603.24472},
  year={2026}
}

@article{ko2026scaling,
  title={Scaling reasoning efficiently via relaxed on-policy distillation},
  author={Ko, Jongwoo and Abdali, Sara and Kim, Young Jin and Chen, Tianyi and Cameron, Pashmina},
  journal={arXiv preprint arXiv:2603.11137},
  year={2026}
}

@inproceedings{ross2011reduction,
  title={A reduction of imitation learning and structured prediction to no-regret online learning},
  author={Ross, St{\'e}phane and Gordon, Geoffrey and Bagnell, Drew},
  booktitle={Proceedings of the fourteenth international conference on artificial intelligence and statistics},
  pages={627--635},
  year={2011},
  organization={JMLR Workshop and Conference Proceedings}
}

@article{ilharco2022editing,
  title={Editing models with task arithmetic},
  author={Ilharco, Gabriel and Ribeiro, Marco Tulio and Wortsman, Mitchell and Gururangan, Suchin and Schmidt, Ludwig and Hajishirzi, Hannaneh and Farhadi, Ali},
  journal={arXiv preprint arXiv:2212.04089},
  year={2022}
}

@inproceedings{frankle2020linear,
  title={Linear mode connectivity and the lottery ticket hypothesis},
  author={Frankle, Jonathan and Dziugaite, Gintare Karolina and Roy, Daniel and Carbin, Michael},
  booktitle={International conference on machine learning},
  pages={3259--3269},
  year={2020},
  organization={PMLR}
}

@article{cai2025predictability,
  title={On predictability of reinforcement learning dynamics for large language models},
  author={Cai, Yuchen and Cao, Ding and Xu, Xin and Yao, Zijun and Huang, Yuqing and Tan, Zhenyu and Zhang, Benyi and Sun, Guangzhong and Liu, Guiquan and Fang, Junfeng},
  journal={arXiv preprint arXiv:2510.00553},
  year={2025}
}

@article{wei2026you,
  title={You Only Need Minimal RLVR Training: Extrapolating LLMs via Rank-1 Trajectories},
  author={Wei, Zhepei and Zhu, Xinyu and Chen, Wei-Lin and Huang, Chengsong and Huang, Jiaxin and Meng, Yu},
  journal={arXiv preprint arXiv:2605.21468},
  year={2026}
}

@article{wang2026not,
  title={Not All Steps are Informative: On the Linearity of LLMs' RLVR Training},
  author={Wang, Tianle and Wu, Zhongyuan and Jin, Shenghao and Xu, Hao and Chen, Wei and Miao, Ning},
  journal={arXiv preprint arXiv:2601.04537},
  year={2026}
}

@article{shen2026geometry,
  title={On the Geometry of On-Policy Distillation},
  author={Shen, Zhennan and Li, Yanshu and Yin, Qingyu and Leong, Chak Tou and Wang, Zhilin and Chen, Yanxu and Han, Rongduo and Lee, Sunbowen and Fung, Yi R},
  journal={arXiv preprint arXiv:2606.07082},
  year={2026}
}

@article{yang2026learning,
  title={Learning beyond teacher: Generalized on-policy distillation with reward extrapolation},
  author={Yang, Wenkai and Liu, Weijie and Xie, Ruobing and Yang, Kai and Yang, Saiyong and Lin, Yankai},
  journal={arXiv preprint arXiv:2602.12125},
  year={2026}
}

@article{wang2026openclaw,
  title={Openclaw-rl: Train any agent simply by talking},
  author={Wang, Yinjie and Chen, Xuyang and Jin, Xiaolong and Wang, Mengdi and Yang, Ling},
  journal={arXiv preprint arXiv:2603.10165},
  year={2026}
}

@article{xu2025kdrl,
  title={Kdrl: Post-training reasoning llms via unified knowledge distillation and reinforcement learning},
  author={Xu, Hongling and Zhu, Qi and Deng, Heyuan and Li, Jinpeng and Hou, Lu and Wang, Yasheng and Shang, Lifeng and Xu, Ruifeng and Mi, Fei},
  journal={arXiv preprint arXiv:2506.02208},
  year={2025}
}

@article{ramos2026combining,
  title={Combining On-Policy Optimization and Distillation for Long-Context Reasoning in Large Language Models},
  author={Ramos, Miguel Moura and Alves, Duarte M and Martins, Andr{\'e} FT},
  journal={arXiv preprint arXiv:2605.12227},
  year={2026}
}

@article{zhang2026reinforcement,
  title={Reinforcement-aware knowledge distillation for llm reasoning},
  author={Zhang, Zhaoyang and Jiang, Shuli and Shen, Yantao and Zhang, Yuting and Ram, Dhananjay and Yang, Shuo and Tu, Zhuowen and Xia, Wei and Soatto, Stefano},
  journal={arXiv preprint arXiv:2602.22495},
  year={2026}
}

@article{oh2026kl,
  title={KL for a KL: On-Policy Distillation with Control Variate Baseline},
  author={Oh, Minjae and Song, Sangjun and Choi, Gyubin and Choi, Yunho and Jo, Yohan},
  journal={arXiv preprint arXiv:2605.07865},
  year={2026}
}

@article{xu2026beyond,
  title={Beyond GRPO and On-Policy Distillation: An Empirical Sparse-to-Dense Reward Principle for Language-Model Post-Training},
  author={Xu, Yuanda and Sang, Hejian and Zhou, Zhengze and He, Ran and Wang, Zhipeng and Geramifard, Alborz},
  journal={arXiv preprint arXiv:2605.12483},
  year={2026}
}

@inproceedings{kakade2002approximately,
  title={Approximately optimal approximate reinforcement learning},
  author={Kakade, Sham and Langford, John},
  booktitle={Proceedings of the nineteenth international conference on machine learning},
  pages={267--274},
  year={2002}
}

@inproceedings{wortsman2022model,
  title={Model soups: averaging weights of multiple fine-tuned models improves accuracy without increasing inference time},
  author={Wortsman, Mitchell and Ilharco, Gabriel and Gadre, Samir Yitzhak and Roelofs, Rebecca and Gontijo-Lopes, Raphael and Morcos, Ari S and Namkoong, Hongseok and Farhadi, Ali and Carlin, Yair and Kornblith, Simon and others},
  booktitle={International conference on machine learning},
  pages={23965--23998},
  year={2022},
  organization={PMLR}
}

@inproceedings{izmailov2018averaging,
  title={Averaging weights leads to wider optima and better generalization},
  author={Izmailov, Pavel and Podoprikhin, Dmitrii and Garipov, Timur and Vetrov, Dmitry and Wilson, Andrew Gordon},
  booktitle={34th Conference on Uncertainty in Artificial Intelligence},
  pages={876--885},
  year={2018}
}

@article{harne2026privileged,
  title={Privileged, but Biased: How PI-Conditioned Teachers Break Self-Distillation},
  author={Harne, Sarthak and Karkar, Chinmay and Pandya, Yash and Awadallah, Ahmed and Nambi, Akshay},
  journal={arXiv preprint arXiv:2608.04794},
  year={2026}
}

@article{yan2026learning,
  title={Learning to reason under off-policy guidance},
  author={Yan, Jianhao and Li, Yafu and Hu, Zican and Wang, Zhi and Cui, Ganqu and Qu, Xiaoye and Cheng, Yu and Zhang, Yue},
  journal={Advances in Neural Information Processing Systems},
  volume={38},
  pages={117157--117186},
  year={2026}
}

@article{cheng2026revisiting,
  title={Revisiting reinforcement learning for llm reasoning from a cross-domain perspective},
  author={Cheng, Jorge Zhoujun and Hao, Shibo and Liu, Tianyang and Zhou, Fan and Xie, Yutao and Yao, Feng and Bian, Yuexin and Dey, Nilabjo and Zhuang, Yonghao and Zha, Yuheng and others},
  journal={Advances in Neural Information Processing Systems},
  volume={38},
  year={2026}
}

@article{he2025skywork,
  title={Skywork Open Reasoner 1 Technical Report},
  author={He, Jujie and Liu, Jiacai and Liu, Chris Yuhao and Yan, Rui and Wang, Chaojie and Cheng, Peng and Zhang, Xiaoyu and Zhang, Fuxiang and Xu, Jiacheng and Shen, Wei and Li, Siyuan and Zeng, Liang and Wei, Tianwen and Cheng, Cheng and An, Bo and Liu, Yang and Zhou, Yahui},
  journal={arXiv preprint arXiv:2505.22312},
  year={2025}
}

@article{feng2025group,
  title={Group-in-Group Policy Optimization for LLM Agent Training},
  author={Feng, Lang and Xue, Zhenghai and Liu, Tingcong and An, Bo},
  journal={arXiv preprint arXiv:2505.10978},
  year={2025}
}

@inproceedings{lightman2024let,
  title={Let's verify step by step},
  author={Lightman, Hunter and Kosaraju, Vineet and Burda, Yuri and Edwards, Harrison and Baker, Bowen and Lee, Teddy and Leike, Jan and Schulman, John and Sutskever, Ilya and Cobbe, Karl},
  booktitle={International Conference on Learning Representations},
  volume={2024},
  pages={39578--39601},
  year={2024}
}

@misc{qwen3technicalreport,
      title={Qwen3 Technical Report}, 
      author={Qwen Team},
      year={2025},
      eprint={2505.09388},
      archivePrefix={arXiv},
      primaryClass={cs.CL},
      url={https://arxiv.org/abs/2505.09388}, 
}

@misc{olmo2025olmo3,
title={Olmo 3},
author={Team Olmo and Allyson Ettinger and Amanda Bertsch and Bailey Kuehl and David Graham and David Heineman and Dirk Groeneveld and Faeze Brahman and Finbarr Timbers and Hamish Ivison and Jacob Morrison and Jake Poznanski and Kyle Lo and Luca Soldaini and Matt Jordan and Mayee Chen and Michael Noukhovitch and Nathan Lambert and Pete Walsh and Pradeep Dasigi and Robert Berry and Saumya Malik and Saurabh Shah and Scott Geng and Shane Arora and Shashank Gupta and Taira Anderson and Teng Xiao and Tyler Murray and Tyler Romero and Victoria Graf and Akari Asai and Akshita Bhagia and Alexander Wettig and Alisa Liu and Aman Rangapur and Chloe Anastasiades and Costa Huang and Dustin Schwenk and Harsh Trivedi and Ian Magnusson and Jaron Lochner and Jiacheng Liu and Lester James V. Miranda and Maarten Sap and Malia Morgan and Michael Schmitz and Michal Guerquin and Michael Wilson and Regan Huff and Ronan Le Bras and Rui Xin and Rulin Shao and Sam Skjonsberg and Shannon Zejiang Shen and Shuyue Stella Li and Tucker Wilde and Valentina Pyatkin and Will Merrill and Yapei Chang and Yuling Gu and Zhiyuan Zeng and Ashish Sabharwal and Luke Zettlemoyer and Pang Wei Koh and Ali Farhadi and Noah A. Smith and Hannaneh Hajishirzi},
year={2025},
eprint={2512.13961},
archivePrefix={arXiv},
primaryClass={cs.CL},
url={https://arxiv.org/abs/2512.13961},
}

@article{shridhar2020alfworld,
  title={Alfworld: Aligning text and embodied environments for interactive learning},
  author={Shridhar, Mohit and Yuan, Xingdi and C{\^o}t{\'e}, Marc-Alexandre and Bisk, Yonatan and Trischler, Adam and Hausknecht, Matthew},
  journal={arXiv preprint arXiv:2010.03768},
  year={2020}
}

@article{yao2022webshop,
  title={Webshop: Towards scalable real-world web interaction with grounded language agents},
  author={Yao, Shunyu and Chen, Howard and Yang, John and Narasimhan, Karthik},
  journal={Advances in Neural Information Processing Systems},
  volume={35},
  pages={20744--20757},
  year={2022}
}

@misc{qwen2.5,
    title = {Qwen2.5: A Party of Foundation Models},
    url = {https://qwenlm.github.io/blog/qwen2.5/},
    author = {Qwen Team},
    month = {September},
    year = {2024}
}

@inproceedings{evalplus,
  title = {Is Your Code Generated by Chat{GPT} Really Correct? Rigorous Evaluation of Large Language Models for Code Generation},
  author = {Liu, Jiawei and Xia, Chunqiu Steven and Wang, Yuyao and Zhang, Lingming},
  booktitle = {Thirty-seventh Conference on Neural Information Processing Systems},
  year = {2023},
  url = {https://openreview.net/forum?id=1qvx610Cu7},
}

@inproceedings{jain2025livecodebench,
  title={Livecodebench: Holistic and contamination free evaluation of large language models for code},
  author={Jain, Naman and Gu, Alex and Li, Wen-Ding and Yan, Fanjia and Zhang, Tianjun and Wang, Sida and Solar-Lezama, Armando and Sen, Koushik and Stoica, Ion},
  booktitle={International Conference on Learning Representations},
  volume={2025},
  pages={58791--58831},
  year={2025}
}

@article{rein2023gpqa,
  title={Gpqa: A graduate-level google-proof q\&a benchmark},
  author={Rein, David and Hou, Betty Li and Stickland, Asa Cooper and Petty, Jackson and Pang, Richard Yuanzhe and Dirani, Julien and Michael, Julian and Bowman, Samuel R},
  journal={arXiv preprint arXiv:2311.12022},
  year={2023}
}

@article{zhou2023instruction,
  title={Instruction-following evaluation for large language models},
  author={Zhou, Jeffrey and Lu, Tianjian and Mishra, Swaroop and Brahma, Siddhartha and Basu, Sujoy and Luan, Yi and Zhou, Denny and Hou, Le},
  journal={arXiv preprint arXiv:2311.07911},
  year={2023}
}

@article{wang2024mmlu,
  title={Mmlu-pro: A more robust and challenging multi-task language understanding benchmark},
  author={Wang, Yubo and Ma, Xueguang and Zhang, Ge and Ni, Yuansheng and Chandra, Abhranil and Guo, Shiguang and Ren, Weiming and Arulraj, Aaran and He, Xuan and Jiang, Ziyan and others},
  journal={Advances in Neural Information Processing Systems},
  volume={37},
  pages={95266--95290},
  year={2024}
}

@article{du2026supergpqa,
  title={Supergpqa: Scaling llm evaluation across 285 graduate disciplines},
  author={Du, Xeron and Yao, Yifan and Ma, Kaijing and Wang, Bingli and Zheng, Tianyu and Liu, Minghao and Liang, Yiming and Jin, Xiaolong and Wei, Zhenlin and Zheng, Chujie and others},
  journal={Advances in Neural Information Processing Systems},
  volume={38},
  year={2026}
}

@inproceedings{chen2023theoremqa,
  title={Theoremqa: A theorem-driven question answering dataset},
  author={Chen, Wenhu and Yin, Ming and Ku, Max and Lu, Pan and Wan, Yixin and Ma, Xueguang and Xu, Jianyu and Wang, Xinyi and Xia, Tony},
  booktitle={Proceedings of the 2023 Conference on Empirical Methods in Natural Language Processing},
  pages={7889--7901},
  year={2023}
}

@article{lewkowycz2022solving,
  title={Solving quantitative reasoning problems with language models},
  author={Lewkowycz, Aitor and Andreassen, Anders and Dohan, David and Dyer, Ethan and Michalewski, Henryk and Ramasesh, Vinay and Slone, Ambrose and Anil, Cem and Schlag, Imanol and Gutman-Solo, Theo and others},
  journal={Advances in neural information processing systems},
  volume={35},
  pages={3843--3857},
  year={2022}
}

@article{he2024olympiadbench,
  title={Olympiadbench: A challenging benchmark for promoting agi with olympiad-level bilingual multimodal scientific problems},
  author={He, Chaoqun and Luo, Renjie and Bai, Yuzhuo and Hu, Shengding and Thai, Zhen Leng and Shen, Junhao and Hu, Jinyi and Han, Xu and Huang, Yujie and Zhang, Yuxiang and others},
  journal={arXiv preprint arXiv:2402.14008},
  year={2024}
}

@misc{aime_2024,
  author = {Maxwell Jia},
  title = {AIME Problem Set 2024},
  year = {2024},
  publisher = {Huggingface},
  url = {https://huggingface.co/datasets/Maxwell-Jia/AIME_2024}
}

@misc{aime_2025,
  author = {math-ai},
  title = {AIME Problem Set 2025},
  year = {2025},
  publisher = {Huggingface},
  url = {https://huggingface.co/datasets/math-ai/aime25}
}

@misc{amc2023,
  title = {{AMC} 10/12 2023},
  author = {{MAA}},
  year = {2023},
  url = {https://www.maa.org/math-competitions/amc-1012},
}
\bibliographystyle{salesforce}

\appendix
\counterwithin{figure}{section}
\counterwithin{table}{section}
\counterwithin{equation}{section}
\newpage

\section{Theoretical Analysis}
\label{app:theory}

\subsection{Suboptimality Decomposition and Optimal Teacher}

We first state the performance difference lemma adapted to the autoregressive setting~\citep{kakade2002approximately}.

\begin{lemma}[Performance difference]
\label{lem:perf-diff}
For two policies $\pi_1, \pi_2$ generating responses of length $T$ with terminal reward $R(x,y) \in [0,1]$:
\begin{equation}
    J(\pi_1) - J(\pi_2) = \sum_{t=1}^{T} \mathbb{E}_{s_t \sim \rho^t_{\pi_1}}\!\left[\sum_{v}\left(\pi_1(v \mid s_t) - \pi_2(v \mid s_t)\right) Q^{\pi_2}(s_t, v)\right],
\end{equation}
where $s_t = (x, y_{<t})$ is the context at position $t$, $\rho^t_{\pi_1}$ is the distribution over position-$t$ contexts induced by $\pi_1$, and $Q^{\pi_2}(s_t, v) = \mathbb{E}_{y_{t+1:T} \sim \pi_2}[R(x, y) \mid s_t, y_t = v] \in [0,1]$ is the action-value function under $\pi_2$.
\end{lemma}

\begin{proof}
The performance difference lemma~\citep{kakade2002approximately}, originally stated for the discounted infinite-horizon setting, specializes to the finite-horizon undiscounted case as
\begin{equation}
J(\pi_1) - J(\pi_2) = \sum_{t=1}^{T} \mathbb{E}_{s_t \sim P_t^{\pi_1}}[\mathbb{E}_{a_t \sim \pi_1}[A^{\pi_2}(s_t, a_t)]],
\end{equation}
where $A^{\pi_2}(s,a) = Q^{\pi_2}(s,a) - V^{\pi_2}(s)$ is the advantage under $\pi_2$.
In the autoregressive setting, states are $s_t = (x, y_{<t})$ and actions are tokens $v$.
Expanding $A^{\pi_2} = Q^{\pi_2} - V^{\pi_2}$ and noting that $V^{\pi_2}(s_t)$ is constant w.r.t.\ the sum over $v$:
\begin{align}
\sum_v \pi_1(v|s_t)\, A^{\pi_2}(s_t,v)
&= \sum_v \pi_1(v|s_t)\, Q^{\pi_2}(s_t,v) - V^{\pi_2}(s_t) \underbrace{\sum_v \pi_1(v|s_t)}_{=\,1} \nonumber \\
&= \sum_v \pi_1(v|s_t)\, Q^{\pi_2}(s_t,v) - \underbrace{\sum_v \pi_2(v|s_t)\, Q^{\pi_2}(s_t,v)}_{=\,V^{\pi_2}(s_t)} \nonumber \\
&= \sum_v \left(\pi_1(v|s_t) - \pi_2(v|s_t)\right) Q^{\pi_2}(s_t, v).
\end{align}
Finally, $Q^{\pi_2}(s_t, v) \in [0,1]$ because it is a conditional expectation of $R \in [0,1]$.
\end{proof}

\begin{theorem}[Suboptimality decomposition for OPD]
\label{thm:suboptimality}
Let $\pi_\theta$ be a student and $\pi_T$ a teacher, with $\pi^* = \arg\max_\pi J(\pi)$, $J(\pi) = \mathbb{E}_x[\mathbb{E}_{y \sim \pi}[R(x,y)]]$, $R \in [0,1]$, and response length $T$.
The student's suboptimality decomposes as:
\begin{equation}
    J(\pi^*) - J(\pi_\theta) \;\leq\;
    \underbrace{J(\pi^*) - J(\pi_T)}_{\text{teacher gap}}
    \;+\; T \cdot \underbrace{\sqrt{\bar{\mathcal{L}}_{\mathrm{OPD}} / 2}}_{\text{distillation error}},
    \label{eq:suboptimality}
\end{equation}
where $\bar{\mathcal{L}}_{\mathrm{OPD}} = \frac{1}{T}\sum_{t=1}^{T} \mathbb{E}_{s_t \sim \rho^t_{\pi_\theta}}\!\left[\KL\!\left(\pi_\theta(\cdot \mid s_t) \;\|\; \pi_T(\cdot \mid s_t)\right)\right]$ is the average per-token reverse KL under the student's own context distribution.
The bound identifies two independent error sources: a \emph{distillation error} that vanishes as OPD converges ($\pi_\theta \to \pi_T$), and a \emph{teacher gap} that persists regardless of optimization quality.
Because the $T$-dependent prefactor makes the bound a qualitative mechanism decomposition rather than a tight numerical certificate, its practical value is directional: the decomposition shows that, all else equal, a stronger teacher (smaller gap) admits a tighter bound on student suboptimality.
\end{theorem}

\begin{proof}
Decompose the gap as $J(\pi^*) - J(\pi_\theta) = [J(\pi^*) - J(\pi_T)] + [J(\pi_T) - J(\pi_\theta)]$.
For the second term, apply Lemma~\ref{lem:perf-diff} with $\pi_1 = \pi_\theta$ and $\pi_2 = \pi_T$, which gives
$J(\pi_\theta) - J(\pi_T) = \sum_{t=1}^{T} \mathbb{E}_{s_t \sim \rho^t_{\pi_\theta}}[\sum_v (\pi_\theta(v \mid s_t) - \pi_T(v \mid s_t))\, Q^{\pi_T}(s_t, v)]$.
Multiplying by $-1$ flips both the performance difference and the inner difference:
\begin{equation}
    J(\pi_T) - J(\pi_\theta)
    = \sum_{t=1}^{T} \mathbb{E}_{s_t \sim \rho^t_{\pi_\theta}}\!\left[
    \sum_{v} \left(\pi_T(v \mid s_t) - \pi_\theta(v \mid s_t)\right)
    Q^{\pi_T}(s_t, v)\right].
\end{equation}
Note that this orientation places the expectation under the \emph{student's} own context distribution.
Since $Q^{\pi_T}(s_t, v) \in [0,1]$ (it is a conditional expectation of $R \in [0,1]$), the range $\max_v Q - \min_v Q \leq 1$.
Because $\sum_v (\pi_T(v|s_t) - \pi_\theta(v|s_t)) = 0$, we may subtract any constant from $Q$ inside the sum.
Choosing $c = (\max Q + \min Q)/2$ and applying H\"{o}lder's inequality:
$|\sum_v (\pi_T - \pi_\theta) Q| = |\sum_v (\pi_T - \pi_\theta)(Q - c)| \leq \|Q - c\|_\infty \cdot \|\pi_T - \pi_\theta\|_1 = \frac{\max Q - \min Q}{2} \cdot 2\,\mathrm{TV} \leq \mathrm{TV}$,
where $\mathrm{TV}(p,q) = \frac{1}{2}\sum_v |p(v) - q(v)|$. Therefore:
\begin{equation}
    J(\pi_T) - J(\pi_\theta) \leq \sum_{t=1}^{T} \mathbb{E}_{s_t \sim \rho^t_{\pi_\theta}}\!\left[\mathrm{TV}(\pi_T(\cdot|s_t), \pi_\theta(\cdot|s_t))\right].
\end{equation}
Applying Pinsker's inequality ($\mathrm{TV}(p,q) \leq \sqrt{\KL(p\|q)/2}$) to each term---valid in either argument order, since $\mathrm{TV}$ is symmetric---then Jensen's inequality ($\mathbb{E}[\sqrt{X}] \leq \sqrt{\mathbb{E}[X]}$ since $\sqrt{\cdot}$ is concave):
\begin{equation}
    \sum_{t=1}^{T} \mathbb{E}_{s_t \sim \rho^t_{\pi_\theta}}\!\left[\mathrm{TV}_t\right]
    \leq \sum_{t=1}^{T} \sqrt{\mathbb{E}_{s_t \sim \rho^t_{\pi_\theta}}[D_t] / 2}
    = T \cdot \frac{1}{T}\sum_{t=1}^{T} \sqrt{\mathbb{E}[D_t]/2},
\end{equation}
where $D_t = \KL(\pi_\theta(\cdot|s_t) \| \pi_T(\cdot|s_t))$.
By Cauchy--Schwarz, $\frac{1}{T}\sum_t \sqrt{a_t} \leq \sqrt{\frac{1}{T}\sum_t a_t}$ for $a_t = \mathbb{E}[D_t]/2 \geq 0$.
Therefore $\sum_t \sqrt{\mathbb{E}[D_t]/2} \leq T \cdot \sqrt{\bar{\mathcal{L}}_{\mathrm{OPD}}/2}$, giving the stated bound.
The teacher gap $J(\pi^*) - J(\pi_T) \geq 0$ by definition of $\pi^*$, with equality iff $\pi_T$ is itself optimal, i.e., $J(\pi_T) = J(\pi^*)$.
\end{proof}

\subsection{Extrapolation Gap and Convergence}
\label{app:extrap-gap}

\begin{proposition}[Safe extrapolation range under a linear trajectory]
\label{prop:extrapolation-gap}
Let $\varphi: \Pi \to \mathcal{Z}$ be a representation map such that the training trajectory is linear: $\varphi(\pi_{\theta_n}) = \varphi(\pi_{\theta_0}) + \alpha_n \cdot \mathbf{d}$ for a unit direction $\mathbf{d}$ and monotonically increasing coefficients $0 = \alpha_0 < \alpha_n < \alpha^*$, where $\varphi(\pi^*) = \varphi(\pi_{\theta_0}) + \alpha^* \cdot \mathbf{d}$.
The extrapolated teacher $\varphi(\pi_{\mathrm{future}}) = \varphi(\pi_{\theta_0}) + \beta \cdot \alpha_n \cdot \mathbf{d}$ is closer to $\pi^*$ than the student is, i.e.\
\begin{equation}
    \|\varphi(\pi_{\mathrm{future}}) - \varphi(\pi^*)\| < \|\varphi(\pi_{\theta_n}) - \varphi(\pi^*)\|,
\end{equation}
exactly when $1 < \beta < 2\alpha^*/\alpha_n - 1$. This is an immediate consequence of the linear-trajectory assumption---it defines the safe $\beta$ range rather than derives it from weaker premises---and its value lies in making two properties explicit: (i)~any $\beta \in (1, \alpha^*/\alpha_n)$ places the teacher strictly between the student and $\pi^*$ without overshooting, and (ii)~the upper limit $2\alpha^*/\alpha_n - 1$ shrinks as $\alpha_n \to \alpha^*$, so the safe range narrows as training converges.
\end{proposition}

\begin{proof}
The gap between the extrapolated teacher and $\pi^*$ is $\|\varphi(\pi_{\mathrm{future}}) - \varphi(\pi^*)\| = |\beta \alpha_n - \alpha^*|$.
The gap between the current student and $\pi^*$ is $\alpha^* - \alpha_n$.
The extrapolated teacher is closer whenever $|\beta \alpha_n - \alpha^*| < \alpha^* - \alpha_n$, which holds iff $1 < \beta < 2\alpha^*/\alpha_n - 1$.
For $\beta \in (1, \alpha^*/\alpha_n)$, we have $\beta \alpha_n < \alpha^*$, so the teacher lies strictly between the student and $\pi^*$.
For $\beta \in [\alpha^*/\alpha_n, 2\alpha^*/\alpha_n - 1)$, the teacher reaches or overshoots $\pi^*$ (with equality at $\beta = \alpha^*/\alpha_n$) but remains closer to it than the student is.
\end{proof}

\begin{corollary}[RISE contracts the suboptimality bound]
\label{cor:contraction}
Consider one RISE iteration: RLVR updates $\theta_n \to \theta_{n+1}'$ (advancing the trajectory coefficient $\alpha_n \to \alpha_{n+1}' < \alpha^*$), then OPD distills the extrapolated teacher $\pi_{\mathrm{future}}$ into $\theta_{n+1}'$ to yield $\theta_{n+1}$.
Under the assumptions of Proposition~\ref{prop:extrapolation-gap} (applied with $\pi_{\theta_{n+1}'}$ in the role of the current policy) and $L$-Lipschitz continuity of $J$ in $\varphi$-space, the Lipschitz suboptimality bound contracts:
\begin{equation}
    \underbrace{J(\pi^*) - J(\pi_{\theta_{n+1}})}_{\text{post-OPD}}
    \;\leq\; \underbrace{\gamma(\beta)}_{{<}\, 1}
    \cdot \underbrace{L(\alpha^* - \alpha_{n+1}')}_{\geq\; J(\pi^*) - J(\pi_{\theta_{n+1}'})\;\text{(pre-OPD)}}
    \;+\; T\sqrt{\bar{\mathcal{L}}_{\mathrm{OPD}}/2},
    \label{eq:contraction}
\end{equation}
where the contraction factor is governed by the teacher's distance to $\pi^*$,
\begin{equation}
    \gamma(\beta) = \frac{\bigl|\alpha^* - \beta\alpha_{n+1}'\bigr|}{\alpha^* - \alpha_{n+1}'},
    \label{eq:gamma}
\end{equation}
and $\gamma(\beta) < 1$ exactly when $\beta$ lies in the range $(1,\, 2\alpha^*/\alpha_{n+1}' - 1)$ established by Proposition~\ref{prop:extrapolation-gap}.
In the idealized limit $\bar{\mathcal{L}}_{\mathrm{OPD}} \to 0$ the residual term vanishes and the bound contracts by $\gamma(\beta) < 1$.
In practice, OPD runs for only a few gradient steps and deliberately does not converge to $\pi_{\mathrm{future}}$---it acts as a trust-region projection (\S\ref{sec:method}) that moves the policy \emph{toward} the extrapolated teacher while remaining anchored near $\pi_{\theta_{n+1}'}$.
The corollary's role is therefore directional: it identifies the extrapolated teacher as a descent direction for the bound rather than prescribing full convergence, which \S\ref{sec:exp-analysis} confirms is unnecessary and even harmful (the w/o-OPD variant that adopts $\pi_{\mathrm{future}}$ directly shows negligible gains).
\end{corollary}

\begin{proof}
Before OPD, the post-RLVR student $\pi_{\theta_{n+1}'}$ lies at $\varphi$-distance $\alpha^* - \alpha_{n+1}'$ from $\pi^*$, so $L$-Lipschitz continuity of $J$ in $\varphi$-space gives
\begin{equation}
    J(\pi^*) - J(\pi_{\theta_{n+1}'}) \leq L(\alpha^* - \alpha_{n+1}').
\end{equation}

The extrapolated teacher sits at $\varphi$-distance $|\alpha^* - \beta\alpha_{n+1}'|$ from $\pi^*$.
Applying $L$-Lipschitz continuity again and factoring out the pre-OPD bound:
\begin{equation}
    J(\pi^*) - J(\pi_{\mathrm{future}}) \leq L\bigl|\alpha^* - \beta\alpha_{n+1}'\bigr|
    = \frac{\bigl|\alpha^* - \beta\alpha_{n+1}'\bigr|}{\alpha^* - \alpha_{n+1}'} \cdot L(\alpha^* - \alpha_{n+1}')
    = \gamma(\beta) \cdot L(\alpha^* - \alpha_{n+1}').
\end{equation}
Lipschitz continuity therefore constrains $J(\pi_{\mathrm{future}})$ through the teacher's distance to $\pi^*$ alone, irrespective of which side of $\pi^*$ the teacher falls on; this is what makes Eq.~\ref{eq:gamma} valid for overshooting as well as non-overshooting $\beta$.

We then apply Theorem~\ref{thm:suboptimality} with teacher $\pi_T = \pi_{\mathrm{future}}$ and student $\pi_\theta = \pi_{\theta_{n+1}}$, the checkpoint produced by OPD.
Substituting the teacher gap bounded above gives
\begin{equation}
    J(\pi^*) - J(\pi_{\theta_{n+1}})
    \;\leq\; \underbrace{J(\pi^*) - J(\pi_{\mathrm{future}})}_{\leq\, \gamma(\beta) L(\alpha^* - \alpha_{n+1}')}
    \;+\; T\sqrt{\bar{\mathcal{L}}_{\mathrm{OPD}}/2},
\end{equation}
which is Eq.~\ref{eq:contraction}.
Finally, $\gamma(\beta) < 1$ holds iff $|\alpha^* - \beta\alpha_{n+1}'| < \alpha^* - \alpha_{n+1}'$, which by Proposition~\ref{prop:extrapolation-gap} is exactly the condition $1 < \beta < 2\alpha^*/\alpha_{n+1}' - 1$.
Therefore, once the distillation error vanishes, the post-OPD bound is $\gamma(\beta)$ times the pre-OPD bound, hence strictly smaller.
\end{proof}

\begin{remark}[Limitations of the linear assumption]
\label{rem:linear-limitation}
Proposition~\ref{prop:extrapolation-gap} and Corollary~\ref{cor:contraction} assume a perfectly linear trajectory in $\varphi$-space.
In practice, linearity holds only approximately and locally: empirical studies of task vectors and linear mode connectivity~\citep{ilharco2022editing,frankle2020linear} show that interpolation and modest extrapolation preserve performance, but large extrapolation factors $\beta$ can exit the linear regime and degrade the teacher quality.
The bounds above should therefore be understood as characterizing the \emph{mechanism} by which extrapolation helps---a teacher closer to $\pi^*$ yields a tighter suboptimality bound---rather than as providing exact convergence rates.
In our experiments, we find that moderate values of $\beta$ consistently improve over self-distillation, consistent with this local-linearity picture.
\end{remark}

\subsection{How Low-Dimensional Are Our Trajectories?}
\label{app:linearity}

We quantify the gap between the linear assumption and the actual training dynamics, refining the low-rank picture reported in prior work~\citep{cai2025predictability,wang2026not} with a direct measurement on our own runs. For GRPO checkpoints of Qwen3-1.7B and Qwen3-8B on DAPOMath, we form the displacement $\delta_t = \theta_t - \theta_0$ at each saved step and compute the Gram matrix $\langle \delta_i, \delta_j \rangle$ over all parameter tensors; its eigenspectrum gives the principal components of the trajectory (Fig.~\ref{fig:trajectory-pca}). An exactly linear trajectory would be rank one, placing all variance in a single component. The leading component accounts for $68.2\%$ at \emph{both} scales, and the top three for $86.6\%$ and $88.5\%$ respectively. That the two agree to within $0.1\%$ on the leading component despite a $5\times$ difference in parameter count suggests this is a property of the training dynamics rather than of a particular model.

\begin{figure}[h]
    \centering
    \includegraphics[width=0.52\textwidth]{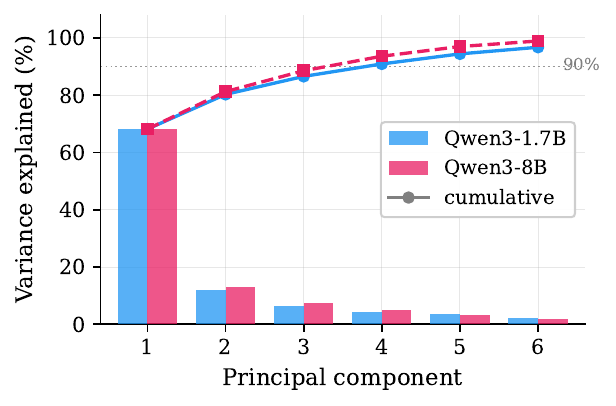}
    \vspace{-6pt}
    \caption{\looseness=-1 \textbf{The RLVR weight trajectory is low-dimensional.} Principal spectrum of the displacements $\delta_t = \theta_t - \theta_0$ from GRPO checkpoints on DAPOMath. Bars show variance explained per component; lines show the cumulative total. A perfectly linear trajectory would place $100\%$ on the first component; the observed $68.2\%$ reflects curvature along the path, while three components capture ${\sim}87\%$ at both scales, confining the trajectory to a low-dimensional subspace.}
    \label{fig:trajectory-pca}
    \vspace{-8pt}
\end{figure}

A single direction therefore does not capture the trajectory exactly: the residual $31.8\%$ lies orthogonal to the leading component. What the measurement does show is that the trajectory is confined to a genuinely low-dimensional subspace: three components capture ${\sim}87\%$ of the variance at both scales, out of a parameter space with billions of degrees of freedom. This confinement bounds the degree of non-linearity---within a rank-3 subspace the leading direction accounts for ${\sim}78\%$ of the subspace variance, limiting the magnitude of orthogonal curvature. For RISE's operating range ($\beta \leq 1.2$, i.e.\ extending $20\%$ beyond the current displacement), the error introduced by the linear approximation is therefore controlled by the small residual variance in the orthogonal components.

Fig.~\ref{fig:extrap-sweep} provides empirical confirmation: extrapolation gains persist out to $\beta \approx 1.3$ and the safe range contracts as the policy converges, consistent with Proposition~\ref{prop:extrapolation-gap}'s prediction even though its linear premise holds only approximately.

\subsection{Top-\texorpdfstring{$K$}{K} Approximation Bias}
\label{app:topk-bias}

As described in \S\ref{sec:method-design}, we approximate the logit-space extrapolated teacher by selecting $S = \text{Top}_K(\pi_{\theta_{n+1}'})$ and projecting the post-RLVR checkpoint $\pi_{\theta_{n+1}'}$ and the anchor $\pi_{\theta_n}$ onto the same $(K\!+\!1)$-simplex before applying the extrapolation (Eq.~\ref{eq:topk-extrapolation}).
This differs from the unbiased procedure, which would extrapolate over the full vocabulary \emph{first} and then select
$S^* = \text{Top}_K(\pi_{\text{future}})$.
The approximation error is captured by the set difference $S \triangle S^*$: tokens outside $\pi_{\theta_{n+1}'}$'s top-$K$ but with a large log-ratio $\log \pi_{\theta_{n+1}'} / \pi_{\theta_n}$ can be promoted into $S^*$ by extrapolation, yet are absent from $S$.

In practice this discrepancy is negligible.
With $K\!=\!100$, the top-$K$ tokens of $\pi_{\theta_{n+1}'}$ account for $>\!99\%$ of the probability mass under typical autoregressive LLM distributions, leaving a residual tail mass $\varepsilon = 1 - \sum_{v \in S}\pi_{\theta_{n+1}'}(v) < 0.01$.
This does not formally bound $|S \triangle S^*|$, since extrapolation with $\beta > 1$ can in principle amplify tail tokens with large log-ratios.
However, for the moderate $\beta$ values used in our experiments, almost all of $\pi_{\text{future}}$'s mass concentrates on tokens already in $S$, so the induced bias is negligible in practice.
For weight-space extrapolation, no such approximation arises: the forward pass through $\theta_{\text{future}}$ directly produces $\pi_{\text{future}}$, and top-$K$ is applied to the exact extrapolated distribution.

\section{Experiment Setup}
\label{sec:app-exp-setup}

This appendix supplements the setup summary in \S\ref{sec:experiments} with full training hyperparameters (\S\ref{sec:app-train-setup}), evaluation protocol details (\S\ref{sec:app-eval-setup}), and a description of the OPSD baselines and their shared privileged-teacher construction (\S\ref{sec:app-baseline}).

\subsection{Training Setup}
\label{sec:app-train-setup}

All runs use the VeRL framework on a single node with 8 GPUs, training for one epoch over the respective training set.
We use AdamW with a constant learning rate and no warmup, GRPO advantage estimation without standard-deviation normalization, and token-level importance-sampling correction (clip threshold $2.0$).
Prompts are truncated to $2{,}048$ tokens and responses to $8{,}192$ tokens.
All methods---GRPO, the three OPSD baselines, and both RISE variants---share the hyperparameters in Table~\ref{tab:app-train-hparams} within each configuration, so performance differences are attributable to the teacher construction rather than training dynamics.

\begin{table}[h]
\centering
\small
\setlength{\tabcolsep}{4pt}
\caption{Training hyperparameters. $n$ denotes the number of rollouts sampled per prompt. \emph{Train batch size} is the number of prompts sampled per iteration for rollout generation; \emph{mini-batch size} is the number of prompts per gradient update within each iteration (the train batch is split into $\text{train batch size}/\text{mini-batch size}$ gradient steps). The OPD phase uses the same mini-batch size as the RLVR phase. Qwen2.5-3B agentic runs follow GIGPO's setup~\citep{feng2025group}; see their paper for full details.}
\label{tab:app-train-hparams}
\scriptsize
\setlength{\tabcolsep}{3pt}
\begin{tabular}{lcccccc}
\toprule
& Qwen3-1.7B & Qwen3-1.7B-Base & Qwen3-8B & Qwen3-4B-Base & Qwen3-8B-Base & OLMo3-7B-Instruct-SFT \\
& (DAPOMath) & (DAPOMath) & (DAPOMath) & (multi-domain) & (code) & (OpenR1-Math) \\
\midrule
Learning rate      & 1e-6 & 1e-6 & 1e-6 & 1e-6 & 1e-6 & 1e-6 \\
Train batch size   & 64   & 64   & 64   & 64   & 64   & 128  \\
Mini-batch size    & 32   & 32   & 32   & 32   & 32   & 64   \\
Rollouts $n$       & 8    & 8    & 8    & 8    & 8    & 8    \\
Epochs             & 1    & 1    & 1    & 1    & 1    & 1    \\
Rollout temperature & 1.0 & 1.0 & 1.0 & 1.0 & 1.0 & 1.0 \\
KL coefficient     & 0    & 0    & 0    & 0    & 0    & 0    \\
\midrule
RISE $\beta_0$     & 1.2  & 1.2  & 1.2  & 1.2  & 1.2  & 1.2  \\
RISE $\beta_N$     & 1.0  & 1.0  & 1.0  & 1.0  & 1.0  & 1.0  \\
RISE $\eta$        & 0.1  & 0.1  & 0.1  & 0.1  & 0.1  & 1.0  \\
RISE $K$           & 100  & 100  & 100  & 100  & 20   & 100  \\
\bottomrule
\end{tabular}
\end{table}

\subsection{Evaluation Setup}
\label{sec:app-eval-setup}

\textbf{Benchmarks.~}
Mathematical reasoning is evaluated on MATH-500~\citep{lightman2024let}, AIME 2024/2025~\citep{aime_2024,aime_2025}, AMC 2023~\citep{amc2023}, Minerva~\citep{lewkowycz2022solving}, and OlympiadBench~\citep{he2024olympiadbench}, with GPQA-Diamond~\citep{rein2023gpqa}, IFEval~\citep{zhou2023instruction}, and MMLU-Pro~\citep{wang2024mmlu} as out-of-distribution checks. The multi-domain suite adds SuperGPQA~\citep{du2026supergpqa} and TheoremQA~\citep{chen2023theoremqa}. Code generation is evaluated on HumanEval+, MBPP+~\citep{evalplus}, and LiveCodeBench v6~\citep{jain2025livecodebench}.
Agentic tasks use ALFWorld~\citep{shridhar2020alfworld} (success rate: fraction of episodes completed) and WebShop~\citep{yao2022webshop}, which reports two metrics: \emph{Score} ($100\times$ average reward, where reward is a product of attribute recall, option selection, price satisfaction, and category match) and \emph{Acc} (success rate: fraction of episodes achieving a perfect reward of 1.0).

\textbf{Sampling protocol.~}
Evaluation samples are drawn stochastically with temperature $1.0$, $\text{top-}p = 1.0$, and $\text{top-}k = -1$ (i.e., unrestricted), matching the rollout distribution used during training.
The number of samples per prompt is set per benchmark according to its size and variance: $16$ for AIME 2024/2025, GPQA-Diamond, and Minerva, $8$ for OlympiadBench, and $4$ for MATH-500 and AMC 2023.
Deterministic benchmarks that admit a single correct completion---IFEval, MMLU-Pro, SuperGPQA, and TheoremQA---use a single sample.
All Qwen3 models are evaluated in non-thinking mode (thinking is disabled).
For code generation we use $4$ samples on HumanEval+, MBPP+, and LiveCodeBench.

We report two complementary metrics.
\emph{avg@$N$} is the mean accuracy across $N$ sampled completions, measuring expected single-sample quality.
\emph{pass@$N$} estimates the probability that at least one of $N$ independent draws is correct: for each problem we draw 1{,}000 bootstrap subsets of size $N$ (with replacement) from the $N$ sampled completions, record whether each subset contains a correct solution, and average; this per-problem estimate is then averaged across prompts.
Reporting both separates per-sample reliability from exploration breadth: a method can improve avg@$N$ while leaving pass@$N$ unchanged (sharpening the distribution) or improve both (genuinely expanding coverage).

\subsection{Baselines}
\label{sec:app-baseline}

All three OPSD baselines---GRPO+SDPO~\citep{hubotter2026reinforcement}, SDAR~\citep{lu2026self}, and RLSD~\citep{yang2026self}---share the same \emph{privileged teacher} and differ only in how its signal enters the update.

\textbf{Privileged teacher construction.~}
For each prompt, the model generates $n$ rollouts. If the group contains at least one correct solution, the teacher is the model re-prompted with that correct solution as privileged context, and its token distribution serves as the distillation target for all responses to that prompt. If no rollout in the group succeeds, the prompt contributes only the RLVR loss---no distillation signal is available.
The teacher is maintained as an exponential moving average of the student (rate $0.05$), following the recommended hyperparameters in \citet{hubotter2026reinforcement}.
Averaged over training, the teacher is active on ${\sim}85\%$ of prompts for Qwen3-8B and ${\sim}70\%$ for Qwen3-1.7B, so the baselines' limited gains are not attributable to low coverage.

\textbf{Integration strategies.~}
GRPO+SDPO adds a KL divergence toward the teacher's token distribution as an auxiliary loss alongside the policy gradient.
SDAR uses the teacher only as a \emph{gate}: the teacher--student probability gap scales the GRPO advantage per token, upweighting positions where the teacher is more confident than the student.
RLSD similarly uses the teacher's per-token signal to reweight advantage estimates, providing finer-grained credit assignment than a single sequence-level reward.

\textbf{Contrast with RISE.~}
All three baselines require privileged information---a correct solution the student did not produce---so their teacher is available only on problems where such a solution exists in the rollout group. RISE derives its teacher from the training trajectory itself and therefore applies uniformly to every prompt, regardless of whether any rollout succeeded.

\section{Experiment Results}
\label{sec:app-exp-results}

This appendix provides supplementary results for the experiments in \S\ref{sec:experiments}: Qwen3-1.7B-Base results (\S\ref{sec:app-1.7b-results}), per-metric convergence curves (\S\ref{sec:app-convergence-curves}), code generation convergence (\S\ref{sec:app-code}), grounding and OPD ablation details (\S\ref{sec:app-grounding-opd}), extrapolation strength and schedule ablations (\S\ref{sec:app-beta-schedule}--\S\ref{sec:app-anchor}), a compute-matched comparison (\S\ref{sec:app-compute-matched}), and a multi-seed reproducibility analysis (\S\ref{sec:app-seed-variance}).

\subsection{Qwen3-1.7B-Base Results}
\label{sec:app-1.7b-results}

Table~\ref{tab:app-1.7b} reports results for Qwen3-1.7B-Base on DAPOMath. RISE (logit) achieves the best in-domain average (29.2) and competitive OOD performance, while RISE (weight) leads on MATH-500 (62.4) and Minerva (20.7). The margin over GRPO is the narrowest of any configuration ($+1.2$ Math Avg), consistent with the capability ceiling of a base model that has not been instruction-tuned.

\begin{table}[h]
\centering
\caption{\textbf{Qwen3-1.7B-Base (DAPOMath).} Accuracy (\%) on math and OOD benchmarks. Best in \textbf{bold}, second best \underline{underlined}.}
\label{tab:app-1.7b}
\vspace{2pt}
\small
\setlength{\tabcolsep}{2.5pt}
\begin{tabular}{l|cccccc>{\columncolor{yellow!15}}c|ccc>{\columncolor{yellow!15}}c}
\toprule
\multirow{2}{*}{\textbf{Method}} & \multicolumn{7}{c|}{\textbf{In-Domain Math}} & \multicolumn{4}{c}{\textbf{OOD}} \\
& {\scriptsize MATH500} & {\scriptsize AIME24} & {\scriptsize AIME25} & {\scriptsize AMC23} & {\scriptsize Minerva} & {\scriptsize OlyBench} & {\scriptsize Avg.} & {\scriptsize GPQA} & {\scriptsize IFEval} & {\scriptsize MMLU} & {\scriptsize Avg.} \\
\midrule
Base & 40.0 & 4.4 & 2.9 & 27.5 & 11.6 & 19.8 & 17.7 & 12.7 & 14.4 & 8.0 & 11.7 \\
GRPO & 60.6 & 9.4 & 6.9 & 42.5 & 19.3 & 29.2 & 28.0 & \underline{26.5} & \underline{31.1} & 39.6 & 32.4 \\
GRPO+SDPO & \underline{61.3} & 8.3 & 6.7 & 41.9 & 20.1 & \textbf{32.3} & \underline{28.4} & \textbf{28.7} & 30.5 & 40.5 & \textbf{33.2} \\
SDAR & 55.6 & 9.6 & \textbf{8.3} & \underline{43.1} & 17.7 & 27.1 & 26.9 & 23.2 & 28.5 & \underline{42.5} & 31.4 \\
RLSD & 57.2 & \underline{11.9} & 7.5 & \underline{43.1} & 17.8 & 28.2 & 27.6 & 23.0 & 29.0 & 42.4 & 31.5 \\
\textbf{RISE (logit)} & 59.6 & \textbf{12.1} & \underline{8.1} & \textbf{45.6} & \underline{20.2} & 29.4 & \textbf{29.2} & 25.4 & \textbf{31.1} & \textbf{42.9} & \underline{33.1} \\
\textbf{RISE (weight)} & \textbf{62.4} & 10.4 & 5.8 & 38.8 & \textbf{20.7} & \underline{30.8} & 28.1 & 25.7 & 29.9 & 41.0 & 32.2 \\
\bottomrule
\end{tabular}
\end{table}

\subsection{Per-Metric Convergence Curves}
\label{sec:app-convergence-curves}

Figure~\ref{fig:app-conv-all} shows per-metric training curves for all model configurations discussed in \S\ref{sec:exp-results}.

\begin{figure}[ht]
\centering

\begin{subfigure}[t]{0.24\textwidth}
    \includegraphics[width=\linewidth]{figures/1.7b/qwen3_1.7b_aime24_avg16.pdf}
\end{subfigure}
\hfill
\begin{subfigure}[t]{0.24\textwidth}
    \includegraphics[width=\linewidth]{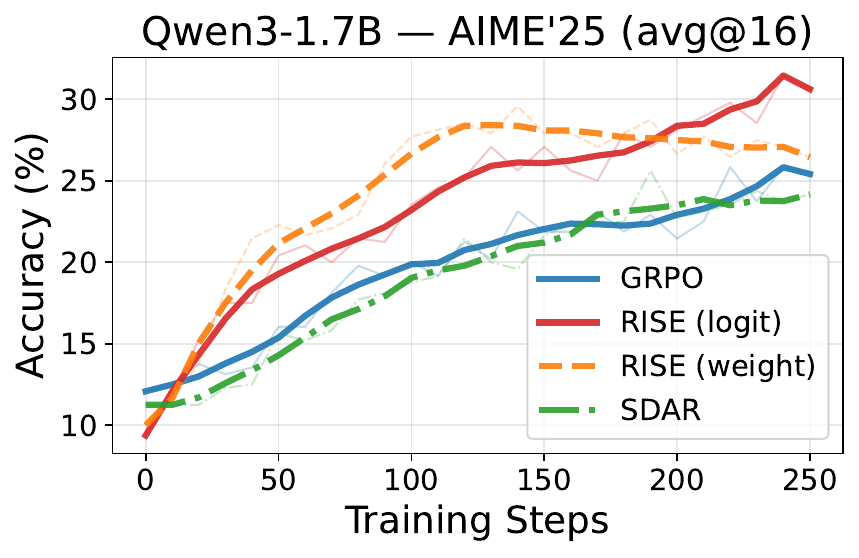}
\end{subfigure}
\hfill
\begin{subfigure}[t]{0.24\textwidth}
    \includegraphics[width=\linewidth]{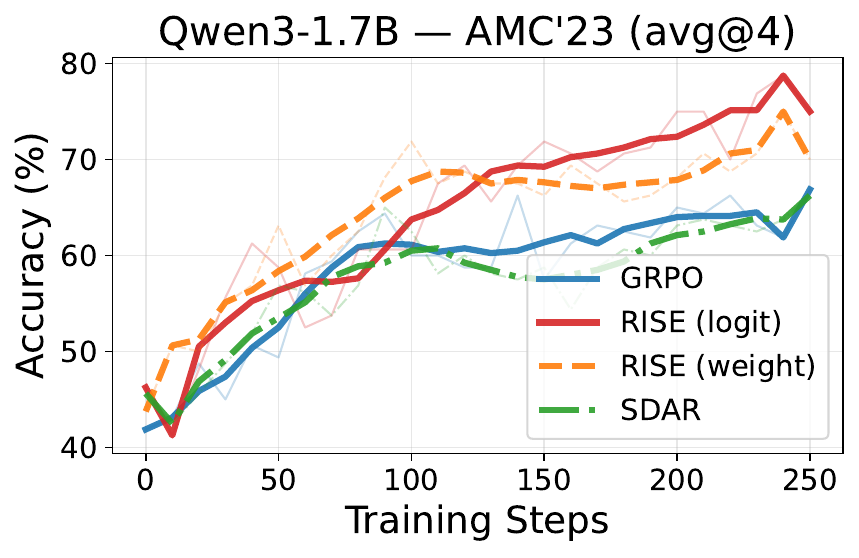}
\end{subfigure}
\hfill
\begin{subfigure}[t]{0.24\textwidth}
    \includegraphics[width=\linewidth]{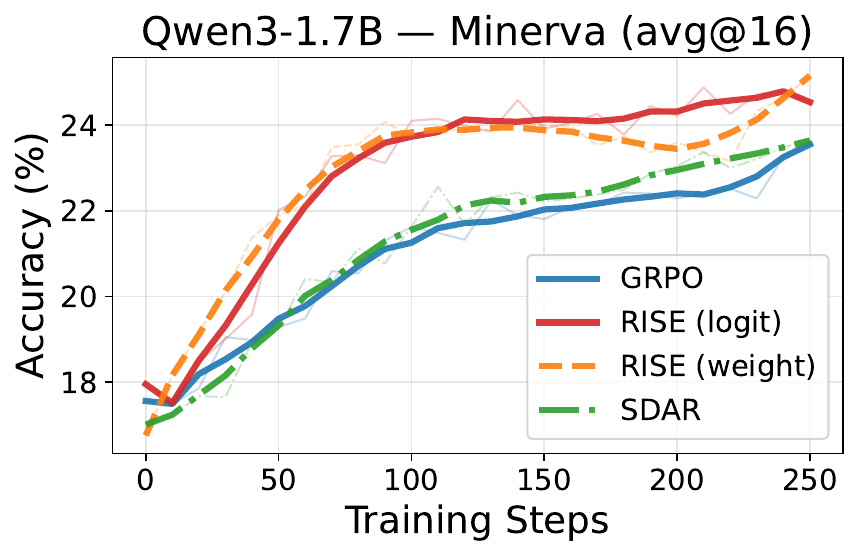}
\end{subfigure}

\vspace{2pt}
{\small (a) Qwen3-1.7B (DAPOMath)}
\label{fig:app-conv-1.7b}

\vspace{8pt}

\begin{subfigure}[t]{0.24\textwidth}
    \includegraphics[width=\linewidth]{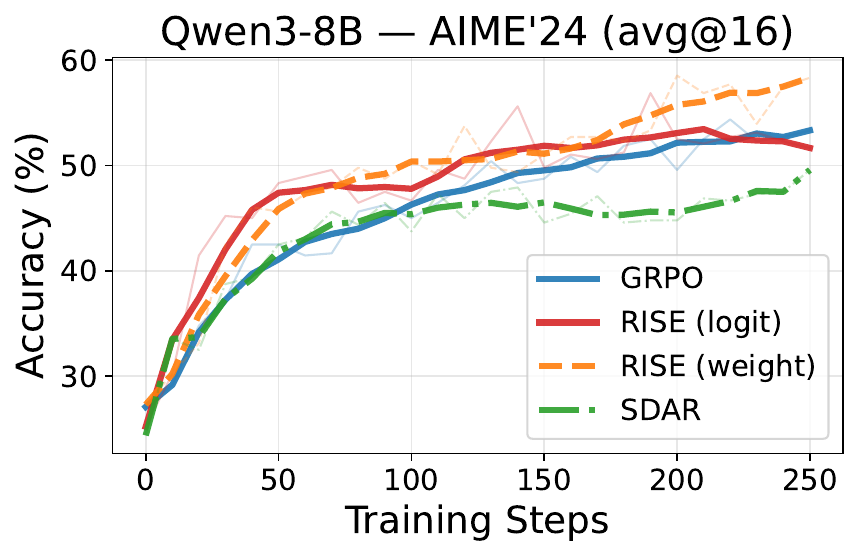}
\end{subfigure}
\hfill
\begin{subfigure}[t]{0.24\textwidth}
    \includegraphics[width=\linewidth]{figures/8b/qwen3_8b_aime25_avg16.pdf}
\end{subfigure}
\hfill
\begin{subfigure}[t]{0.24\textwidth}
    \includegraphics[width=\linewidth]{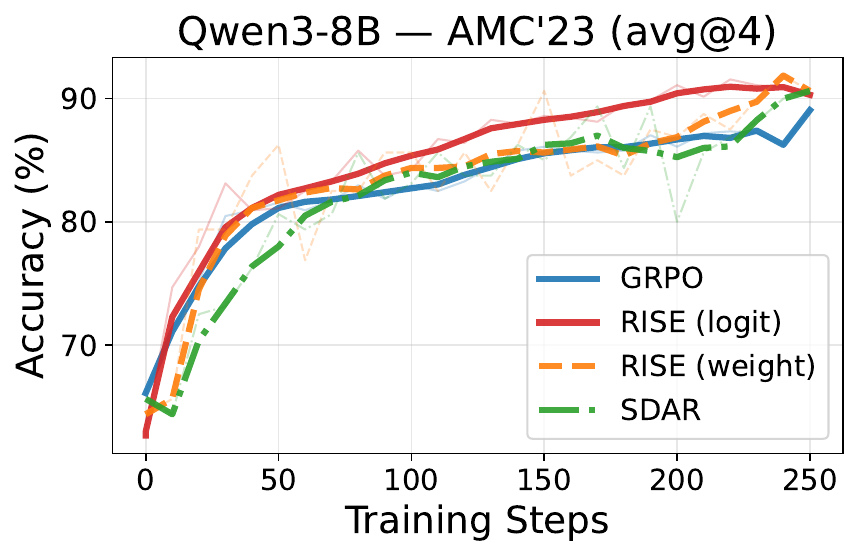}
\end{subfigure}
\hfill
\begin{subfigure}[t]{0.24\textwidth}
    \includegraphics[width=\linewidth]{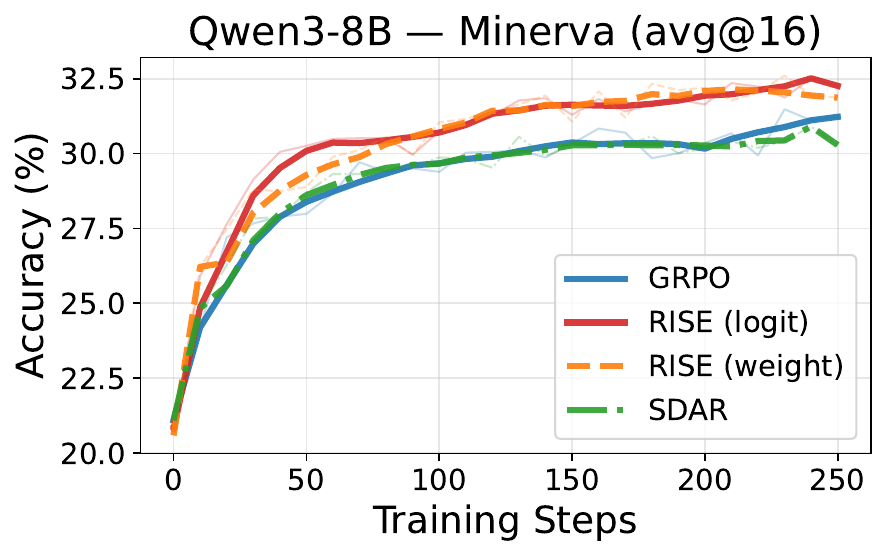}
\end{subfigure}

\vspace{2pt}
{\small (b) Qwen3-8B (DAPOMath)}
\label{fig:app-conv-8b}

\vspace{8pt}

\begin{subfigure}[t]{0.24\textwidth}
    \includegraphics[width=\linewidth]{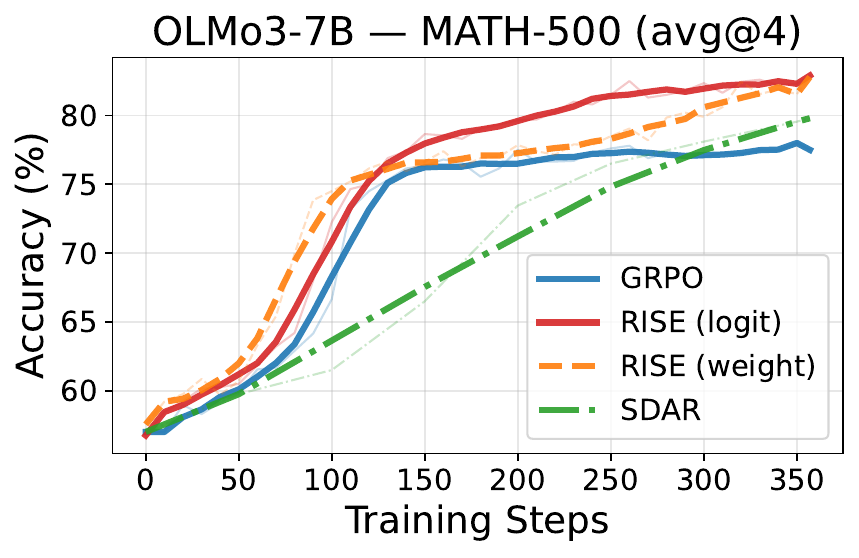}
\end{subfigure}
\hfill
\begin{subfigure}[t]{0.24\textwidth}
    \includegraphics[width=\linewidth]{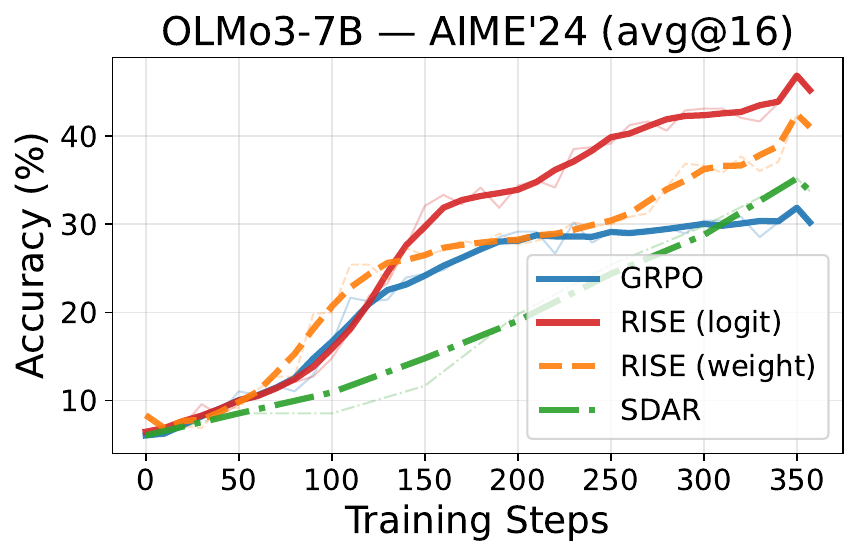}
\end{subfigure}
\hfill
\begin{subfigure}[t]{0.24\textwidth}
    \includegraphics[width=\linewidth]{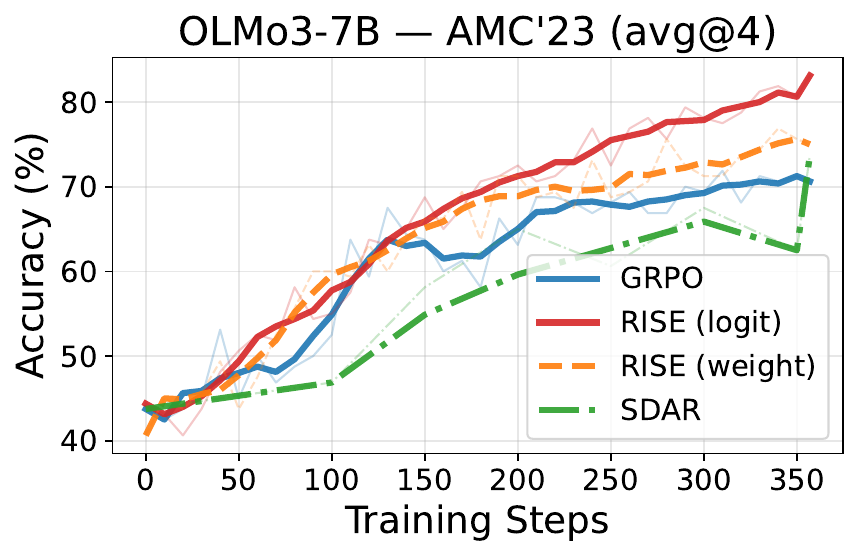}
\end{subfigure}
\hfill
\begin{subfigure}[t]{0.24\textwidth}
    \includegraphics[width=\linewidth]{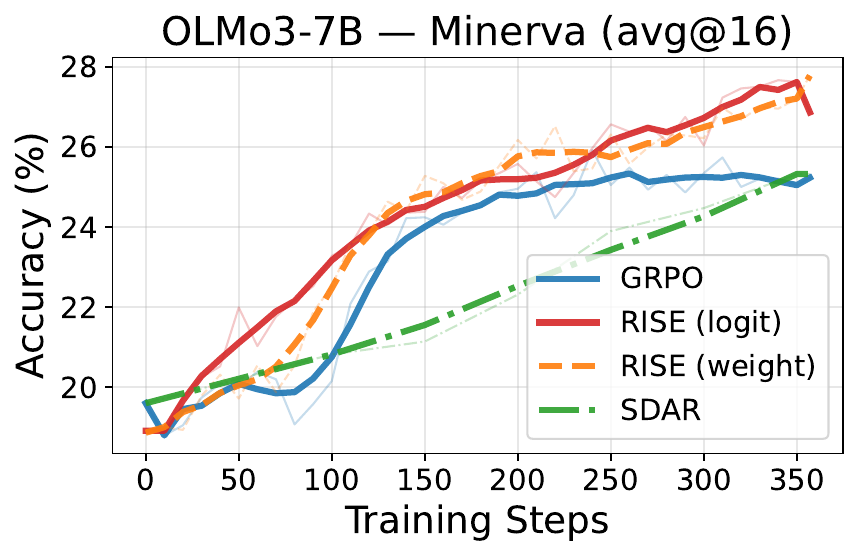}
\end{subfigure}

\vspace{2pt}
{\small (c) OLMo3-7B-Instruct-SFT (OpenR1Math)}
\label{fig:app-conv-olmo7b}

\vspace{8pt}

\begin{subfigure}[t]{0.235\textwidth}
    \includegraphics[width=\linewidth]{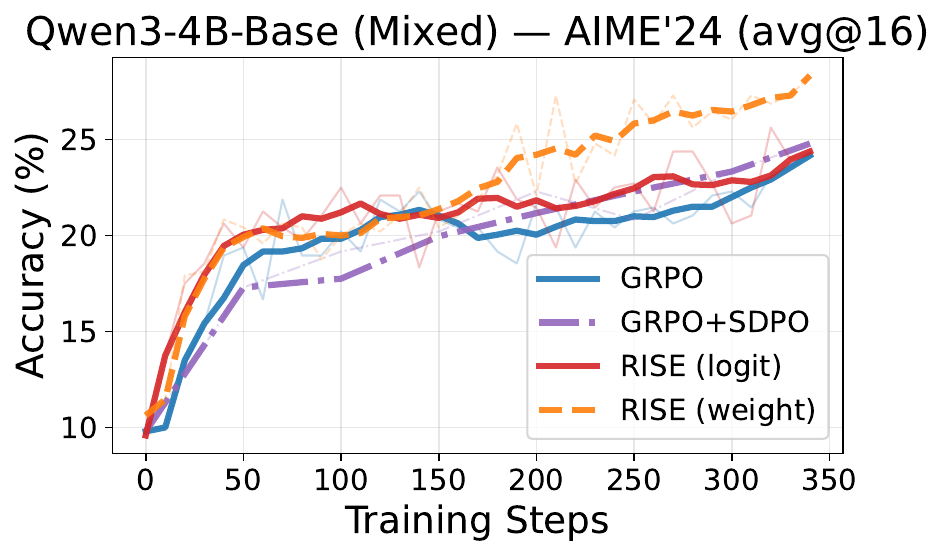}
\end{subfigure}
\hfill
\begin{subfigure}[t]{0.235\textwidth}
    \includegraphics[width=\linewidth]{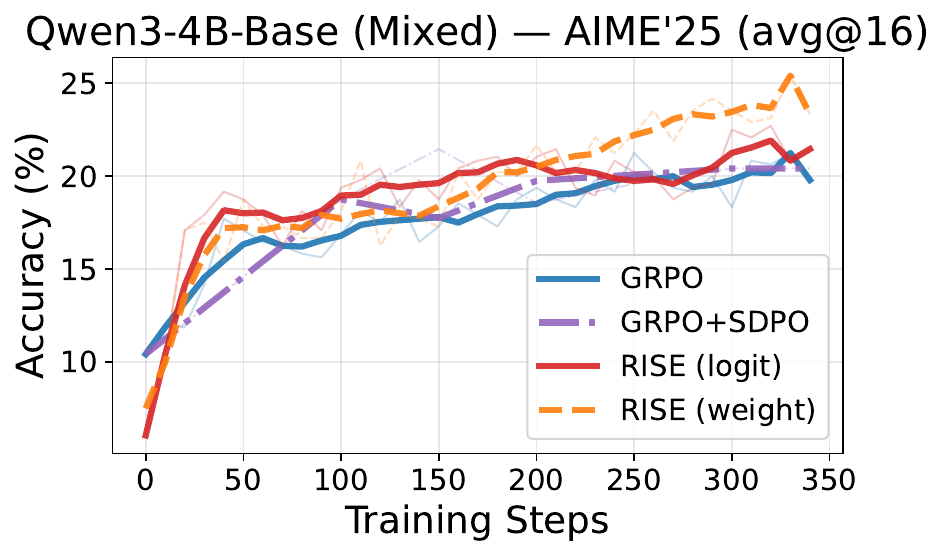}
\end{subfigure}
\hfill
\begin{subfigure}[t]{0.235\textwidth}
    \includegraphics[width=\linewidth]{figures/4b_mixed/qwen3_4b_mixed_amc23_avg4.pdf}
\end{subfigure}
\hfill
\begin{subfigure}[t]{0.275\textwidth}
    \includegraphics[width=\linewidth]{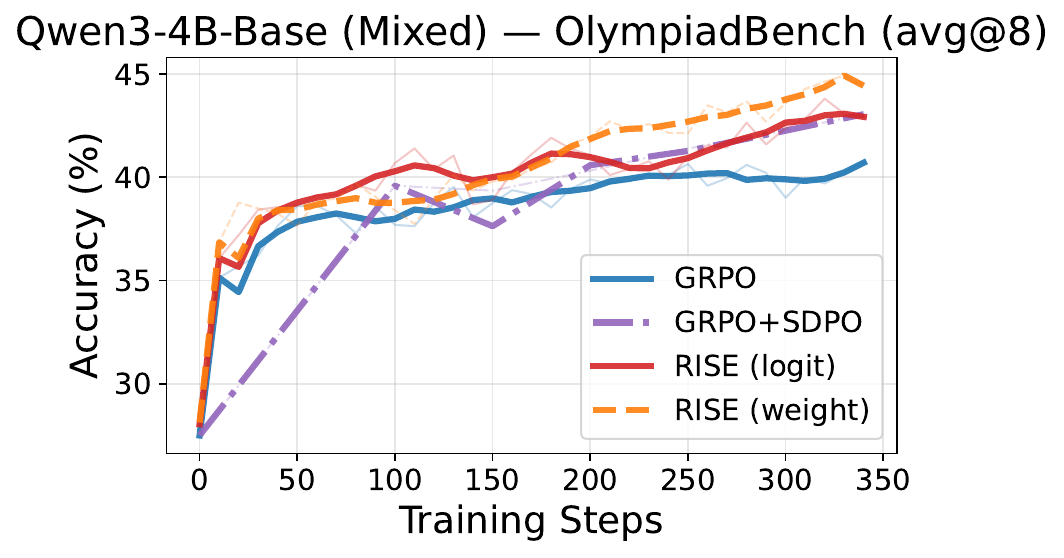}
\end{subfigure}

\vspace{2pt}
{\small (d) Qwen3-4B-Base (mixed math + STEM)}
\label{fig:app-conv-4b-mixed}

\caption{\textbf{Per-metric convergence curves} for all model configurations. Each row shows four representative benchmarks. RISE (both variants) consistently reaches higher accuracy in fewer training steps.}
\label{fig:app-conv-all}
\end{figure}

\subsection{Code Generation Convergence Curves}
\label{sec:app-code}

Figure~\ref{fig:app-code} shows avg@4 and pass@4 convergence curves for Qwen3-8B-Base trained on Skywork-OR1-Code, evaluated on HumanEval+, MBPP+, and LiveCodeBench. Both RISE variants converge faster than GRPO in early training across all three benchmarks.

\begin{figure}[h]
\centering
\begin{subfigure}[t]{0.32\linewidth}
    \centering
    \includegraphics[width=\linewidth]{figures/code_8b/qwen3_8b_code_humanevalplus_avg4.pdf}
    \caption{HumanEval+ (avg@4)}
\end{subfigure}
\hfill
\begin{subfigure}[t]{0.3\linewidth}
    \centering
    \includegraphics[width=\linewidth]{figures/code_8b/qwen3_8b_code_mbppplus_avg4.pdf}
    \caption{MBPP+ (avg@4)}
\end{subfigure}
\hfill
\begin{subfigure}[t]{0.34\linewidth}
    \centering
    \includegraphics[width=\linewidth]{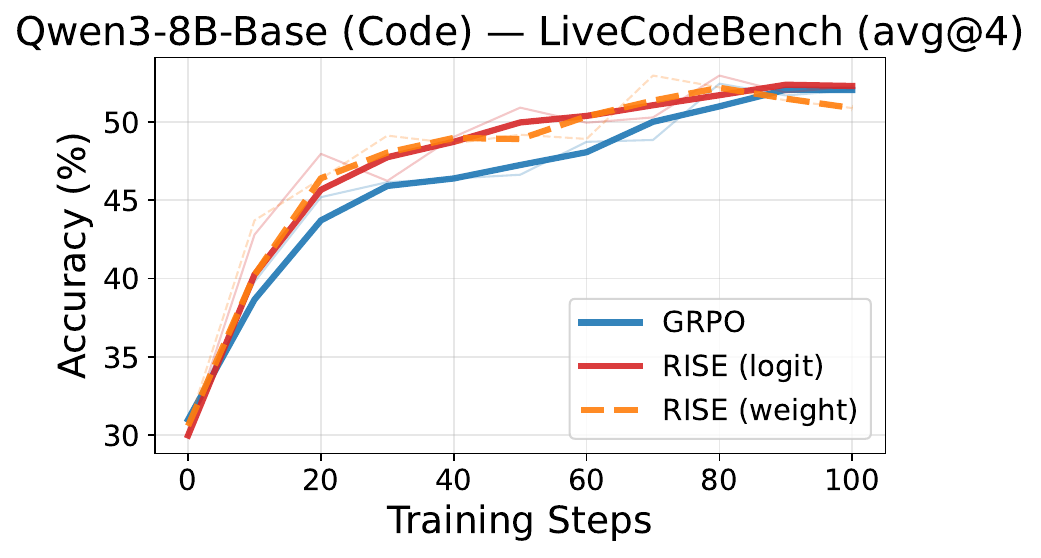}
    \caption{LiveCodeBench (avg@4)}
\end{subfigure}

\vspace{6pt}

\begin{subfigure}[t]{0.32\linewidth}
    \centering
    \includegraphics[width=\linewidth]{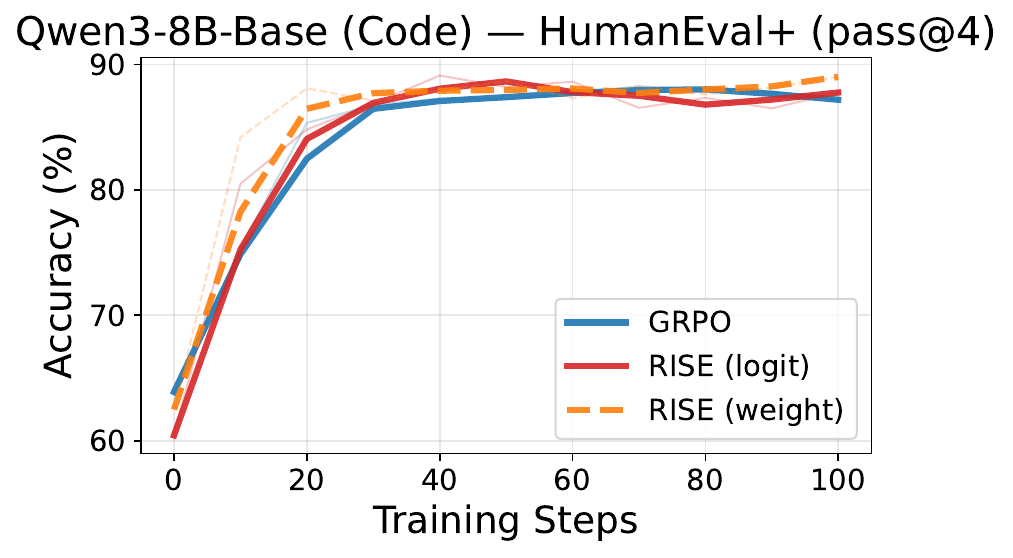}
    \caption{HumanEval+ (pass@4)}
\end{subfigure}
\hfill
\begin{subfigure}[t]{0.3\linewidth}
    \centering
    \includegraphics[width=\linewidth]{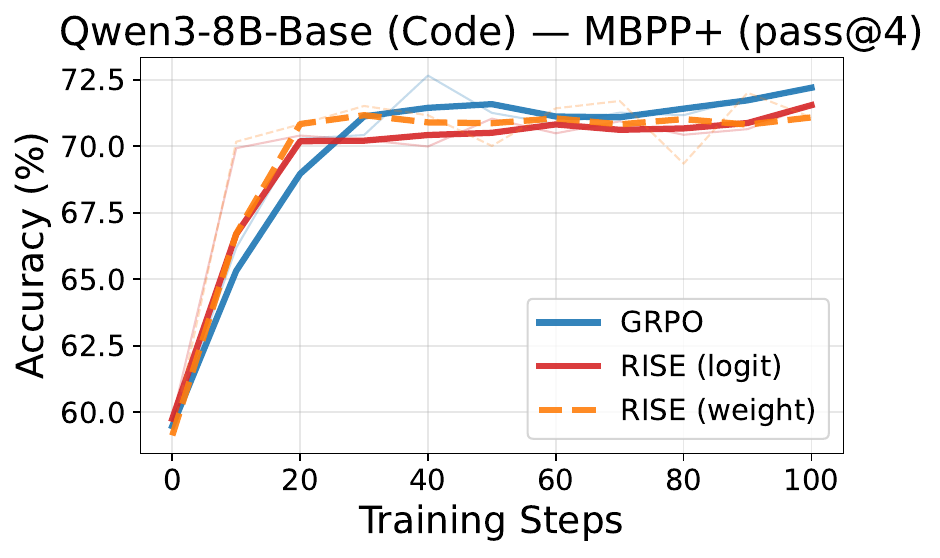}
    \caption{MBPP+ (pass@4)}
\end{subfigure}
\hfill
\begin{subfigure}[t]{0.34\linewidth}
    \centering
    \includegraphics[width=\linewidth]{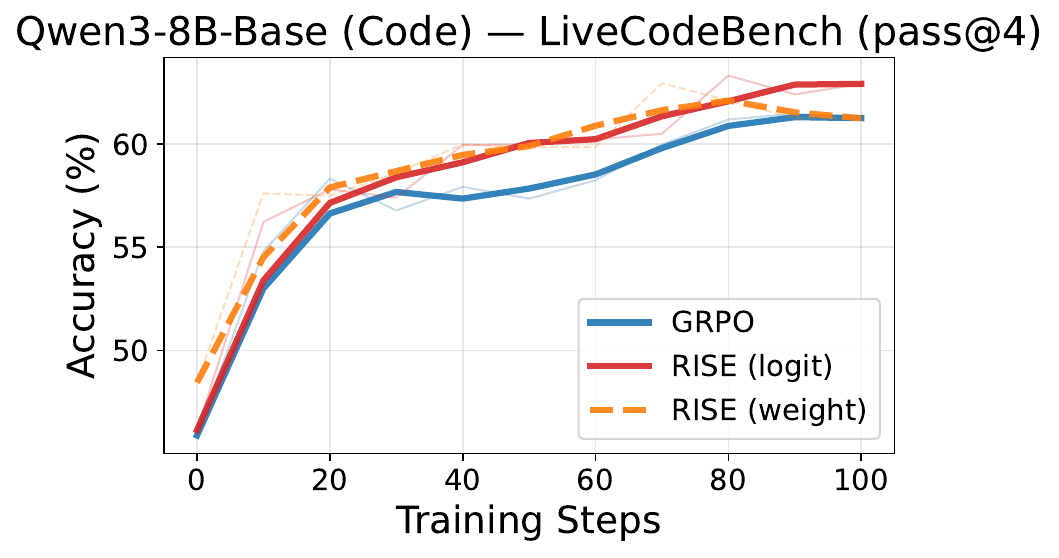}
    \caption{LiveCodeBench (pass@4)}
\end{subfigure}
\caption{\textbf{Code generation convergence curves (Qwen3-8B-Base).} Both RISE variants consistently converge faster than GRPO across all three benchmarks under both avg@4 and pass@4 metrics.}
\label{fig:app-code}
\end{figure}

\subsection{Average vs.\ Best-of-\texorpdfstring{$N$}{N} Accuracy}
\label{sec:app-avg-vs-pass}

Table~\ref{tab:avg-vs-pass} reports avg@16 and pass@16 on the AIME benchmarks, isolating whether RISE's gains come from sharpening the policy around solutions it already finds or from expanding the set of solvable problems. If distillation merely concentrated probability mass on existing solutions, avg@16 would improve while pass@16 stagnated or declined. Instead, RISE improves both, and at 1.7B the pass@16 gains exceed the avg@16 gains on AIME'24 ($+9.6$ vs.\ $+6.3$ for RISE (logit)), indicating genuinely broadened coverage. At 8B the coverage margins are smaller, consistent with a stronger base policy leaving less headroom; the one regression is RISE (logit) on AIME'25 pass@16 ($-0.6$), where the weight-space variant instead gains $+3.9$.

\begin{table}[h]
\centering
\caption{\textbf{Average vs.\ best-of-16 accuracy (\%) on AIME benchmarks.} $\Delta$ denotes improvement over GRPO. RISE improves pass@16 (coverage) alongside avg@16 (mean quality), with the largest coverage gains at 1.7B.}
\label{tab:avg-vs-pass}
\vspace{2pt}
\small
\setlength{\tabcolsep}{3.5pt}
\begin{tabular}{l|l|cc|cc}
\toprule
& & \multicolumn{2}{c|}{\textbf{AIME'24}} & \multicolumn{2}{c}{\textbf{AIME'25}} \\
\textbf{Model} & \textbf{Method} & avg@16 & pass@16 & avg@16 & pass@16 \\
\midrule
\multirow{3}{*}{Qwen3-8B}
& GRPO & 54.4 & 78.7 & 42.9 & 63.3 \\
& RISE (logit) & 56.9 \dt{+2.5} & 79.0 \dt{+0.3} & 46.7 \dt{+3.8} & 62.7 \dt{-0.6} \\
& RISE (weight) & 58.1 \dt{+3.7} & 81.5 \dt{+2.8} & 45.8 \dt{+2.9} & 67.2 \dt{+3.9} \\
\midrule
\multirow{3}{*}{Qwen3-1.7B}
& GRPO & 30.0 & 54.4 & 26.3 & 45.8 \\
& RISE (logit) & 36.3 \dt{+6.3} & 64.0 \dt{+9.6} & 33.3 \dt{+7.0} & 53.9 \dt{+8.1} \\
& RISE (weight) & 32.9 \dt{+2.9} & 57.0 \dt{+2.6} & 31.3 \dt{+5.0} & 49.2 \dt{+3.4} \\
\bottomrule
\end{tabular}
\end{table}

\subsection{Grounding and OPD Ablation}
\label{sec:app-grounding-opd}

Figure~\ref{fig:app-grounding-reward-entropy} plots training reward and policy entropy for the grounding ablation (\S\ref{sec:exp-analysis}). Without RLVR, mean training reward collapses from 0.56 to 0.00 by step~60, tracking the accuracy cliff in Fig.~\ref{fig:analysis-grounding}. Policy entropy spikes to 0.78 at step~49---just before the final crash---then drops to near zero, consistent with the policy degenerating into repetitive non-terminating generation. GRPO and RISE show stable entropy throughout.

\begin{figure}[h]
\centering
\includegraphics[width=0.85\textwidth]{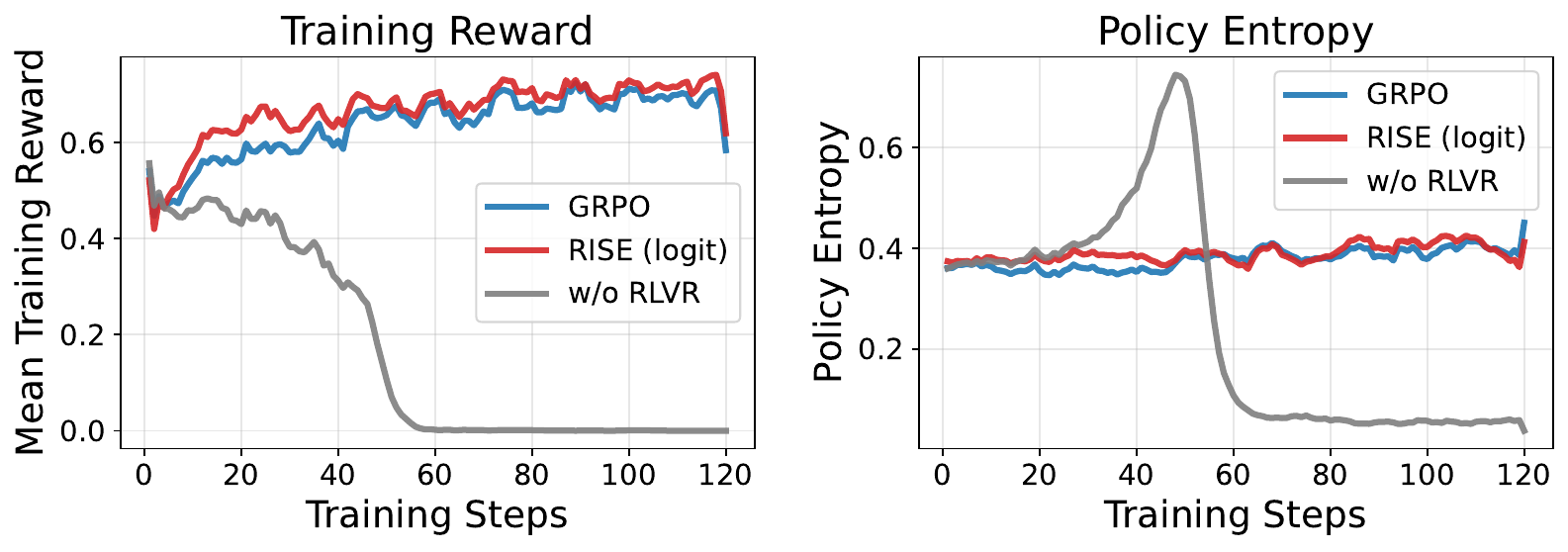}
\caption{\textbf{Grounding ablation: training reward and policy entropy.} Same three Qwen3-8B runs as Fig.~\ref{fig:analysis-grounding}. Without RLVR (grey), training reward collapses to zero (left) while policy entropy spikes then crashes (right), confirming the phase-transition nature of the collapse.}
\label{fig:app-grounding-reward-entropy}
\end{figure}

Table~\ref{tab:app-no-opd} reports the full without-OPD ablation results discussed in \S\ref{sec:exp-analysis}. Across all three model scales, removing the OPD phase (i.e., directly adopting $\theta_{\text{future}}$ without distillation) leaves Math Avg essentially unchanged relative to GRPO (60.0$\to$60.3 at 8B, 45.4$\to$45.6 at 1.7B, 28.0$\to$27.7 at 1.7B-Base), forfeiting the gains that RISE achieves. Per-benchmark accuracy shifts---e.g., the no-OPD arm gains on OlympiadBench at 8B but loses on the AIME pair, while the pattern reverses at 1.7B---but these redistributions do not translate into consistent improvement without the distillation phase.

\begin{table}[h]
\centering
\caption{\textbf{Without-OPD ablation.} Accuracy (\%) on math benchmarks. The w/o OPD variant adopts $\theta_{\text{future}}$ directly without distillation. All runs are configuration-matched to their RISE counterpart.}
\label{tab:app-no-opd}
\vspace{2pt}
\small
\setlength{\tabcolsep}{3pt}
\begin{tabular}{l|l|cccccc>{\columncolor{yellow!15}}c}
\toprule
\textbf{Model} & \textbf{Method} & {\scriptsize MATH500} & {\scriptsize AIME24} & {\scriptsize AIME25} & {\scriptsize AMC23} & {\scriptsize Minerva} & {\scriptsize OlyBench} & {\scriptsize Avg.} \\
\midrule
\multirow{4}{*}{Qwen3-8B}
& GRPO & 83.8 & 54.4 & 42.9 & 89.4 & 31.5 & 57.9 & 60.0 \\
& RISE (logit) & 84.8 & 56.9 & 46.7 & 91.6 & 32.5 & 62.4 & 62.5 \\
& RISE (weight) & 84.4 & 58.1 & 45.8 & 91.9 & 32.6 & 63.3 & 62.7 \\
& w/o OPD & 84.6 & 51.7 & 39.8 & 90.6 & 31.2 & 64.1 & 60.3 \\
\midrule
\multirow{4}{*}{Qwen3-1.7B}
& GRPO & 75.6 & 30.0 & 26.3 & 65.6 & 23.9 & 51.1 & 45.4 \\
& RISE (logit) & 75.2 & 36.3 & 33.3 & 78.8 & 24.9 & 52.9 & 50.2 \\
& RISE (weight) & 77.3 & 32.9 & 31.3 & 75.0 & 25.7 & 53.3 & 49.2 \\
& w/o OPD & 72.0 & 32.3 & 27.3 & 70.0 & 24.6 & 47.4 & 45.6 \\
\midrule
\multirow{4}{*}{\shortstack[l]{Qwen3-1.7B\\-Base}}
& GRPO & 60.6 & 9.4 & 6.9 & 42.5 & 19.3 & 29.2 & 28.0 \\
& RISE (logit) & 59.6 & 12.1 & 8.1 & 45.6 & 20.2 & 29.4 & 29.2 \\
& RISE (weight) & 62.4 & 10.4 & 5.8 & 38.8 & 20.7 & 30.8 & 28.1 \\
& w/o OPD & 58.8 & 10.6 & 5.8 & 43.1 & 19.4 & 28.6 & 27.7 \\
\bottomrule
\end{tabular}
\end{table}

\subsection{Extrapolation Strength and Decay Schedule}
\label{sec:app-beta-schedule}

Table~\ref{tab:app-ablation-detail} reports per-benchmark results for the $\beta_0$ and decay schedule ablations summarised in Table~\ref{tab:ablation-beta-schedule}. All runs use logit-space extrapolation.

\begin{table}[h]
\centering
\caption{\textbf{Detailed $\beta_0$ and decay schedule ablation results.} Accuracy (\%) on individual math benchmarks. The top block varies $\beta_0$ with linear decay; the middle block varies the schedule with $\beta_0\!=\!1.2$. Both blocks use Qwen3-8B. The bottom block varies $\beta_0$ on Qwen3-1.7B-Base (schedule noted in parentheses).}
\label{tab:app-ablation-detail}
\vspace{2pt}
\small
\setlength{\tabcolsep}{3pt}
\begin{tabular}{l|cccccc>{\columncolor{yellow!15}}c}
\toprule
\textbf{Configuration} & {\scriptsize MATH500} & {\scriptsize AIME24} & {\scriptsize AIME25} & {\scriptsize AMC23} & {\scriptsize Minerva} & {\scriptsize OlyBench} & {\scriptsize Avg.} \\
\midrule
\rowcolor{gray!15} \multicolumn{8}{l}{\textbf{Qwen3-8B --- $\beta_0$ sensitivity (linear decay)}} \\
$\beta_0 = 1.2$ & 84.8 & 56.9 & 46.7 & 91.6 & 32.5 & 62.4 & \textbf{62.5} \\
$\beta_0 = 1.3$ & 83.3 & 56.9 & 47.3 & 92.5 & 32.1 & 61.9 & 62.3 \\
$\beta_0 = 1.5$ & 83.5 & 55.2 & 47.9 & 91.2 & 31.9 & 61.6 & 61.9 \\
\midrule
\rowcolor{gray!15} \multicolumn{8}{l}{\textbf{Qwen3-8B --- decay schedule ($\beta_0 = 1.2$)}} \\
Linear & 84.8 & 56.9 & 46.7 & 91.6 & 32.5 & 62.4 & 62.5 \\
Cosine & 84.3 & 58.5 & 47.7 & 91.8 & 32.0 & 63.8 & \textbf{63.0} \\
Exponential & 82.4 & 57.5 & 47.1 & 91.2 & 31.7 & 61.8 & 62.0 \\
Fixed ($\beta\!=\!\beta_0$) & 83.8 & 56.2 & 44.4 & 90.6 & 31.6 & 63.9 & 61.8 \\
\midrule
\rowcolor{gray!15} \multicolumn{8}{l}{\textbf{Qwen3-1.7B-Base}} \\
$\beta_0 = 1.2$ (linear) & 59.6 & 12.1 & 8.1 & 45.6 & 20.2 & 29.4 & 29.2 \\
$\beta_0 = 1.2$ (cosine) & 58.6 & 10.8 & 8.3 & 42.5 & 19.8 & 29.7 & 28.3 \\
$\beta_0 = 1.3$ (linear) & 61.0 & 11.7 & 8.3 & 46.9 & 19.1 & 29.4 & \textbf{29.4} \\
$\beta_0 = 1.5$ (linear) & 59.2 & 11.0 & 6.9 & 45.6 & 19.5 & 28.6 & 28.5 \\
\bottomrule
\end{tabular}
\end{table}

\subsection{Resample vs.\ Rollout Reuse}
\label{sec:app-resample}

Table~\ref{tab:app-resample} reports per-benchmark results for the resample ablation discussed in \S\ref{sec:exp-ablations}. All runs use logit-space extrapolation with $\beta_0\!=\!1.2$ and linear decay.

\begin{table}[h]
\centering
\caption{\textbf{Resample vs.\ rollout reuse.} Accuracy (\%) on individual math benchmarks. ``Resample'' generates fresh rollouts from $\pi_{\theta_{n+1}'}$ for OPD; ``Rollout reuse'' reuses RLVR rollouts.}
\label{tab:app-resample}
\vspace{2pt}
\small
\setlength{\tabcolsep}{3pt}
\begin{tabular}{l|l|cccccc>{\columncolor{yellow!15}}c}
\toprule
\textbf{Model} & \textbf{Variant} & {\scriptsize MATH500} & {\scriptsize AIME24} & {\scriptsize AIME25} & {\scriptsize AMC23} & {\scriptsize Minerva} & {\scriptsize OlyBench} & {\scriptsize Avg.} \\
\midrule
\multirow{2}{*}{Qwen3-8B}
& Rollout reuse & 84.8 & 56.9 & 46.7 & 91.6 & 32.5 & 62.4 & 62.5 \\
& Resample & 83.7 & 58.5 & 47.9 & 90.6 & 31.3 & 63.1 & 62.5 \\
\midrule
\multirow{2}{*}{\shortstack[l]{Qwen3-1.7B\\-Base}}
& Rollout reuse & 59.6 & 12.1 & 8.1 & 45.6 & 20.2 & 29.4 & 29.2 \\
& Resample & 59.6 & 11.9 & 7.5 & 43.8 & 18.8 & 29.0 & 28.4 \\
\bottomrule
\end{tabular}
\end{table}

\subsection{Anchor Dynamics ($\eta$) Ablation}
\label{sec:app-anchor}

Table~\ref{tab:app-anchor} reports per-benchmark results for the anchor ablation discussed in \S\ref{sec:exp-ablations}. All runs use $\beta_0\!=\!1.2$. On OLMo3-7B-Instruct-SFT, the drop from EMA smoothing is substantial ($-6.3$ logit, $-2.7$ weight), likely because GRPO on OLMo produces large, stable per-step displacements that do not benefit from smoothing.

\begin{table}[h]
\centering
\caption{\textbf{Anchor ablation.} Accuracy (\%) on individual math benchmarks. $\eta\!=\!1$ uses the previous checkpoint as anchor; $\eta\!<\!1$ uses an EMA anchor. EMA helps on Qwen3-1.7B but degrades on OLMo3-7B-Instruct-SFT.}
\label{tab:app-anchor}
\vspace{2pt}
\small
\setlength{\tabcolsep}{3pt}
\begin{tabular}{l|cccccc>{\columncolor{yellow!15}}c}
\toprule
\textbf{Configuration} & {\scriptsize MATH500} & {\scriptsize AIME24} & {\scriptsize AIME25} & {\scriptsize AMC23} & {\scriptsize Minerva} & {\scriptsize OlyBench} & {\scriptsize Avg.} \\
\midrule
\rowcolor{gray!15} \multicolumn{8}{l}{\textbf{Qwen3-1.7B --- Logit-space}} \\
$\eta = 1.0$ & 77.4 & 29.4 & 26.0 & 71.3 & 24.8 & 48.6 & 46.2 \\
$\eta = 0.95$ & 75.4 & 32.7 & 29.8 & 73.1 & 25.0 & 50.2 & 47.7 \\
$\eta = 0.3$ & 74.6 & 33.3 & 29.4 & 76.9 & 25.5 & 50.6 & 48.4 \\
$\eta = 0.1$ & 75.2 & 36.3 & 33.3 & 78.8 & 24.9 & 52.9 & \textbf{50.2} \\
\midrule
\rowcolor{gray!15} \multicolumn{8}{l}{\textbf{Qwen3-1.7B --- Weight-space}} \\
$\eta = 1.0$ & 75.8 & 31.3 & 29.2 & 67.5 & 24.5 & 50.0 & 46.4 \\
$\eta = 0.3$ & 76.3 & 31.3 & 29.0 & 73.1 & 24.8 & 52.5 & 47.8 \\
$\eta = 0.1$ & 77.3 & 32.9 & 31.3 & 75.0 & 25.7 & 53.3 & \textbf{49.2} \\
\midrule
\rowcolor{gray!15} \multicolumn{8}{l}{\textbf{OLMo3-7B-Instruct-SFT --- Logit-space}} \\
$\eta = 1.0$ & 82.9 & 46.9 & 36.0 & 83.1 & 27.6 & 61.9 & \textbf{56.4} \\
$\eta = 0.1$ & 78.4 & 31.3 & 32.3 & 76.9 & 25.4 & 56.2 & 50.1 \\
\midrule
\rowcolor{gray!15} \multicolumn{8}{l}{\textbf{OLMo3-7B-Instruct-SFT --- Weight-space}} \\
$\eta = 1.0$ & 82.8 & 42.5 & 32.5 & 76.9 & 27.8 & 59.6 & \textbf{53.7} \\
$\eta = 0.1$ & 80.6 & 35.0 & 31.0 & 75.0 & 26.6 & 57.9 & 51.0 \\
\bottomrule
\end{tabular}
\end{table}

\subsection{Compute-Matched Comparison}
\label{sec:app-compute-matched}

Table~\ref{tab:app-compute-matched} reports per-benchmark results for the compute-matched comparison discussed in \S\ref{sec:exp-ablations}. GRPO-2$\times$ performs a second inner-loop gradient pass on the same rollouts within each iteration, matching the total gradient budget of RISE's RLVR + OPD phases.

\begin{table}[h]
\centering
\caption{\textbf{Compute-matched comparison: RISE vs.\ GRPO-2$\times$.} Accuracy (\%) on math benchmarks. GRPO-2$\times$ performs a second inner-loop gradient pass on the same rollouts within each iteration, matching RISE's total gradient budget at equal sampling cost.}
\label{tab:app-compute-matched}
\vspace{2pt}
\small
\setlength{\tabcolsep}{3pt}
\begin{tabular}{l|l|cccccc>{\columncolor{yellow!15}}c}
\toprule
\textbf{Model} & \textbf{Method} & {\scriptsize MATH500} & {\scriptsize AIME24} & {\scriptsize AIME25} & {\scriptsize AMC23} & {\scriptsize Minerva} & {\scriptsize OlyBench} & {\scriptsize Avg.} \\
\midrule
\multirow{4}{*}{Qwen3-8B}
& GRPO & 83.8 & 54.4 & 42.9 & 89.4 & 31.5 & 57.9 & 60.0 \\
& GRPO-2$\times$ & 84.5 & 53.5 & 42.3 & 92.5 & 31.4 & 58.9 & 60.5 \\
& RISE (logit) & 84.8 & 56.9 & 46.7 & 91.6 & 32.5 & 62.4 & 62.5 \\
& RISE (weight) & 84.4 & 58.1 & 45.8 & 91.9 & 32.6 & 63.3 & \textbf{62.7} \\
\midrule
\multirow{4}{*}{Qwen3-1.7B}
& GRPO & 75.6 & 30.0 & 26.3 & 65.6 & 23.9 & 51.1 & 45.4 \\
& GRPO-2$\times$ & 74.5 & 32.9 & 30.6 & 70.6 & 25.6 & 49.7 & 47.3 \\
& RISE (logit) & 75.2 & 36.3 & 33.3 & 78.8 & 24.9 & 52.9 & \textbf{50.2} \\
& RISE (weight) & 77.3 & 32.9 & 31.3 & 75.0 & 25.7 & 53.3 & 49.2 \\
\bottomrule
\end{tabular}
\end{table}

\subsection{Reproducibility Across Random Seeds}
\label{sec:app-seed-variance}

Table~\ref{tab:seed-variance} reports mean $\pm$ standard deviation across three random seeds for Qwen3-1.7B and Qwen3-8B on DAPOMath. Table~\ref{tab:main-math} reports a single representative seed; the multi-seed statistics confirm that RISE's gains are well outside seed-to-seed variance at both scales. At 1.7B, RISE (logit) improves Math Avg by $+4.5$ over GRPO ($49.6$ vs.\ $45.1$), roughly $10\times$ the per-method standard deviation ($\leq 0.6$). At 8B, RISE (weight) improves Math Avg by $+2.2$ over GRPO ($62.4$ vs.\ $60.1$) and by $+2.8$ over SDAR ($59.6$), again far exceeding the per-method standard deviation ($\leq 0.6$). Individual competition benchmarks (AIME, AMC) exhibit larger variance due to small problem sets, but the aggregate Math Avg is stable across all methods.

\begin{table}[h]
\centering
\caption{\textbf{Multi-seed reproducibility (DAPOMath): mean $\pm$ std across 3 random seeds.} Standard deviations are shown in \dt{gray}. RISE's Math Avg gains over GRPO far exceed seed-to-seed variance at both scales.}
\label{tab:seed-variance}
\vspace{2pt}
\scriptsize
\setlength{\tabcolsep}{2pt}
\begin{tabular}{l|ccccccc|cccc}
\toprule
\multirow{2}{*}{\textbf{Method}} & \multicolumn{7}{c|}{\textbf{In-Domain Math}} & \multicolumn{4}{c}{\textbf{OOD}} \\
& {\tiny MATH500} & {\tiny AIME24} & {\tiny AIME25} & {\tiny AMC23} & {\tiny Minerva} & {\tiny OlyBench} & {\tiny Avg.} & {\tiny GPQA} & {\tiny IFEval} & {\tiny MMLU} & {\tiny Avg.} \\
\midrule
\rowcolor{gray!15} \multicolumn{12}{l}{\textbf{Qwen3-8B (DAPOMath)}} \\
GRPO & 84.0 \dt{±0.5} & 53.1 \dt{±1.2} & 42.5 \dt{±0.3} & 88.5 \dt{±2.1} & 31.4 \dt{±0.4} & 61.3 \dt{±2.7} & 60.1 \dt{±0.2} & 56.3 \dt{±0.5} & 82.2 \dt{±0.5} & 71.6 \dt{±1.2} & 70.0 \dt{±0.4} \\
SDAR & 84.7 \dt{±0.4} & 48.9 \dt{±1.0} & 39.1 \dt{±1.2} & 90.0 \dt{±1.3} & 30.6 \dt{±0.6} & 64.2 \dt{±0.7} & 59.6 \dt{±0.6} & 55.3 \dt{±0.3} & 83.2 \dt{±0.6} & 71.4 \dt{±1.5} & 69.9 \dt{±0.6} \\
\textbf{RISE (weight)} & 84.9 \dt{±0.5} & 55.8 \dt{±1.7} & 44.1 \dt{±1.2} & 92.8 \dt{±0.6} & 32.3 \dt{±0.2} & 64.4 \dt{±0.9} & \textbf{62.4} \dt{±0.2} & 58.1 \dt{±0.7} & 84.1 \dt{±0.5} & 72.3 \dt{±1.0} & 71.5 \dt{±0.4} \\
\midrule
\rowcolor{gray!15} \multicolumn{12}{l}{\textbf{Qwen3-1.7B (DAPOMath)}} \\
GRPO & 75.4 \dt{±0.1} & 28.2 \dt{±1.3} & 26.2 \dt{±0.3} & 66.5 \dt{±1.6} & 24.1 \dt{±0.5} & 50.4 \dt{±0.6} & 45.1 \dt{±0.4} & 31.2 \dt{±0.8} & 68.7 \dt{±0.2} & 54.4 \dt{±0.7} & 51.4 \dt{±0.1} \\
\textbf{RISE (logit)} & 75.3 \dt{±0.2} & 34.5 \dt{±1.5} & 31.7 \dt{±1.5} & 78.8 \dt{±2.0} & 25.0 \dt{±0.1} & 52.4 \dt{±0.7} & \textbf{49.6} \dt{±0.4} & 34.3 \dt{±0.5} & 69.1 \dt{±1.0} & 55.5 \dt{±1.5} & 53.0 \dt{±0.7} \\
\textbf{RISE (weight)} & 76.5 \dt{±0.6} & 32.3 \dt{±0.8} & 30.2 \dt{±0.8} & 74.8 \dt{±0.5} & 25.1 \dt{±0.5} & 51.7 \dt{±1.3} & 48.4 \dt{±0.6} & 35.9 \dt{±1.1} & 68.7 \dt{±1.4} & 55.3 \dt{±1.5} & 53.3 \dt{±1.0} \\
\bottomrule
\end{tabular}
\end{table}

\end{document}